\documentclass{article}
\PassOptionsToPackage{numbers, compress}{natbib}

 \usepackage[preprint]{neurips_2026}

\usepackage[utf8]{inputenc} 
\usepackage[T1]{fontenc}    
\usepackage{hyperref}       
\usepackage{url}            
\usepackage{booktabs}       
\usepackage{amsfonts}       
\usepackage{nicefrac}       
\usepackage{microtype}      
\usepackage{xcolor}         
\usepackage{amsmath}
\usepackage{amssymb}
\usepackage{enumerate}
\usepackage{enumitem}
\usepackage{wrapfig}
\usepackage{booktabs}
\usepackage{array}
\usepackage{makecell}
\usepackage{caption}
\usepackage{subcaption}
\usepackage{graphicx}
\usepackage{adjustbox}
\usepackage{amsthm}

\usepackage{algorithm}
\usepackage{algpseudocode}
\usepackage{multirow}

\title{Isotonic Survival Regression: Calibrated Survival Distributions from Deep Cox Models}

\makeatletter
\renewcommand{\@fnsymbol}[1]{%
  \ifcase#1%
  \or *
  \or 1
  \or \dagger%
  \or \ddagger%
  \or \mathsection%
  \or \mathparagraph%
  \else \@arabic{#1}%
  \fi
}
\makeatother

\author{%
    Anchit Jain\thanks{Equal contribution. Correspondence to \texttt{\{ajain625,\,kzhang02\}@mit.edu}.} \textsuperscript{ }\thanks{EECS, MIT, Cambridge MA, USA.}
    \qquad Kevin Zhang\footnotemark[1] \textsuperscript{ }\footnotemark[2] \qquad Stephen Bates\footnotemark[2]
}

\DeclareMathOperator{\argmin}{\arg\min}
\algnewcommand{\LComment}[1]{\Statex \textbf{$\triangleright$ #1}}
\newtheorem{theorem}{Theorem}
\newcommand{\Sor}{S_{0,r}(t\mid r)}
\newtheorem{lemma}{Lemma}

\begin{document}

\maketitle

\vspace{-6pt}
\begin{abstract}

  Time-to-event data is widespread across the life sciences and engineering, but it is typically encountered together with censoring, which complicates the application of standard machine learning methods. Deep Cox models have emerged as a popular method for analyzing time-to-event data because they gracefully handle censoring and can be used with unstructured data such as clinical text reports, genomic sequences, and pathology images. However, their predicted survival probabilities are often poorly calibrated, thus limiting their practical utility. In this paper, we propose a novel post hoc calibration method for Deep Cox models that uses isotonic regression to refine predicted survival probabilities without affecting discriminative power. We establish favorable theoretical guarantees, including a double-robustness property and asymptotic calibration. Experiments on synthetic and real-world clinical data demonstrate the empirical effectiveness of our method.
\end{abstract}

\section{Introduction}
Survival analysis uses statistical methods to study the time until an event of interest occurs. One common objective is to predict the survival function, which gives the probability that the event time exceeds a specified time~\citep{Klein1997SurvivalAnalysis, Kalbfleisch2002FailureTimeData}. Accurate prediction of survival probabilities has many important applications in medicine~\citep{Knaus1993Sepsis, Fleming2000ClinicalTrials, Bachelot2000PrognosticFactors, Masciocchi2022FederatedCox}, epidemiology~\citep{Noone2017CancerIncidence, Spooner2020DementiaPrediction, Assael2002CysticFibrosis, Liu2022EpithelioidHemangio}, and engineering~\citep{Ruppert2021SetupTimes, Papathanasiou2023MachineFailure, Ahmed2025LeveragingSurvival}. 
For example, physicians may wish to predict the probability that a cancer patient remains alive and free of disease progression over time, also known as progression-free survival ~\citep{Luo2021Camrelizumab}.

The unique challenge in survival analysis is that event times are often only partially observed due to censoring~\citep{Leung1997CensoringIssues, Collett2014Modelling}. In this paper, we focus on \emph{right-censoring}, where for some individuals, we only observe that the event time exceeds the censoring time~\citep{Lagakos1979RightCensoring}. This occurs frequently in practice, for example when a patient is lost to follow-up or when a study ends before the event is observed~\citep{Flynn2012SurvivalAnalysis}. 

A wide range of methods have been developed for survival analysis, including parametric methods which assume specific distributional forms for the survival function~\citep{Beck1979CompetingRisks}, and non-parametric methods, such as the Kaplan-Meier estimator~\citep{KaplanMeier1958Incomplete}. 
The Cox proportional hazards model provides a popular semiparametric approach that combines a nonparametric baseline hazard with a covariate-dependent risk score~\citep{Cox1972Regression, Breslow1975ProportionalHazards}. Because of its interpretability and practicality, the Cox model remains one of the most widely used methods in survival analysis~\citep{Schober2018TortoiseHare,VanNess2023HeartFailure,Ahmed2007ColonCancer}. Other semiparametric methods, such as the accelerated failure time and proportional odds models, have also been proposed~\citep{Wei1992AFT,Murphy1997ProportionalOdds}.

Many modern applications of survival analysis now involve large amounts of high-dimensional and unstructured data, such as medical images~\citep{Wang2024PathologyFoundation,Lu2024ComputationalPathology}, genomic measurements~\citep{Hao2024SingleCell}, and electronic health records~\citep{Rasmy2021MedBERT} in clinical settings. This has motivated the use of deep learning methods, which can better capture complex relationships in these large and heterogeneous datasets. 
One prominent example is the \emph{Deep Cox model}~\citep{Katzman2018DeepSurv}, which extends the classical Cox model by replacing the linear risk predictor with a neural network to improve discriminative performance. Deep Cox models have since become a dominant method in contemporary survival analysis applications involving complex biomedical data~\cite{Kang2025BulkFormer, lu2021wsi_foundation, zhang2025standardizing, vaidya2025molecular, ding2025multimodal, chen2024uni, lu2024avisionlanguage, Lu_2023_CVPR, wang2024head_and_neck_survival, zhao2021bertsurv}.

For informed decision-making in safety-critical settings, good discrimination alone is not sufficient. Another desirable property of survival function estimators is \emph{calibration}, meaning that predicted probabilities reflect the true event frequencies~\cite{Guo2017Calibration}. 
Recent work has established that deep learning methods are often poorly calibrated in practice~\citep{Guo2017Calibration,Minderer2021RevisitingCalibration,Nixon2019MeasuringCalibration,Thulasidasan2019MixupCalibration}, a concern which extends to Deep Cox models~\citep{Kvamme2019TimeToEvent}. 
Isotonic regression offers a popular approach for post hoc calibration of deep learning models, because it learns the most flexible mapping from predicted to calibrated probabilities while preserving performance~\citep{Zadrozny2002IRCalibration}. However, extending such methods to survival analysis is challenging, because it requires targeting an entire distributional forecast while accounting for censoring.


\paragraph{Our contributions.} In this work, we study the problem of calibrating survival predictions from classical Cox and Deep Cox models. We propose a lightweight post hoc calibration procedure based on isotonic regression that uses estimated risk scores together with an initial survival and censoring model.
Our main contributions are as follows: (i) we introduce a novel post hoc method for calibrating survival probability predictions under right-censoring while preserving discriminative performance, (ii) we establish theoretical guarantees for the calibration and consistency of our algorithm, including a double-robustness result and (iii) we demonstrate empirically on synthetic and real-world datasets that our method improves calibration for Cox and Deep Cox models. 

\section{Preliminaries}
\label{sec:prelims}
We introduce the notation and preliminary background for survival analysis and calibration.

\paragraph{Survival Analysis.} Let $X \in \mathcal{X}$ denote the covariates for an individual, $T$ denote the time until the event of interest occurs, and $C$ denote the censoring time. 
Due to right-censoring, we only observe time $Y = \min(T,C)$, together with the event indicator $\delta = \mathbf{1}[T \leq C]$.
The goal of survival prediction is to estimate the conditional survival function $S(t \mid X) = \mathbb{P}(T > t \mid X)$, which gives the probability that the event time exceeds $t$ given $X$. The survival function can also be written using the hazard function $\lambda_{T \mid X}(t \mid X) = -\frac{d}{dt}\log S(t \mid X)$, which describes the instantaneous event rate at time $t$ among individuals who have survived up to that time. 
Similarly, we define the censoring function $G(t \mid X) = \mathbb{P}(C > t \mid X)$ as the probability that the censoring time exceeds $t$ given $X$. 

\paragraph{Cox and Deep Cox models.} The Cox model is a popular semi-parametric approach for learning the survival function. The model assumes that the hazard function can be represented in the form 
\[ \lambda(t \mid X) = \lambda_0(t) \exp(r_\theta(X)) \] where $\lambda_0(t)$ is the baseline hazard and $r_\theta(X)=\theta^\top X$ is a covariate-dependent risk score. Deep Cox models~\citep{Katzman2018DeepSurv} extend the classical Cox model by relaxing the linear form of the risk with a neural network $r_\theta: \mathcal{X} \to \mathbb{R}$. Intuitively, the model scales a shared baseline hazard, so higher-risk individuals experience the event at a faster rate. Typically, the parameters $\theta$ are estimated by maximizing the partial likelihood while the baseline hazard is estimated using Breslow's estimator~\citep{Breslow1972BreslowEstimator}.

A common way to evaluate Cox models is through discrimination, or the ability to correctly rank individuals by risk. This is typically quantified using Harrell's concordance index (C-index)~\citep{Harrell1996Cindex}, which estimates the frequency that among comparable pairs of samples, the one that experiences the event earlier is assigned a higher predicted risk.

\paragraph{Calibration.} 

We desire survival models that are well-calibrated, meaning that predicted probabilities are statistically reliable. To evaluate this, we focus on two related notions of calibration. The first is \emph{threshold calibration}~\citep{Henzi2021IsotonicDistributional}, which requires for any predicted survival probability $p$, we have $\smash{\mathbb{P}(T > t \mid \widehat{S}(t \mid X) = p) = p}$ almost surely. Another weaker notion is \emph{PIT calibration}, based on the probability integral transform (PIT). We say a survival function $\smash{\widehat{S}(t \mid X)}$ is PIT calibrated if the corresponding PIT values $\widehat{F}(T \mid X) = 1 - \widehat{S}(T \mid X)$ are uniformly distributed on $[0,1]$. 


\paragraph{Isotonic regression.} Given data $\{(x_i,y_i)\}_{i=1}^n$, one-dimensional isotonic regression estimates a monotonic function by solving $f_\text{IR} = \argmin_{f \in \mathcal{F}} \sum_{i=1}^n |f(x_i) - y_i|^2$, where $\mathcal{F}$ is the set of non-increasing functions. 
For classification tasks, isotonic regression is widely-used for post hoc calibration of machine learning probability estimates~\citep{Zadrozny2002IRCalibration}, as it preserves accuracy while learning a flexible recalibration mapping.
Algorithmically, this problem can be implemented efficiently using the pool adjacent violators algorithm (PAVA)~\citep{barlow1972IRBook}, which runs in linear time after sorting the observations.


\begin{figure}
    \centering
    \includegraphics[width=0.95\linewidth]{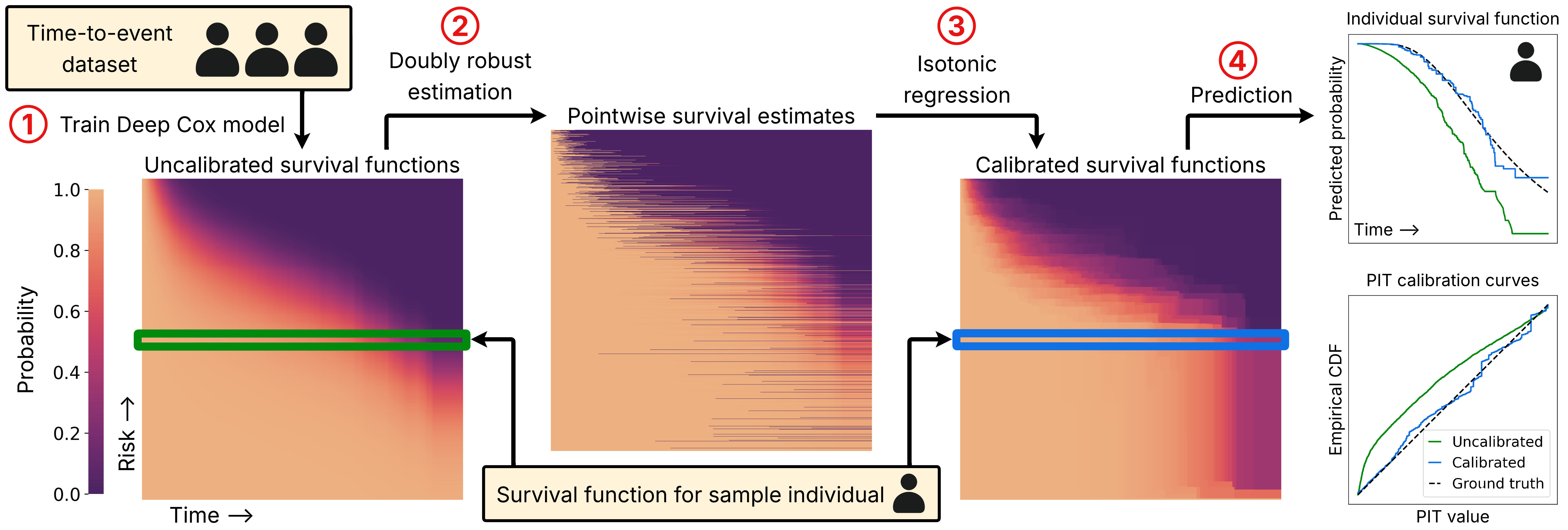}
    \caption{Our DR-ISR calibration method first uses uncalibrated estimates from an initial Deep Cox model to construct doubly robust pointwise survival estimates. We then use isotonic regression to project onto the set of valid survival functions that are monotonic in both time and risk. We demonstrate our approach improves the distributional accuracy and calibration of survival forecasts.}
    \label{fig:teaser}
\end{figure}

\section{Distributional regression for survival analysis}
\label{sec:estimator_construction}

We now introduce two methods for estimating the conditional survival function under right censoring using isotonic regression. We begin with a simpler reweighting estimator and then proceed to a more sophisticated doubly robust estimator. 

Suppose we have fit a Deep Cox model with learned risk score $r_\theta$.
Since the predicted risks determine the discriminative power of the model, our goal is to calibrate the survival function without altering the ordering induced by the learned risks. We therefore restrict our attention to estimators $S(t \mid X)$ that are non-increasing in time $t$ and in the predicted risk $r_\theta(X)$, so that individuals assigned higher risk have uniformly lower survival probabilities. We call this class of survival curves \emph{valid} and refer to our method as Isotonic Survival Regression (ISR).

Let $(X_i, Y_i, \delta_i)_{i=1}^n$ denote the observed data, and let $S_0$ and $G_0$ denote the true survival and censoring functions, respectively. Following the standard post hoc calibration paradigm, we split the data into a training set $\mathcal{D}_{\mathrm{train}}$ and a calibration set $\mathcal{D}_{\mathrm{cal}}$. The training set is used to fit an initial survival and censoring model, while the calibration set is reserved for constructing the final calibrated estimator.

\subsection{Reweighted isotonic survival regression (RW-ISR)}

We now turn to a relatively simple reweighting estimator. To this end, consider first predicting the binary outcome $\mathbf{1}[T > t]$.
In the absence of censoring, predicting the survival function can be done with isotonic regression. However, censoring causes some of the indicators to be unobserved. One popular approach to handling missing data is Inverse Probability of Censoring Weighting (IPCW), which assigns higher importance to samples that are more likely to be censored. To apply this approach, we fit a censoring model $\smash{\widehat{G}}$ on the training data, for example using a Cox model with Breslow's estimate for the baseline hazard function. With this in hand, we define the RW-ISR estimator as the solution to the following weighted isotonic regression problem:
\begin{equation} \label{eqn:hajek_optimization}
    \widehat{S}_\text{ISR}^\text{RW}(t \mid x) = f_t(r_\theta(x)), \quad \text{where } f_t = \underset{f:\text{non-increasing}}{\arg\min} \:\: \sum_{X_i \in \mathcal{D}_\text{cal}} \frac{\delta_i (f(r_\theta(X_i)) - \mathbf{1}[Y_i > t])^2}{\widehat{G}(Y_i \mid X_i)}.
\end{equation}


To extend this idea to estimate the entire survival curve, we can solve the isotonic regression for each time $t$. Since the target labels $\mathbf{1}[Y_i > t]$ are non-increasing across $t$, the resulting estimator is non-increasing in the risk as well. Thus, the estimator yields a valid survival curve. 


\paragraph{Shortcomings.} RW-ISR has two main weaknesses: (i) the estimator depends only on uncensored individuals, thereby discarding information from censored cases, and (ii) its calibration depends on the quality of the estimated censoring probabilities. In missing-data problems, parameter estimation under potentially misspecified weights is commonly addressed using \emph{doubly robust} estimators, which are more robust to imperfection in the censoring model and remain consistent if either the outcome model or the weight model is correctly specified. We turn to such an approach next.


\subsection{Doubly robust isotonic survival regression (DR-ISR)}

While RW-ISR provides a simple censoring-adjusted calibration procedure, we can improve upon it using a doubly robust alternative. To do so, we first construct pointwise estimates of the survival probabilities, which we refer to as \emph{pseudo-outcomes}, over the grid $(X,t) \in \mathcal{D}_\mathrm{cal} \times \mathcal{T}$, 
where $\mathcal{T}$ is a time grid, either user-specified or induced by the natural discretization of observed events (e.g., days). 
We then project these estimates onto the set of valid survival functions using isotonic regression, as shown in Figure~\ref{fig:teaser}. DR-ISR both uses more information from censored cases and is less sensitive to misspecification of the censoring model compared to RW-ISR (see Section~\ref{sec:theory} for formal results). 

\paragraph{Pseudo-outcome construction.} 
We construct pseudo-outcomes by estimating pointwise survival probabilities using the augmented IPCW estimator motivated by standard semiparametric theory~\citep{Robins1992DependentCensoring,Tsiatis2006Semiparametric,Farina2026DoublyRobustEfficient}.
The estimator starts with an initial estimate of the survival function $\smash{\widehat S}$ and censoring probability $\smash{\widehat G}$, then adds an augmentation term based on the observed calibration data. The correction is carefully constructed by measuring the error in the models $\smash{\widehat S}$ and $\smash{\widehat G}$ across time. This can be done in a way that is doubly robust, and thus is less sensitive to model misspecification.

To formulate the estimator, define the event counting process $\smash{dN_T(u) = \mathbf{1}[Y \in du, \delta = 1]}$, which indicates whether an uncensored event is observed at time $u$, and the estimated cumulative hazard function $\smash{\widehat{\Lambda}_{T \mid X} (u \mid X) = -\log \widehat{S} (u \mid X)}$. We then define the following event difference process $\smash{d\widehat{M}_{T \mid X} (u \mid X) = dN_T(u) - \mathbf{1}[Y \geq u] \, d\widehat{\Lambda}_{T \mid X} (u \mid X)}$, whose expectation tracks the difference between the frequency of observed events and the model-based probabilities. Note that $\smash{d\widehat{M}_{T \mid X}}$ is a random variable, while $\smash{\widehat{\Lambda}_{T \mid X}}$ is fixed. We propose the following doubly robust pointwise survival estimate as the pseudo-outcome:
\begin{equation}
\label{eq:dr_pseudo_label}
     \tilde{S}^\text{DR} \big(t \mid X; \widehat{S}, \widehat{G} \big) = \widehat{S}(t \mid X) - \int_0^t \frac{\widehat{S}(t \mid X)}{\widehat{S}(u \mid X) \widehat{G}(u \mid X)} \,d\widehat{M}_{T \mid X} (u \mid X).
 \end{equation}
The integral is a residual correction accumulated over times $u \leq t$ that adjusts the initial model  when the observed event process differs from the model predictions. Similar to RW-ISR, we estimate $\smash{\widehat S}$ and $\smash{\widehat G}$ on the training data, for example using a Cox model with Breslow's estimator of the baseline hazard function.
%
Note that the difference process is a discrete measure for piecewise constant survival estimators, including Cox models. Thus, the integral reduces to a discrete sum over times in $\mathcal{T}$.

\paragraph{From pseudo-outcomes to distributional forecasts.} First, the pseudo-outcomes $\tilde{S}^\text{DR} \big(t \mid X; \widehat{S}, \widehat{G} \big)$ are separately constructed for each $X \in \mathcal{D}_\text{cal}$ and $t \in \mathcal{T}$. Although the conditional expectation of the pseudo-outcome does target $S_0(t \mid X)$ (see Theorem~\ref{thm:dr}), they do not directly yield satisfactory conditional survival probabilities because each is constructed with only a single observation.
As such, the pseudo-outcomes can fluctuate rapidly across nearby covariate values---see Figure~\ref{fig:teaser}. 

To obtain distributional forecasts, we next regress these pseudo-outcomes on the risk, thereby sharing information across similar covariates. This smoothing step improves statistical efficiency and stabilizes the otherwise noisy pointwise estimates. Moreover, the pseudo-outcomes do not necessarily define a valid family of survival functions and may have worse discriminative performance than the initial survival model. By projecting the pseudo-outcomes onto the class of survival functions that are monotonic in both time and risk, we address this issue as well. 





Let $\mathcal{F}$ denote the class of functions $f:\mathbb{R}_{\geq 0}\times\mathbb{R}\to[0,1]$, where the two arguments correspond to time and risk. Let $\mathcal{F}_\leq \subset \mathcal{F}$ denote the subset of functions that are non-increasing in both coordinates. We define the calibrated survival estimate as $\smash{\widehat{S}_{\mathrm{ISR}}^\mathrm{DR}(t \mid x) = f(t,r_\theta(x))}$, where $f$ solves
\begin{equation} \label{eqn:DR_estimation}
     f = \underset{f \in \mathcal{F}_\leq}{\arg\min} \:\: \sum_{t \in \mathcal{T}} \: \sum_{X_i \in \mathcal{D}_\text{cal}} \left[ \left( f(t, r_\theta(X_i)) - \tilde{S}^\text{DR} \big(t \mid X_i; \widehat{S}, \widehat{G} \big) \right)^2 \right].
 \end{equation}
While the two-dimensional isotonic regression optimization in Equation~\ref{eqn:DR_estimation} is convex, solving it with a generic convex optimization solver requires $\mathcal{O}(n|\mathcal{T}|$) constraints, which can be computationally burdensome. Instead, we apply Dykstra's projection algorithm~\citep{Dykstra1982alternatingprojectionsPAVA}, which exploits the geometry of $\mathcal{F}_\leq$. The latter can be written as the intersection of two convex sets: non-increasing functions in time and non-increasing functions in risk. We then solve the optimization by alternating between projections onto the two sets. Each projection reduces to a collection of one-dimensional isotonic regression problems along time or risk, each of which can be computed efficiently using PAVA~\citep{barlow1972IRBook}. 

Putting both steps together, our calibration procedure is outlined in Algorithm~\ref{alg:dr}. The time complexity is $\mathcal{O}(T_\mathrm{Dyk} \cdot n|\mathcal{T}|)$, where $T_\mathrm{Dyk}$ is the number of Dykstra's iterations. In practice, PAVA is extremely fast which makes DR-ISR a lightweight post hoc calibration procedure that is practical to deploy on top of modern Deep Cox models. We also highlight that the calibration procedure does not require tuning any hyperparameters, making it especially straightforward to deploy. 



\begin{algorithm}[t]
\caption{DR-ISR calibration algorithm} \label{alg:dr}
\begin{algorithmic}[1]
\Require Survival $\widehat{S}$, censoring $\widehat{G}$, and risk $r_\theta$ models from an initial Deep Cox model,
\Statex \hspace{2.2em} $\mathcal{D}_{\text{cal}} = \{(X_i, Y_i, \delta_i)\}_{i=1}^{n_{\text{cal}}}$ sorted by increasing risk, time grid $\mathcal{T} = \{t_1, \ldots, t_K\}$

\For{$i \in [K]$ and $j \in [n_\text{cal}]$} \Comment{Doubly robust point estimation}
    \State $d\widehat{\Lambda}_{T \mid X} (t_i \mid X_j) \gets 1 - \widehat{S}(t_i \mid X_j) / \widehat{S}(t_{i-1} \mid X_j)$ \Comment{Define $\widehat{S}(t_0 \mid X_j) = 1$}
    \State $d\widehat{M}_{T \mid X} (t_i \mid X_j) \gets \mathbf{1}[Y_j = t_i, \delta_j = 1] - \mathbf{1}[Y_j > t_i] \, d\widehat{\Lambda}_{T \mid X} (t_i \mid X_j)$

    \State $\tilde{S}^\mathrm{DR}[i,j] \gets \widehat{S}(t_i \mid X_j) - \sum_{k=1}^i \frac{\widehat{S}(t_i \mid X_j)}{\widehat{S}(t_k \mid X_j) \widehat{G}(t_k \mid X_j)} d\widehat{M}_{T \mid X} (t_k \mid X_j)$ \Comment{Equation~\ref{eq:dr_pseudo_label}}
\EndFor


\State Initialize $S \gets \tilde{S}_{\text{DR}}$ and residuals $\epsilon_{\text{time}} \gets \mathbf{0}^{n_{\text{cal}} \times |\mathcal{T}|}$, $\epsilon_{\text{risk}} \gets \mathbf{0}^{n_{\text{cal}} \times |\mathcal{T}|}$ \Comment{Dykstra's algorithm}

\While{not converged}
    \State $U \gets S + \epsilon_{\text{risk}}, \quad S[:, j] \gets \texttt{PAVA}(U[:, j]) \:\:\: \forall j, \quad \epsilon_{\text{risk}} \gets U - S$ \Comment{Monotonicity in risk}
    \State $V \gets S + \epsilon_{\text{time}}, \quad S[i, :] \gets \texttt{PAVA}(V[i, :]) \:\:\: \forall i, \quad \epsilon_{\text{time}} \gets V - S$ \Comment{Monotonicity in time}
    
\EndWhile
 \State Interpolate $S$ to form 2D function $\hat{f}$, then make predictions $\widehat{S}_{\mathrm{ISR}}^\mathrm{DR}(t \mid X_{\text{test}}) \gets \hat{f}(t, r_\theta(X_{\text{test}}))$
 \end{algorithmic}
\end{algorithm}

\section{Theoretical Guarantees}
\label{sec:theory}
We next describe the theoretical properties of DR-ISR. First, we formalize the double-robustness property of the pseudo-outcomes (Theorem~\ref{thm:dr}). We then show that this property is preserved by the isotonic regression projection step, so that our end-to-end procedure remains consistent if either the event model or the censoring model converges locally uniformly (Theorem~\ref{thm:consistency}). Lastly, we show that even when the survival curves are not isotonic in the risk score, the procedure still yields calibrated survival curves (Theorem~\ref{thm:calibration}). We will make the following key assumptions.

 \begin{enumerate}[label=A\arabic*., ref=A\arabic*]
     \item \label{ass:independent_censoring} The censoring and event times are conditionally independent, i.e. $T \perp\!\!\!\perp C \mid X$.
     \item \label{ass:risk_correctness} The risk scores from the uncalibrated model $r_\theta$ induce a monotone ordering over survival functions, i.e.
     \[ r(x) \geq r(x') \implies \mathbb{P}(T > t \mid X = x) \leq \mathbb{P}(T > t \mid X = x'), \quad \text{for all } t \geq 0. \]
 \end{enumerate}

Assumption~\ref{ass:independent_censoring} is a standard identifiability condition in survival analysis, 
while assumption~\ref{ass:risk_correctness} imposes the additional structure that survival is monotone in this score. 

We first make the double-robustness property of the point estimates from Equation~\ref{eq:dr_pseudo_label} precise. Let $\Lambda_{C \mid X} (u \mid X) = -\log G(u \mid X)$ be the censoring analog of the cumulative hazard function.

\begin{theorem}[Pseudo-outcome Double-robustness]
\label{thm:dr}
    Suppose that assumption~\ref{ass:independent_censoring} holds and we also have the mild regularity condition~\ref{ass:positivity} from Appendix~\ref{app:sec:proofs}. Define the bias of the pseudo-outcomes as $B \big(t, x; \widehat{S}, \widehat{G} \big) = |\mathbb{E}[ \tilde{S}^\mathrm{DR} \big(t \mid X; \widehat{S}, \widehat{G} \big) \mid X = x ] - S_0(t\mid x)|$. Then, we have that $B \big(t, x; \widehat{S}, \widehat{G} \big) \leq \min\{B_1 \big( t,x;\widehat{S}, \widehat{G} \big), B_2 \big( t,x;\widehat{S}, \widehat{G} \big)\}$, where the bounds $B_1$ and $B_2$ admit the following mixed-product form,
\begin{equation} \label{eqn:dr_bias}
     B_1 \big(t, x; \widehat{S}, \widehat{G} \big) = c_1(t) \int_0^t \left|\frac{1}{\widehat{G}(u\mid x)} - \frac{1}{G_0(u\mid x)\vphantom{\widehat{G}}}\right|\left|d\Lambda_{0,T\mid X}(u\mid x) - d\widehat{\Lambda}_{T\mid X}(u\mid x)\right|,
\end{equation}
\begin{equation} \label{eqn:dr_bias2}
     B_2 \big(t, x; \widehat{S}, \widehat{G} \big) = c_2(t) \int_0^t \left|\frac{\widehat{S}(t\mid x)}{\widehat{S}(u\mid x)} - \frac{S_0(t\mid x)}{S_0(u\mid x)\vphantom{\widehat{S}}}\right|\left|d\Lambda_{0,C\mid X}(u\mid x) - d\widehat{\Lambda}_{C\mid X}(u\mid x)\right|.
\end{equation}
Here, $c_1(t)$ and $c_2(t)$ are finite positive constants depending on $t$ as defined in~\ref{ass:positivity}. In particular, if either $\widehat{S}=S_0$ or $\widehat{G}=G_0$, then the DR-ISR pseudo-outcome is conditionally unbiased for $S_0(t \mid x)$. 
\end{theorem}



The double-robustness property similarly extends to uniform consistency of the survival function estimator under locally uniformly consistent estimation of either $S_0$ or $G_0$.

\begin{theorem}[Survival Function Double-robustness]
\label{thm:consistency}
Suppose that assumptions~\ref{ass:independent_censoring}--\ref{ass:risk_correctness} hold and we also have the mild regularity conditions~\ref{ass:positivity}--~\ref{ass:continous_survival} from Appendix~\ref{app:sec:proofs}. Then, if either $\smash{\widehat{S}} \to S_0$ or $\smash{\widehat{G}} \to G_0$ locally uniformly, i.e. for all times $t \in \mathcal{T}$ and covariates $x$, we have
\[ \sup_{u \in [0,t]} |S(u \mid x) - S_0(u \mid x)| \stackrel{p}{\longrightarrow} 0, \quad \text{or } \quad \sup_{u \in [0,t]} |G(u \mid x) - G_0(u \mid x)| \stackrel{p}{\longrightarrow} 0, \]
then
\[ \sup_{x \in \mathcal{X}'} |\widehat{S}^{\mathrm{DR}}_{\mathrm{ISR}}(t \mid x)-S_0(t\mid x)| \stackrel{p}{\longrightarrow} 0, \quad \text{for all } t \in \mathcal{T}, \]
where $\mathcal{X}' = \{x \in \mathcal{X} : r(x) \in \mathrm{int}(r(\mathcal{X}))\}$.
\end{theorem}

Theorem~\ref{thm:consistency} provides an asymptotic conditional calibration guarantee since the estimator converges to the true survival function. In practice, assumption ~\ref{ass:risk_correctness} may not hold because the risk scores are estimates from the model we wish to calibrate and may be incorrect. Nonetheless, DR-ISR is still well-calibrated in the following sense.

\begin{theorem}[Calibration under Risk Misspecification]
\label{thm:calibration}
Suppose all assumptions from Theorem~\ref{thm:consistency} hold except ~\ref{ass:risk_correctness}. Define $\smash{\widehat{S}^\infty_\mathrm{DR\text{-}ISR}(t \mid x)  = \lim_{n\rightarrow\infty}\widehat{S}^{\mathrm{DR}}_{\mathrm{ISR}}(t \mid x)}$ as the pointwise limit of the DR-ISR estimator. 
If either $\widehat{S} \rightarrow S_0$ or $\widehat{G} \rightarrow G_0$ locally uniformly, then for all $p$ in the range of $\widehat{S}^\infty_\mathrm{DR\text{-}ISR}$,
\begin{equation}
    \mathbb{P} \big( T>t \ \big| \ \widehat{S}^\infty_\mathrm{DR\text{-}ISR} (t,X) = p \big) = p, \qquad \text{for all } t \in \mathcal{T}.
\end{equation}
\end{theorem}

Intuitively, Theorem~\ref{thm:calibration} guarantees that the average of true survival probabilities over all individuals that are assigned a probability $p$ by the DR-ISR estimator will also be $p$, even if the risk scores are misspecified. This is a form of threshold calibration (Sec.~\ref{sec:prelims}).

\section{Related Work}

\paragraph{Conformal methods for survival analysis.} Much of the recent work on statistically valid survival prediction has focused on \emph{conformal prediction}~\citep{Candes2023Conformalized,Gui2024AdaptiveCutoffs,Sesia2025SurvivalBands,Qin2025Resampling,Yi2025RandomCensoring,Si2025TrainingSetConditional,Davidov2025GeneralRightCensored,Holmes2024TwoSided,Sesia2025SurvivalBands,Farina2026DoublyRobustEfficient}. Rather than calibration, the typical goal in this setting is to construct a valid \emph{lower prediction bound} (LPB) for the event time, i.e., a time beyond which an individual is expected to survive with probability at least $1-\alpha$ for a user-specified miscoverage level $\alpha$. These methods build on the conformal inference framework, which provides finite-sample coverage guarantees under minimal distributional assumptions.
Many of these works estimate the censoring mechanism and use IPCW to construct conformal prediction sets~\citep{Qin2025Resampling, Yi2025RandomCensoring, Si2025TrainingSetConditional, Farina2026DoublyRobustEfficient}. Closest to our work, \citet{Farina2026DoublyRobustEfficient} use doubly robust estimators to estimate marginal quantiles of the event time distribution. Although their estimator resembles our pseudo-label point estimate, our approach instead targets survival probabilities directly with isotonic regression.

\begin{wraptable}{r}{0.51\textwidth}
\vspace{-0.7em}
\centering
\captionsetup{width=\linewidth}
\caption{Comparison of relevant calibration method properties. Calibration guarantees are asymptotic.}
\label{tab:method_comparison}
\renewcommand{\arraystretch}{1.35}
    {\small \begin{tabular}{ccccc}
    \toprule
    Method & \makecell{Calibration\\guarantee} & Monotonic & \makecell{C-index\\preserving}\\
    \midrule
    CSD & $\times$ & $\times$ & $\checkmark$ \\
    CiPOT & $\checkmark$ & $\checkmark$ & $\times$ \\
    \midrule
    \textbf{Ours} & \textcolor{blue}{$\checkmark$} & \textcolor{blue}{$\checkmark$} & \textcolor{blue}{$\checkmark$} \\
    \bottomrule
    \end{tabular}}
\vspace{-1em}
\end{wraptable}

\paragraph{Post hoc calibration of survival distributions.}
Turning to the calibration of probability estimates, the methods most closely related to ours are CSD~\citep{Qi2024ConformalizedDistributions} and CiPOT~\citep{Qi2024ConditionalDistributionCalibration}. CSD calibrates survival curves by discretizing them at multiple percentile levels and applying conformal quantile regression at each level. 
CiPOT extends this line of work by directly mapping the empirical distribution of PIT values to the uniform scale. This guarantees asymptotic marginal and conditional calibration, but does not in general preserve the C-index. 
CSD handles censoring by sampling missing event times from an estimated survival model, while CiPOT directly samples on the PIT scale. Finally, several modifications to the training procedure~\citep{avati2020CountdownRegression,Goldstein2020Xcal,chapfuwa2023calibration} to improve the calibration of deep learning models for survival analysis have been proposed, but these often come at the cost of worse discriminative performance~\citep{Qi2024ConformalizedDistributions}. 


\vspace{-6pt}
\paragraph{Isotonic Distributional Regression.}
Our methodology also builds on Isotonic Distributional Regression ~\citep{Henzi2021IsotonicDistributional} and related work~\citep{jimenez2003estimation, el2005inferences, davidov2012estimating, mosching2020IDR}. 
Given a partial ordering over the input space, Isotonic Distributional Regression is a nonparametric approach for estimating the conditional CDF of the response variable subject to this ordering. 
An important extension is the Distributional Single-Index Model~\citep{Henzi2021SingleIndex} which first uses an index model to order the covariate space, and then applies Isotonic Distributional Regression with respect to the learned index. We likewise learn an ordering, which is induced by the Deep Cox risk score, and then adapt the Distributional Single-Index Model framework to right-censored survival data with two key extensions. First, we account for censoring using IPCW or doubly robust pseudo-outcomes. Second, because the calibrated survival function must be monotone in both risk and time, we use a two-dimensional isotonic regression procedure to enforce both constraints which we implement efficiently using Dykstra's algorithm with PAVA~\citep{Dykstra1982alternatingprojectionsPAVA}.

\section{Experiments} \label{sec:experiments}

We evaluate ISR for calibrating survival functions on synthetic data and on two real-world settings using unstructured data from The Cancer Genome Atlas. For all Cox models, we estimate the baseline hazard using Breslow's estimator. Hyperparameter details are provided in Appendix~\ref{app:sec:experiment_details}.

\subsection{Metrics}
\label{sec:metrics}
We evaluate the estimated survival distributions along three metrics: PIT calibration, quantile estimation, and the Integrated Brier Score (IBS). On the real data, we use IPCW to correct for censoring since we do not have access to the true event times.
For all three metrics, lower values indicate better performance. We compare ISR with the two existing baselines for post hoc calibration of survival curves: CSD~\citep{Qi2024ConformalizedDistributions} and CiPOT~\citep{Qi2024ConditionalDistributionCalibration}. 

\paragraph{PIT Calibration (AUPIT).} Recall PIT calibration requires PIT values be distributed uniformly on the interval $(0,1)$. We quantify miscalibration as the unsigned area between the empirical cumulative distribution function (CDF) of the predicted PIT values and the uniform CDF:
\[ \mathrm{AUPIT} = \int_0^1 \left| \widehat{F}_{\mathrm{PIT}}(u) - u \right| \, du, \]
where $\widehat{F}_{\mathrm{PIT}}$ is the IPCW-adjusted empirical CDF of the PIT values. Notably, the AUPIT is equivalent to the Wasserstein-1 distance between the predicted PIT distribution and the ideal uniform distribution.

\paragraph{Quantile Estimation (IPCW Pinball Loss).} 
We also evaluate the quality of the survival curves by assessing the accuracy of predicted quantiles using quantile regression loss~\citep{Koenker2005QuantileRegression}. The IPCW estimate reweights the observable indicators in the pinball loss as the following:
\[ \mathrm{QS}(\tau) = \frac{1}{|\mathcal{I}_{\mathrm{test}}|} \sum_{i \in \mathcal{I}_{\mathrm{test}}} \bigg[ (1-\tau) \int_0^{\hat q_\tau(X_i)} \frac{\delta_i \mathbf{1}[Y_i \leq t]}{\widehat G(Y_i \mid X_i)} \, dt + \tau \int_{\hat q_\tau(X_i)}^{t_{\max}} \frac{\mathbf{1}[Y_i > t]}{\widehat G(t \mid X_i)} \, dt \bigg], \]
where $t_{\max}$ is the maximum time up till which predictions are made (e.g., the largest observed time in the calibration set) and $\hat{q}_\tau(X_i)$ is the time at which $\widehat{S}(t|X_i)$ reaches $1-\tau$. The left and right integrals compute the weighted penalty for overestimation and underestimation respectively, analogous to the standard quantile loss (see Appendix~\ref{app:sec:metrics}). Note that if the survival curve does not reach the $1-\tau$ level before $t_{\max}$, then $\mathrm{QS}(\tau)$ is ill-defined; we discuss ways to handle this in Appendix~\ref{app:sec:metrics}. 

\paragraph{Overall Distributional Accuracy (IBS).}
While the quantile score evaluates predictive accuracy at a specific level of the event-time distribution, it does not capture the calibration of the \textit{entire} predicted survival curve. To assess the full forecast, we evaluate overall distributional accuracy using IBS~\citep{graf1999IBS}. The Brier score measures the mean squared error of a probabilistic prediction at a single time point $t$. By integrating this penalty across all times up to $t_{\max}$, the IBS measures the calibration and sharpness of the entire predicted survival curve:
\[ \mathrm{IBS} = \frac{1}{|\mathcal{I}_{\mathrm{test}}|} \sum_{i \in \mathcal{I}_{\mathrm{test}}} \frac{1}{t_{\max}}\int_0^{t_{\max}} \bigg[ \frac{\delta_i \mathbf{1}[Y_i \leq t] \hat S(t \mid X_i)^2}{G(Y_i \mid X_i)} + \frac{\mathbf{1}[Y_i > t] (1 - \hat S(t \mid X_i))^2}{G(t \mid X_i)} \bigg] \, dt . \]
We show in Appendix~\ref{app:ibs} that IBS is also equivalent to integrating the quantile score across all levels $\tau \in (0,1)$, up to scaling factors.

\begin{figure}[t]
    \centering

    \newsavebox{\leftfigbox}
    \newlength{\leftfigheight}

    \sbox{\leftfigbox}{%
        \includegraphics[width=0.64\textwidth]{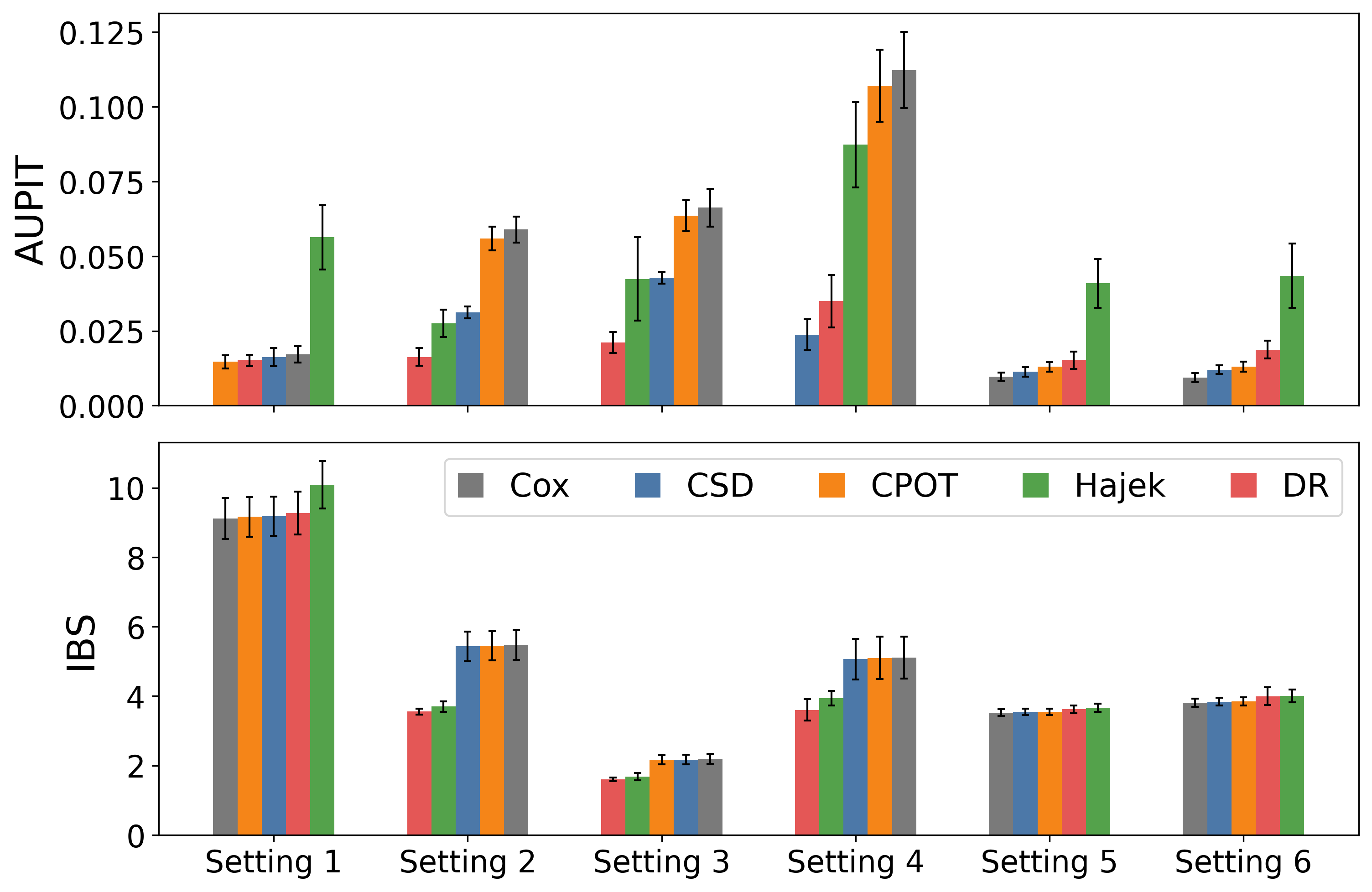}
    }

    \setlength{\leftfigheight}{\dimexpr\ht\leftfigbox+\dp\leftfigbox\relax}

    \subcaptionbox{\label{fig:simulation_aupit_ibs_combined}}{%
        \includegraphics[width=0.62\textwidth]{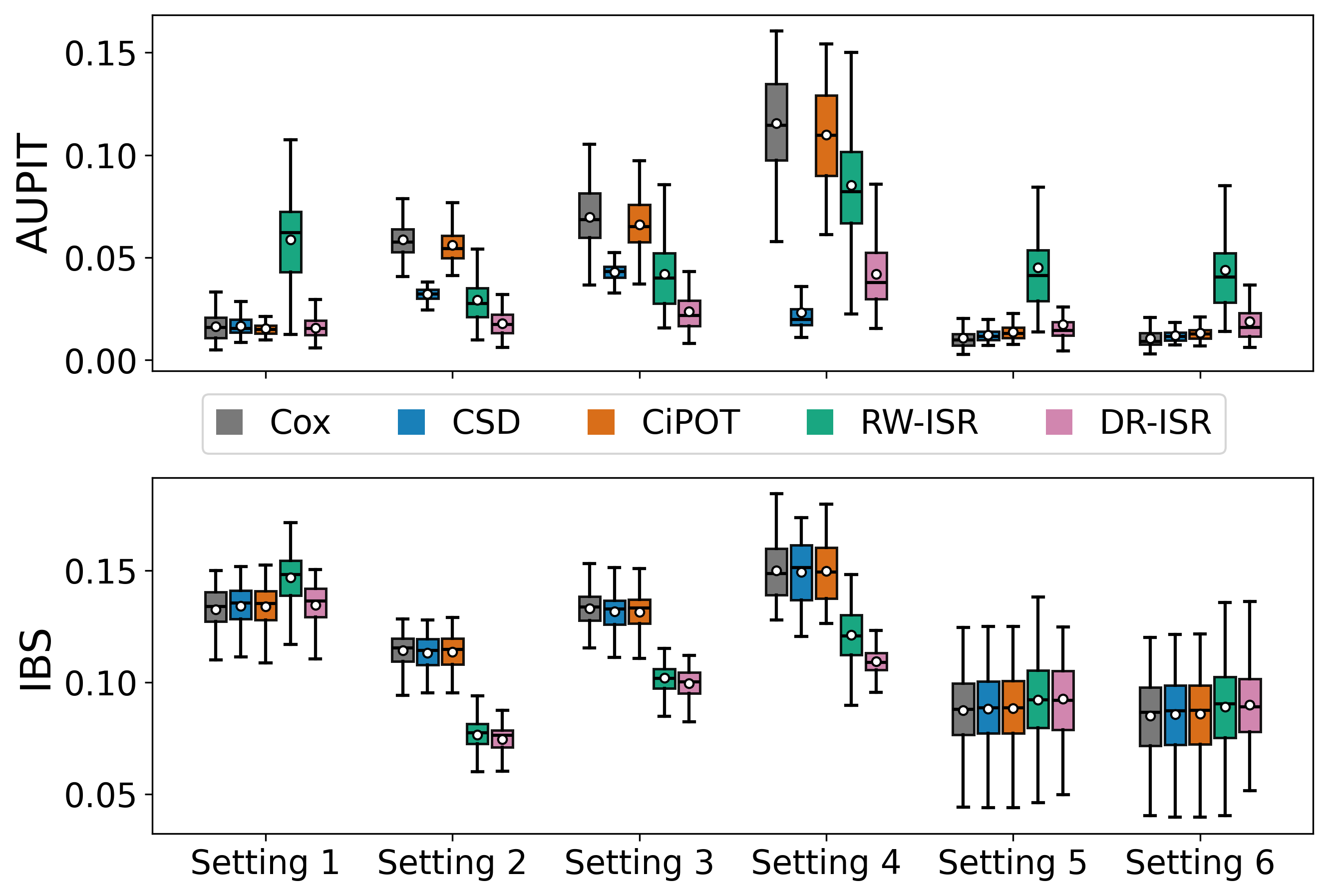}
    }
    \hfill
    \subcaptionbox{\label{fig:simulation_quantile}}{%
        \includegraphics[height=\leftfigheight]{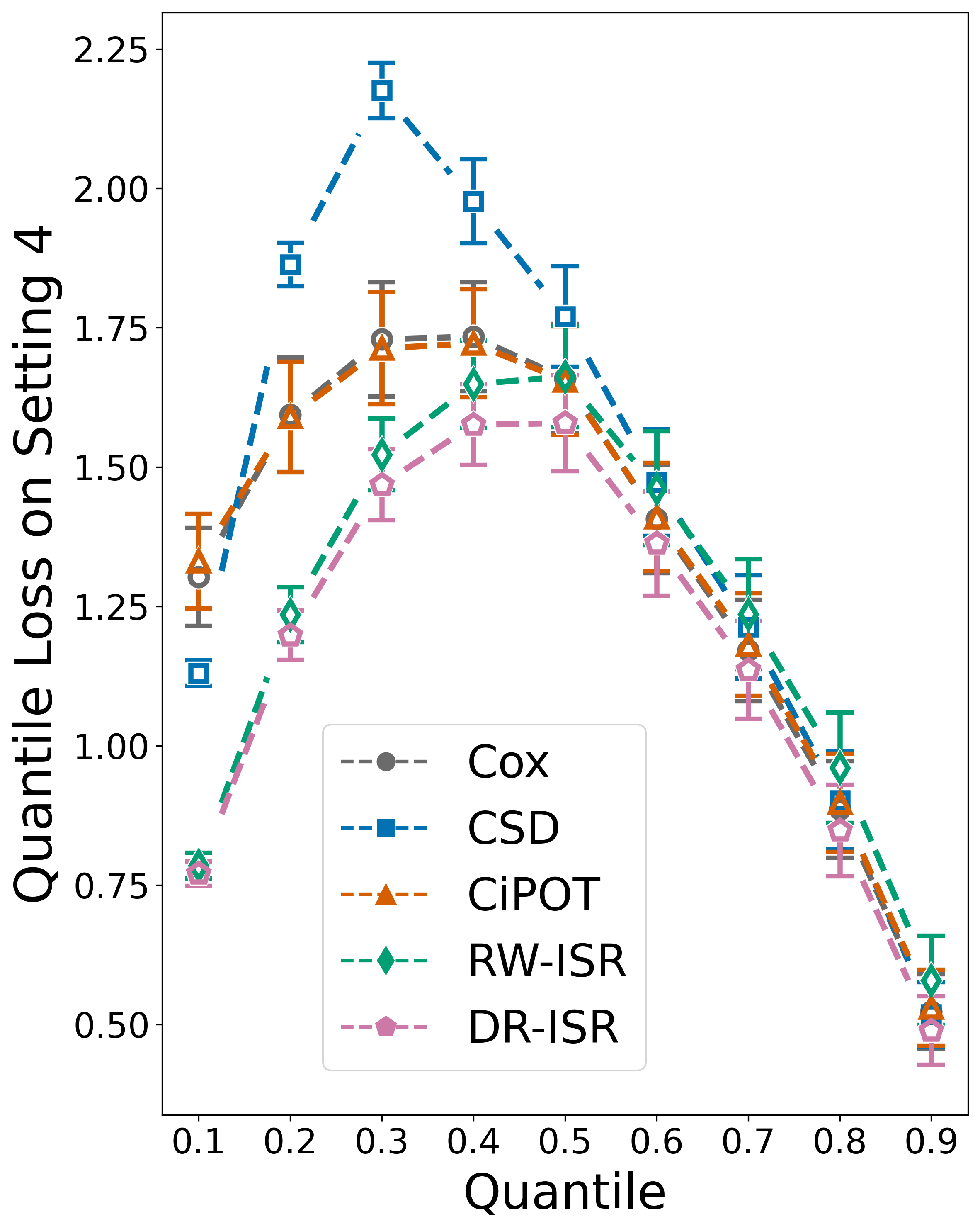}
    }

    \caption{AUPIT, IBS, and quantile loss across the six experimental settings. Boxplots in (\ref{fig:simulation_aupit_ibs_combined}) summarize the variation over seeds, where the center line indicates the median and the white dot indicates the mean. In (\ref{fig:simulation_quantile}), we show the quantile loss across nine levels for Setting 4. All simulations were repeated with 100 different random seeds, and error bars in (\ref{fig:simulation_quantile}) denote $\vphantom{}\pm 2$ standard errors.}
    \label{fig:simulation_combined}
\end{figure}

\subsection{Simulations}
\label{sec:simulations}

We first evaluate our method in synthetic settings where the true event times are known, including for individuals who are censored in the data. This allows us to compute oracle values of all metrics. 


We use the six settings from~\citep{Gui2024AdaptiveCutoffs}, where $\log T \mid X \sim \mathcal{N}(\mu(X), \sigma^2(X))$ and $C \mid X \sim P_{C \mid X}$. The first four simulations use a univariate covariate $X$, while Settings 5--6 use a multivariate covariate with dimension $10$. In each setting, we sample a dataset of size $5{,}000$ and split it evenly into training and calibration sets, as recommended in~\citep{Candes2023Conformalized}. We evaluate all metrics on an independent test sample of size $5{,}000$. Further details on the data-generating mechanisms are provided in Appendix~\ref{app:sec:simulations} reproduced from~\citep{Gui2024AdaptiveCutoffs}.

Because these simulation settings are relatively low-dimensional, we use a canonical linear Cox model as the initial survival estimator. We then apply each calibration method to the fitted survival curves and compare performance across the six settings.

Figure~\ref{fig:simulation_aupit_ibs_combined} summarizes the PIT calibration and distributional accuracy results. The DR-ISR estimator performs the best on Settings 2--3 and slightly worse than CSD in Setting 4. We also observe that DR-ISR is more stable and accurate compared to RW-ISR. On the IBS metric, both ISR methods yield substantial improvement in Settings 2--4. All estimators perform similarly in other settings, with the baseline methods providing little to no improvement over the uncalibrated Cox model on IBS.

We study threshold calibration more directly in Figure~\ref{fig:simulation_quantile} by evaluating the quantile score across nine levels in Setting 4, where the variation in IBS is the greatest. Both RW-ISR and DR-ISR achieve lower quantile loss compared to other methods across all levels, with DR-ISR performing better at medium quantiles. 
We include results for the quantile score on remaining settings in the Appendix.


\subsection{Real-data results}
\label{sec:real_data}

In recent years, powerful foundation models for unstructured medical data (clinical reports, histopathology slides, RNA sequencing etc.) have become available. A common approach is to then combine the features extracted from these foundation models with Deep Cox networks for downstream survival analysis tasks~\citep{hoptimus1, hemker2024healnet, Kang2025BulkFormer, lu2021wsi_foundation, zhang2025standardizing, vaidya2025molecular, ding2025multimodal, chen2024uni, lu2024avisionlanguage, Lu_2023_CVPR, wang2024head_and_neck_survival, zhao2021bertsurv, zhu2017wsisa}. We explore exactly this common paradigm in our real-data experiments.

We use The Cancer Genome Atlas dataset for our experiments which consists of about 10,000 cancer patients of which we use $60\%$ for model fitting, $20\%$ for calibration and $20\%$ for testing. We aim to predict the progression-free survival (PFS) which is the length of time during and after the treatment of a disease that a patient lives with the disease but it does not get worse. 

For our first experiment, we use the text-based pathology reports from~\citep{Kefelli2024TCGAreports}. To train our model, we first extract the final-layer embeddings from the BioClinical BERT foundation model~\citep{alsentzer2019BioBERT} as covariate features. We then train a three-layer MLP head as our Deep Cox model to predict risk scores for the PFS (as compiled by~\citep{Liu2024TCGAOutcomes}) by maximizing the Cox partial likelihood~\citep{Katzman2018DeepSurv}. 
A separate three-layer MLP head is trained to obtain the censoring model. The PFS and censoring models have C-indices of $0.643\pm0.002$ and $0.672\pm0.002$ respectively (mean over 100 seeds $\pm$ 2 standard error). In the second experiment, we use bulk RNA-seq data from~\citep{Vivian2017Toil} and the BulkFormer foundation model~\citep{Kang2025BulkFormer} to extract features, but otherwise keep the exact same training procedure. The PFS and censoring models have C-indices of $0.686\pm0.002$ and $0.565\pm0.002$ respectively.

We report performance using the IPCW-estimated metrics in Table~\ref{tab:real_data_results}, where the censoring weights are obtained from an uncalibrated Deep Cox model.
DR-ISR achieves the best performance on AUPIT, even improving over the baselines, which directly target PIT calibration. We believe this is because DR-ISR explicitly models the dependence of the censoring distribution on the covariates, whereas the CSD and CiPOT do not directly account for this structure. On IBS, we find that DR-ISR again achieves the lowest value among all methods. Decomposing IBS by quantile level shows that the gain is primarily at higher quantiles, while performance at lower quantiles is more mixed.

The real-data results suggest that our approach is robust in practice. Even when our assumptions are imperfectly satisfied, ISR still offers an advantage over existing post hoc calibration methods.


\begin{table}[t]
  \caption{IPCW estimates of AUPIT, IBS, and quantile loss at the $10\%,50\%$ and $90\%$ levels over the test set for the pathology and RNA-seq model. Values are mean over 100 seeds $\vphantom{}\pm2$ standard error.}
  \vspace{2pt}
  \label{tab:real_data_results}
  \centering
  \renewcommand{\arraystretch}{0.90}
  { \footnotesize
  \begin{tabular}{lcccccc}
    \toprule
    & Method    &  AUPIT      & IBS & $10\%$ Quantile & $50\%$ Quantile & $90\%$ Quantile \\
    \midrule
    \multirow{5}{*}{\rotatebox[origin=c]{90}{Pathology}} 
& Deep Cox & $0.0665 \pm 0.0033$ & $0.184 \pm 0.0068$ & $266 \pm 7.3$ & $936 \pm 31$ & $580 \pm 61$ \\
& CSD & $0.0720 \pm 0.0056$ & $0.185 \pm 0.0069$ & $266 \pm 7.3$ & $944 \pm 32$ & $516 \pm 55$ \\
& CiPOT & $0.0627 \pm 0.0034$ & $0.184 \pm 0.0069$ & $263 \pm 7.1$ & $\textbf{932} \pm 31$ & $562 \pm 59$ \\
& RW-ISR (ours) & $0.0761 \pm 0.0054$ & $0.189 \pm 0.0082$ & $\textbf{262} \pm 7.0$ & $986 \pm 36$ & $529 \pm 58$ \\
& DR-ISR (ours) & $\textbf{0.0538} \pm 0.0035$ & $\textbf{0.182} \pm 0.0070$ & $264 \pm 7.1$ & $950 \pm 35$ & $\textbf{504} \pm 50$ \\
    \midrule
    \multirow{5}{*}{\rotatebox[origin=c]{90}{RNA-seq}} 
& Deep Cox & $0.0371 \pm 0.0020$ & $0.154 \pm 0.0053$ & $\textbf{252} \pm 5.3$ & $971 \pm 27$ & $739 \pm 56$ \\
& CSD & $0.0478 \pm 0.0038$ & $0.155 \pm 0.0056$ & $\textbf{252} \pm 5.3$ & $971 \pm 28$ & $766 \pm 58$ \\
& CiPOT & $0.0358 \pm 0.0020$ & $0.154 \pm 0.0054$ & $\textbf{252} \pm 5.3$ & $966 \pm 27$ & $733 \pm 56$ \\
& RW-ISR (ours) & $0.0698 \pm 0.0055$ & $0.156 \pm 0.0071$ & $253 \pm 5.3$ & $974 \pm 29$ & $797 \pm 63$ \\
& DR-ISR (ours) & $\textbf{0.0345} \pm 0.0024$ & $\textbf{0.152} \pm 0.0055$ & $253 \pm 5.3$ & $\textbf{958} \pm 28$ & $\textbf{696} \pm 56$ \\
    \bottomrule
  \end{tabular}}
  \vspace{-5pt}
\end{table}
\section{Discussion}
\label{sec:discussion}

In this work, we propose a new algorithm for calibrating survival functions under right censoring using Isotonic Survival Regression (ISR). Our method extends principled pointwise survival probability estimates to valid, monotone survival functions while preserving discriminative performance. Theoretically, we establish calibration and consistency guarantees for the proposed procedure. 
Across simulation and real-data experiments, we find that DR-ISR gives the strongest performance when calibrating both linear Cox and Deep Cox models.
We believe our calibration algorithm has utility across a wide range of biomedical applications, where Cox models are broadly relevant.

While this work primarily focuses on calibrating the survival function, our procedure works analogously for the censoring function. 
One could consider iterative schemes that alternate between calibrating the survival and censoring models. 
More generally, how best to calibrate both mechanisms simultaneously remains an open question. We leave these extensions for future work.


\paragraph{Limitations.} As with many survival analysis methods, our algorithm relies on assumptions that may be difficult to verify in practice. Nonetheless, our real-data experiments suggest that the proposed method is moderately robust to violations of these assumptions. Like some survival analysis methods, our approach relies on a partial ordering of individuals by way of a univariate risk score. While simplistic, we believe this modeling choice is useful in many applications, and it is supported by the prevalence and practical success of both linear and Deep Cox models.

\paragraph{Broader Impact.} Our contributions provide a practical algorithm for producing better calibrated survival functions, which can support more reliable risk assessment and downstream decision-making. We do not foresee any direct negative impacts of this work.


\newpage

\section*{Acknowledgments}
S.B.\ thanks Jing Lei for his insightful comments and discussion. The authors thank Aziz Ayed, Serena Khoo, Gabrielle Cohn and Regina Barzilay for their guidance regarding the TCGA data. The authors were supported by the ARPA-H ADAPT program.

\bibliography{references}
\bibliographystyle{plainnat}


\clearpage

\vbox{%
  \hsize\textwidth
  \linewidth\hsize
  \vskip 0.1in

  \hrule height 4pt
  \vskip 0.25in
  \vskip -\parskip

  \centering
  {\LARGE\bfseries Supplementary Material for \\ Isotonic Survival Regression: Calibrated Survival Distributions from Deep Cox Models \par}

  \vskip 0.29in
  \vskip -\parskip
  \hrule height 1pt
  \vskip 0.09in
}

\appendix

\section{Proofs}
\label{app:sec:proofs}

We first establish a key result used in our arguments: under the risk-ordering assumption, the survival function can be expressed solely in terms of the risk score produced by the uncalibrated model.

\begin{lemma}[Risk Score Sufficiency] \label{lemma:risk_score_sufficiency}
    Under assumption~\ref{ass:risk_correctness}, the event time depends on the covariates only through the risk score $r(X)$. Equivalently,
    \[ \mathbb{P}(T > t \mid X) = \mathbb{P}(T > t \mid r(X)), \quad \text{for all } t \geq 0. \]
\end{lemma}

\begin{proof}
    Let $x, x' \in \mathcal{X}$ such that $r(x) = r(x')$. Then, we have $r(x)\geq r(x')$ and $r(x')\geq r(x)$. Applying assumption~\ref{ass:risk_correctness} twice gives $S(t \mid x) = S(t \mid x')$. By the tower property,
    \begin{align*}
        \mathbb{P}(T > t \mid r(X)) &= \mathbb{E}[\mathbf{1}[T > t] \mid r(X)] \\
        &= \mathbb{E}[\mathbb{E}[\mathbf{1}[T > t] \mid X] \mid r(X)] \\
        &= \mathbb{E}[S(t \mid X) \mid r(X)] \\
        &= S(t \mid X).
    \end{align*}
\end{proof}

Lemma~\ref{lemma:risk_score_sufficiency} states that the survival function can be expressed in terms of the risk. Thus, for ease of exposition, we write $\Sor = \mathbb{P}(T>t\mid r(X)=r)$. We also introduce notation for the nuisance parameters $\eta = (S, G)$, with true value $\eta_0 = (S_0, G_0)$ and estimated value $\widehat{\eta} = (\widehat S, \widehat G)$. Our proofs will make use of the following additional regularity assumptions as specified in the corresponding theorem statements.

\begin{enumerate}[label=S\arabic*., ref=S\arabic*]
    \item \label{ass:positivity} We assume that there exists constants $0 < c_1(t) < \infty$ and $0 < c_2(t) < \infty$ depending on $t$ such that $S_0(t \mid x) \:\wedge\: \widehat{S}(t \mid x) > 1 / c_1(t)$ and $G_0(t \mid x) \:\wedge\: \widehat{G}(t \mid x) > 1 / c_2(t)$ for all $x$.
    \item \label{ass:continuous_density} The covariates $X\stackrel{\text{i.i.d.}}{\sim}P_X$ induce a distribution for the risks $r(X)\stackrel{\text{i.i.d.}}{\sim}P$. We assume $P$ has bounded support $[a,b]$ and a Lebesgue density $p$ which is strictly positive over the support.
    \item \label{ass:continous_survival} We assume the survival function in terms of risk $\Sor$ is continuous in $r$ at all times $t$.
\end{enumerate}

Assumption~\ref{ass:positivity} is a pointwise-in-time positivity condition. At each fixed time $t$, the survival and censoring probabilities are uniformly bounded away from zero over $x$, but the lower bound may vanish as $t \to \infty$.

\subsection{Proof of Theorem 1}
We first discuss the augmented IPCW (AIPCW) estimating equation for the conditional survival probability using standard semiparametric analysis~\citep{Robins1992DependentCensoring,Tsiatis2006Semiparametric}. In particular, we will demonstrate that the AIPCW estimator has two equivalent representations using either the censoring or event martingale difference. Next, we will show that the two bias bounds hold separately.
\[ B(t, x; \widehat{\eta}) = |\mathbb{E}[\tilde{S}^\mathrm{DR}(t \mid x; \widehat{\eta}) \mid X = x] - S_0(t\mid x)| \leq B_i(t, x; \widehat{\eta}). \]
We use the event martingale representation to get the bound with $B_1$ and the censoring martingale representation to get the bound with $B_2$. The full result then follows by taking the minimum bound between $B_1$ and $B_2$. 

For simplicity and notational convenience, we will assume all of the survival and censoring functions ($S_0$, $G_0$, $\smash{\widehat{S}}$, and $\smash{\widehat{G}}$) are continuous; the general case follows by using the appropriate left-limits. We will first prove two useful identities regarding the survival and censoring probabilities.

\begin{lemma} \label{lemma:dS_to_dLambda}
    For any time $t > 0$ and continuous survival function $S$, 
    \[ d \left( \frac{1}{S(t\mid X)} \right) = \frac{1}{S(t\mid X)}\,d\Lambda_{T\mid X}(t\mid X), \]
    where $\Lambda_{T\mid X}(t\mid X) = -\log S(t \mid X)$ is the cumulative hazard function. The same result holds when replacing $S$ with the censoring model equivalent $G$.
\end{lemma}

\begin{proof}
    We differentiate $\frac{1}{S(u \mid X)} = \exp \{ \Lambda_{T \mid X} (u \mid X) \}$ with respect to $\Lambda_{T \mid X}$. Then,
    \[ d \left( \frac{1}{S(u \mid X)} \right) = \exp \{ \Lambda_{T \mid X} (u \mid X) \} d\Lambda_{T \mid X} (u \mid X) = \frac{1}{S(u \mid X)} d\Lambda_{T \mid X} (u \mid X). \]
    The proof is identical for any censoring model $G$ using the cumulative hazard function $\Lambda_{C \mid X}$.
\end{proof}

\begin{lemma} \label{lemma:ratio_to_integral}
    For any time $t > 0$ and pair of continuous survival functions $S$ and $S'$, 
    \[ S(t\mid X) - S'(t\mid X) = S'(t \mid X) \int_0^t \frac{S(u\mid X)}{S'(u\mid X)}\left(d\Lambda'_{T\mid X}(u\mid X) - d\Lambda_{T\mid X}(u\mid X)\right), \]
    where $\Lambda_{T \mid X} (u \mid X) = -\log S(u \mid X)$ and $\Lambda'_{T \mid X} (u \mid X) = -\log S'(u \mid X)$ are the respective cumulative hazard functions. The same result holds when replacing $S$ and $S'$ with their censoring model equivalents $G$ and $G'$, respectively.
\end{lemma}

\begin{proof}
    We express the ratio of survival probabilities in terms of the cumulative hazard.
    \[ \frac{S(u \mid X)}{S'(u \mid X)} = \exp \left\{ \Lambda'_{T \mid X} (u \mid X) - \Lambda_{T \mid X} (u \mid X) \right\}. \]
    Taking the derivative, we have
    \[ d \left( \frac{S(u \mid X)}{S'(u \mid X)} \right) = \frac{S(u \mid X)}{S'(u \mid X)} \left( d\Lambda'_{T \mid X} (u \mid X) - d\Lambda_{T \mid X} (u \mid X) \right). \] 
    We integrate and use $S(0 \mid X) = S'(0 \mid X) = 1$ to get
    \[ \frac{S(t \mid X)}{S'(t \mid X)} - 1 = \int_0^t \frac{S(u\mid X)}{S'(u\mid X)} \left(d\Lambda'_{T\mid X}(u\mid X) - d\Lambda_{T\mid X}(u\mid X)\right). \]
    Multiplying both sides by $S'(t \mid X)$ gives the desired result. The proof is identical for any pair of censoring models $G$ and $G'$ using the cumulzative hazard functions $\Lambda_{C \mid X}$ and $\Lambda'_{C \mid X}$.
\end{proof}

\subsubsection{AIPCW estimators of the survival probability}

We now analyze AIPCW estimator for the conditional survival probability. For a fixed $t$, the target can be written as the conditional mean of the complete-data outcome $\mathbf{1}[T > t]$. Let
\[ \beta(x) = \mathbb{P}(T > t \mid X = x) = \mathbb{E}[\mathbf{1}[T > t] \mid X = x]. \]
This perspective allows us to use standard semiparametric results for AIPCW estimators, including their efficiency and robustness properties. At a high level, right censoring is a monotone coarsening mechanism, and the AIPCW estimator can be expressed using the martingale difference process associated with the censoring mechanism. We refer the reader to~\citet{Robins1992DependentCensoring} and~\citet{Tsiatis2006Semiparametric} for additional details on semiparametric theory and AIPCW estimation.

Define the counting process $dN_C(u) = \mathbf{1}[Y \in du, \delta = 0]$ and censoring martingale increment
\[ dM_{C \mid X} (u \mid X) = dN_C(u) - \mathbf{1}[Y \geq u] \, d\Lambda_{C \mid X} (u \mid X), \] 
where $\Lambda_{C \mid X} (u \mid X) = -\log G (u \mid X)$ is the censoring analog of the cumulative hazard function. Following the construction in Section 9.3 of~\citet{Tsiatis2006Semiparametric}, the AIPCW estimation equation for $\beta$ is
\begin{equation} \label{eqn:aipcw_general}
    \mathbb{E} \bigg[ \underbrace{\frac{\delta (\mathbf{1}[T > t] - \beta(X))}{G(Y \mid X)}}_\text{IPCW estimate} + \underbrace{\int_0^\infty \frac{dM_{C \mid X} (u \mid X)}{G (u \mid X)} h(u, X)}_\text{augmentation} \ \ \Big| \ \ X = x \bigg] = 0.
\end{equation}
The first term is the IPCW estimate of $\beta(X)$, while the second augmentation term lies in the censoring augmentation space $\mathcal{A}$ for any square integrable function $h \in L^2$. The choice of $h$ determines the efficiency of the resulting estimator. By Theorem 10.4 of~\citet{Tsiatis2006Semiparametric}, the efficient choice is such that the augmentation is the projection of the IPCW estimating function onto $\mathcal{A}$, which for survival probability estimation is
\[ h(u, X) = \mathbb{E}[\mathbf{1}[T > t] - \beta(X) \mid X, \; T \geq u]. \]
We simplify \(h(u,X)\) in terms of the conditional survival function. Since $\beta(X)=S(t\mid X)$,
\begin{align*}
    h(u,X)
    &= \mathbb{E}\left[\mathbf{1}(T > t)-\beta(X)\mid X,\;T \geq u\right] \\
    &= \mathbb{P}(T>t\mid X,\;T \geq u)-\beta(X) \\
    &= \frac{\mathbb{P}(T>t,\;T \geq u\mid X)}{\mathbb{P}(T \geq u\mid X)}-\beta(X) \\
    &= \frac{\mathbb{P}(T>\max\{u,t\}\mid X)}{\mathbb{P}(T \geq u\mid X)}-\beta(X) \\
    &= \frac{S(\max\{u,t\}\mid X)}{S(u\mid X)}-\beta(X).
\end{align*}
Substituting into Equation~\ref{eqn:aipcw_general}, we obtain the observed-data estimating equation
\[ \beta(x) = \mathbb{E} \left[ \frac{\delta \mathbf{1}[Y > t]}{G(Y \mid X)} + \int_0^\infty \frac{S(\max\{u,t\}\mid X)}{S(u \mid X) G(u \mid X)}dM_{C \mid X} (u \mid X) \ \ \bigg| \ \ X = x \right] \Big/ W(x), \]
where 
\[ W(x) = \mathbb{E} \left[\frac{\delta}{G(Y\mid X)} + \int_0^\infty \frac{dM_{C\mid X}(u\mid X)}{G(u\mid X)} \ \  \bigg| \ \ X = x \right]. \]
We now show that $W(x) = 1$. Since $dN_C(u)$ increments at $Y$ if and only if $\delta = 0$,
\begin{align*}
    W(x) &= \mathbb{E}\left[\frac{\delta}{G(Y\mid X)}+\int_0^\infty \frac{1}{G(u\mid X)}\,dN_C(u)-\int_0^\infty \frac{\mathbf{1}[Y\ge u]}{G(u\mid X)}\,d\Lambda_{C\mid X}(u\mid X)\;\bigg|\;X=x\right] \\
    &= \mathbb{E}\left[\frac{\delta}{G(Y\mid X)}+\frac{1-\delta}{G(Y\mid X)}-\int_0^Y \frac{d\Lambda_{C\mid X}(u\mid X)}{G(u\mid X)} \;\bigg|\;X=x\right] \\
    &= \mathbb{E}\left[\frac{\delta}{G(Y\mid X)}+\frac{1-\delta}{G(Y\mid X)}-\int_0^Y d \left( \frac{1}{G(u\mid X)} \right) \;\bigg|\;X=x\right] \hspace{24pt} \tag*{(Lemma~\ref{lemma:dS_to_dLambda})} \\
    &= \mathbb{E}\left[\frac{1}{G(Y\mid X)}-\left\{\frac{1}{G(Y\mid X)}-1\right\}\;\bigg|\;X=x\right] = 1.
\end{align*}
Thus, we have the AIPCW estimating equation
\[ \beta(x) = \mathbb{E} \left[ \frac{\delta \mathbf{1}[Y > t]}{G(Y \mid X)} + \int_0^\infty \frac{S(\max\{u,t\}\mid X)}{S(u \mid X) G(u \mid X)}dM_{C \mid X} (u \mid X) \ \bigg| \ X = x \right]. \]

\subsubsection{Censoring and event martingale representations for AIPCW estimators}
The construction in Tsiatis for the efficient AIPCW estimator uses the censoring martingale difference $dM_{C \mid X}$. However, for our doubly robust pseudo-outcomes in Equation~\ref{eq:dr_pseudo_label}, we use the event martingale difference $dM_{T \mid X}$. We will now show that these representations are equivalent.

\begin{lemma}
    \normalsize{Suppose assumption~\ref{ass:independent_censoring} holds. For any time $t > 0$ and continuous functions $S$ and $G$, the censoring and event martingale representations of the AIPCW estimator are equivalent, i.e.}
    \begin{align*}
        \frac{\mathbf{1}[Y > t]}{{G(Y \mid X)}} + \int_0^\infty \frac{S(\max \{ u, t \} \mid X)}{S(u \mid X) G(u \mid X)} \,dM_{C \mid X} &(u \mid X) = \phantom{.} \\
        &S(t \mid X) - \int_0^t \frac{S(t \mid X)}{S(u \mid X) G(u \mid X)} \,dM_{T \mid X} (u \mid X).
    \end{align*}
\end{lemma}

\begin{proof} 
Define $H(u) = \frac{S(t \mid X)}{S(u \mid X) G(u \mid X)}$ for $0 \leq u \leq t$. Then,
\[ dH(u) = H(u) \left( d\Lambda_{T \mid X} (u \mid X) + d\Lambda_{C \mid X} (u \mid X) \right). \]
Define the at-risk indicator $R(u) = \mathbf{1}[Y \geq u]$, with $dR(u) = -dN_T(u) - dN_C(u)$. Then,
\begin{align*}
    d[H(u) R(u)] &= H(u) R(u) \left( d\Lambda_{T \mid X} (u \mid X) + d\Lambda_{C \mid X} (u \mid X) \right) - H(u) \left( dN_T(u) + dN_C(u) \right)\\
    &= H(u) \left( -dM_{T \mid X} (u \mid X) - dM_{C \mid X} (u \mid X) \right).
\end{align*}
Integrating and using $S(0 \mid X) = G(0 \mid X) = R(0) = 1$ gives 
{\small \[ -\int_0^t H(u) dM_{T \mid X} (u \mid X) - \int_0^t H(u) dM_{C \mid X} (u \mid X) = H(t) R(t) - H(0) R(0) = \frac{\mathbf{1}[Y > t]}{G(t \mid X)} - S(t \mid X). \]
}Rearranging the result and expressing in terms of $S$ and $G$ gives,
{\small  \begin{equation} \label{eqn:eif_equiv_term1}
    \frac{\mathbf{1}[Y > t]}{G(t \mid X)} - S(t \mid X) + \int_0^t \frac{S(t \mid X)}{S(u \mid X) G(u \mid X)} \,dM_{C \mid X} (u \mid X) = - \int_0^t \frac{S(t \mid X)}{S(u \mid X) G(u \mid X)} dM_{T \mid X} (u \mid X).
\end{equation}
}We now analyze the component of the censoring martingale integral from $t$ to $\infty$. Using Lemma~\ref{lemma:dS_to_dLambda},
\begin{align*}
    d \left( \frac{R(u)}{G(u \mid X)} \right) &= -\frac{dN_T(u) + dN_C(u)}{G(u \mid X)} + \frac{R(u)}{G(u \mid X)} d\Lambda_{C \mid X} (u \mid X). \\
    &= -\frac{dM_{C \mid X} (u \mid X)}{G(u \mid X)} - \frac{dN_T(u)}{G(u \mid X)}.
\end{align*}
By Assumption~\ref{ass:positivity}, we have $G(Y \mid X) > 0$, so
\[ -\int_t^\infty \frac{dM_{C \mid X} (u \mid X)}{G(u \mid X)} - \int_t^\infty \frac{dN_T(u)}{G(u \mid X)} = \frac{R(u)}{G(u \mid X)} \,\bigg|_t^\infty = -\frac{\mathbf{1}[Y > t]}{G(t \mid X)}. \]
Since $dN_T(u)$ increments at $u = Y$ only when $Y > t$ and $\delta = 1$,
\[ \int_t^\infty \frac{dN_T(u)}{G(u \mid X)} = \frac{\delta \mathbf{1}[Y > t]}{G(Y \mid X)}. \]
Thus, we have that
\begin{equation} \label{eqn:eif_equiv_term2}
    \frac{\mathbf{1}[Y > t]}{G(t \mid X)} = \frac{\delta \mathbf{1}[Y > t]}{G(Y \mid X)} + \int_t^\infty \frac{dM_{C \mid X} (u \mid X)}{G(u \mid X)}.
\end{equation}
Since,
\[ { \renewcommand{\arraystretch}{1.5} 
    \frac{S(\max \{u, t\} \mid X)}{S(u \mid X) G(u \mid X)} = 
    \left\{ \begin{array}{c l}
        \frac{S(t \mid X)}{S(u \mid X) G(u \mid X)} & u < t \\
        \frac{1}{G(u \mid X)} & u \geq t,
    \end{array} \right.
} \]
substituting Equation~\ref{eqn:eif_equiv_term2} into Equation~\ref{eqn:eif_equiv_term1} gives the desired equivalence.
\end{proof}

The equivalence between the censoring and event martingale representation holds for the plug-in AIPCW estimator using $\widehat{S}$ and $\widehat{G}$ as well under our continuity simplification. We will use the two different expressions to derive separate bounds for the bias.

\subsubsection{Event martingale bound}
Define the target functional
\[ \psi(t,x;\eta) = \mathbb{E} \big[ \tilde{S}^\text{DR}(t \mid x; \widehat{\eta}) \ \big| \ X = x \big]. \]
The bias of the pseudo-outcome is
\[ B(t,x;\widehat{\eta}) = |\psi(t,x;\widehat{\eta}) - S_0(t \mid x)|. \]
Using the event martingale representation,
\[ B(t,x;\widehat{\eta}) = \left| \mathbb{E} \left[ \widehat{S}(t \mid X) - S_0(t \mid X) - \int_0^t \frac{\widehat{S}(t \mid X)}{\widehat{S}(u \mid X) \widehat{G}(u \mid X)} d\widehat{M}_{T \mid X} (u \mid X) \ \Big| \ X = x \right] \right|. \]
Note that the only randomness in this expression is in the event process difference $d\widehat{M}_{T \mid X}$. Let $\Lambda_{0, T \mid X} (t \mid X) = -\log S_0(t \mid X)$ and $\widehat{\Lambda}_{T \mid X} (t \mid X) = -\log \widehat{S} (t \mid X)$ be the true and estimated cumulative hazard functions, respectively. Under assumption~\ref{ass:independent_censoring},
\begin{align*}
    \mathbb{E}\left[dN_T(u)\mid X = x \right] &= \mathbb{P}(Y \in du,\delta = 1\mid X = x) \\
    &= \mathbb{P}(T \in du, \, C \geq u\mid X = x) \\
    &= \mathbb{P}(C \geq u\mid X = x) \, \mathbb{P}(T \in du\mid X = x) \\
    &= S_0(u\mid x) \, G_0(u\mid x) \, d\Lambda_{0,T\mid X}(u\mid x). \\
    \mathbb{E}\left[\mathbf{1}[Y\ge u]\mid X = x \right] &= \mathbb{P}(Y\ge u\mid X = x) \\
    &= \mathbb{P}(T\ge u,C\ge u\mid X = x) \\
    &= S_0(u\mid x) \, G_0(u\mid x).
\end{align*}
Thus, we have that 
\[ \mathbb{E} \big[ d\widehat{M}_{T\mid X}(u\mid X)\mid X = x \big] = S_0(u\mid x)G_0(u\mid x)\left(d\Lambda_{0,T\mid X} (u \mid x) - d\widehat{\Lambda}_{T\mid X}(u\mid x) \right). \]
Substituting this result into the bias expression gives,
\begin{align*}
    B(t,x;\widehat{\eta}) = \bigg| \widehat{S}(t\mid x) &- S_0(t\mid x) \\
    &- \int_0^t \frac{\widehat{S}(t\mid x) S_0(u\mid x) G_0(u\mid x)}{\widehat{S}(u\mid x) \widehat{G}(u\mid x)} \left(d\Lambda_{0, T\mid X}(u\mid x) - d\widehat{\Lambda}_{T\mid X}(u\mid x) \right) \bigg|.
\end{align*}
Using Lemma~\ref{lemma:ratio_to_integral}, we have
\[ B(t,x;\widehat{\eta}) = \left| \int_0^t \frac{\widehat{S}(t\mid x) S_0(u\mid x)}{\widehat{S}(u\mid x)}\left(1 - \frac{G_0(u\mid x)}{\widehat{G}(u\mid x)}\right)\left(d\Lambda_{0,T\mid X}(u\mid x) - d\widehat{\Lambda}_{T\mid X}(u\mid x) \right) \right|. \]
Under Assumption~\ref{ass:positivity}, there exists some constant $c_1(t)$ dependent on $t$ such that $\widehat{S}(t \mid x) > 1 / c_1(t)$.
Then, we have the bound
\[ 0 \leq \frac{\widehat{S}(t\mid x)S_0(u\mid x)G_0(u\mid x)}{\widehat{S}(u\mid x)} \leq c_1(t), \quad \text{for all } u \in [0,t]. \]
Therefore, we have that $B(t,x;\widehat{\eta}) \leq B_1 (t, x; \widehat{\eta})$, where 
\[ B_1 (t, x; \widehat{\eta}) = c_1(t) \int_0^t \left|\frac{1}{G_0\vphantom{\widehat{G}}(u\mid x)} - \frac{1}{\widehat{G}(u \mid x)}\right|\left|d\Lambda_{0,T\mid X}(u\mid x) - d\widehat{\Lambda}_{T\mid X}(u\mid x)\right|. \]

\subsubsection{Censoring martingale bound}

We now analyze the bias using the censoring martingale representation. Recall the bias of the pseudo-outcome is 
\[ B(t,x;\widehat{\eta}) = |\psi(t,x;\widehat{\eta}) - S_0(t \mid x)|. \]
Using the censoring martingale representation,
\[ B(t,x;\widehat{\eta}) = \left| \mathbb{E} \left[ \frac{\delta \mathbf{1}[Y > t]}{\widehat{G}(Y \mid x)} - S_0(t \mid x) + \int_0^\infty \frac{\widehat{S}(\max\{u,t\}\mid x)}{\widehat{S}(u \mid x) \widehat{G}(u \mid x)}d\widehat{M}_{C \mid X} (u \mid x) \ \bigg| \ X = x \right] \right|, \]
where the randomness is in the IPCW estimation and censoring process difference $d\widehat{M}_{C \mid X} (u \mid X)$. Let $\Lambda_{0, C \mid X} (t \mid X) = -\log G_0(t \mid X)$ and $\widehat{\Lambda}_{C \mid X} (t \mid X) = -\log \widehat{G} (t \mid X)$ be the true and estimated cumulative hazard functions for censoring, respectively. Under assumption~\ref{ass:independent_censoring},
\begin{align*}
    \mathbb{E}\left[dN_C(u)\mid X\right] &= \mathbb{P}(Y \in du,\delta = 0\mid X) \\
    &= \mathbb{P}(C \in du,T \geq u\mid X) \\
    &= S_0(u\mid X)G_0(u\mid X)\,d\Lambda_{0,C\mid X}(u\mid X). \\
    \mathbb{E}\left[\mathbf{1}[Y\ge u]\mid X\right] &= \mathbb{P}(Y\ge u\mid X) \\
    &= \mathbb{P}(T\ge u,C\ge u\mid X) \\
    &= S_0(u\mid X)G_0(u\mid X).
\end{align*}
Thus,
\[ \mathbb{E}\left[d\widehat{M}_{C\mid X}(u\mid X)\mid X\right] = S_0(u\mid X)G_0(u\mid X)\left(d\Lambda_{0,C\mid X}(u\mid X) - d\widehat{\Lambda}_{C\mid X}(u\mid X)\right). \]
We can write the bias as
\begin{align*}
    B(t,X;\widehat{\eta}) &= \Bigg| \mathbb{E}\left[\frac{\delta \mathbf{1}[Y > t]}{\widehat{G}(Y\mid X)} + \int_t^\infty \frac{1}{\widehat{G}(u\mid X)}\,d\widehat{M}_{C\mid X}(u\mid X) \ \bigg| \ X = x \right] - S_0(t\mid X) \\
    &\phantom{=====}+ \int_0^t \frac{\widehat{S}(t\mid x) S_0(u\mid x)G_0(u\mid x)}{\widehat{S}(u\mid x)\widehat{G}(u\mid x)}\left(d\Lambda_{0,C\mid X}(u\mid x) - d\widehat{\Lambda}_{C\mid X}(u\mid x)\right) \Bigg|.
\end{align*}

We first simplify the expectation. Using Equation~\ref{eqn:eif_equiv_term2}, we have
\[ \frac{\mathbf{1}[Y>t]}{\widehat{G}(t\mid X)} = \frac{\delta\mathbf{1}[Y>t]}{\widehat{G}(Y\mid X)}+\int_t^\infty \frac{1}{\widehat{G}(u\mid X)}\,d\widehat{M}_{C\mid X}(u\mid X). \]

Taking conditional expectations yields
\[ \mathbb{E}\left[\frac{\delta \mathbf{1}[Y > t]}{\widehat{G}(Y\mid X)} + \int_t^\infty \frac{1}{\widehat{G}(u\mid X)}\,d\widehat{M}_{C\mid X}(u\mid X) \ \bigg| \  X = x \right] = \frac{S_0(t\mid x)G_0(t\mid x)}{\widehat{G}(t\mid x)}. \]
Therefore,
\begin{align*}
    B(t,x;\widehat{\eta}) &= \Bigg| S_0(t\mid x)\left(\frac{G_0(t\mid x)}{\widehat{G}(t\mid x)} - 1\right) \\
    &\phantom{======}+ \int_0^t \frac{\widehat{S}(t\mid x)S_0(u\mid x)G_0(u\mid X)}{\widehat{S}(u\mid x)\widehat{G}(u\mid x)}\left(d\Lambda_{0,C\mid X}(u\mid x) - d\widehat{\Lambda}_{C\mid X}(u\mid x)\right) \Bigg|.
\end{align*}
Using Lemma~\ref{lemma:ratio_to_integral},
\[ S_0(t\mid x)\left(\frac{G_0(t\mid x)}{\widehat{G}(t\mid x)} - 1\right) = -S_0(t\mid x)\int_0^t \frac{G_0(u\mid x)}{\widehat{G}(u\mid x)}\left(d\Lambda_{0,C\mid X}(u\mid x) - d\widehat{\Lambda}_{C\mid X}(u\mid x)\right). \]
Substituting this identity into the expression for the bias gives
{\small \[ B(t,x;\widehat{\eta}) = \left| \int_0^t \frac{G_0(u\mid x)}{\widehat{G}(u\mid x)}\left(\frac{\widehat{S}(t\mid x)S_0(u\mid x)}{\widehat{S}(u\mid x)} - S_0(t\mid x)\right)\left(d\Lambda_{0,C\mid X}(u\mid x) - d\widehat{\Lambda}_{C\mid X}(u\mid x)\right) \right|. \]
}Under Assumption~\ref{ass:positivity}, there exists some constant $c_2(t)$ depending on $t$ such that $\widehat{G}(t \mid x) > 1 / c_2(t)$. Then, we have the bound
\[ 0 \leq \frac{S_0(u\mid x)G_0(u\mid x)}{\widehat{G}(u\mid x)} \leq c_2(t), \quad \text{for all } u \in [0,t]. \]
Therefore, we have that $B(t,x;\widehat{\eta}) \leq B_2(t,x;\widehat{\eta})$, where
\[ B_2(t,x;\widehat{\eta}) = c_2(t) \int_0^t \left|\frac{\widehat{S}(t\mid x)}{\widehat{S}(u\mid x)} - \frac{S_0(t\mid x)}{S_0(u\mid x)}\right|\left|d\Lambda_{0,C\mid X}(u\mid x) - d\widehat{\Lambda}_{C\mid X}(u\mid x)\right|. \]

\subsubsection{Pointwise convergence of bias}

The mixed-product bounds imply pointwise convergence of the bias under local uniform consistency of either nuisance parameter. We formalize this in the following lemma.

\begin{lemma}
\label{lemma:bias_convergence}
    Suppose the assumptions of Theorem~\ref{thm:dr} hold. Fix a finite time $t$ and a covariate value $x$ and let the bias of the pseudo-outcome be
    \[ B(t, x; \widehat{\eta}) = |\mathbb{E}[ \tilde{S}^\mathrm{DR}(t \mid X; \widehat{\eta}) \mid X = x ] - S_0(t\mid x)|. \]
    If either of $\widehat{S}$ or $\widehat{G}$ converges to the true value uniformly on $[0,t]$,
    \[ \sup_{u \in [0,t]} |\widehat{S}(u \mid x) - S_0 (u \mid x)| \stackrel{p}{\longrightarrow} 0 \quad \text{or} \quad \sup_{u \in [0,t]} |\widehat{G}(u \mid x) - G_0 (u \mid x)| \stackrel{p}{\longrightarrow} 0, \]
    then, $B(t, x; \widehat{\eta}) \stackrel{p}{\longrightarrow} 0$.
\end{lemma}

\begin{proof}
    By Theorem~\ref{thm:dr},
    \[ B(t,x;\widehat\eta) \le \min\{B_1(t,x;\widehat\eta),B_2(t,x;\widehat\eta)\}. \]
    Therefore, it suffices to show that either \(B_1(t,x;\widehat\eta)\stackrel{p}{\to}0\) or \(B_2(t,x;\widehat\eta)\stackrel{p}{\to}0\). First suppose
    \[ \sup_{u\in[0,t]} |\widehat G(u\mid x)-G_0(u\mid x)| \stackrel{p}{\longrightarrow} 0. \]
    By assumption~\ref{ass:positivity},
    {\small \[ \sup_{u\in[0,t]} \left|\frac{1}{\widehat G(u\mid x)}-\frac{1}{G_0(u\mid x)}\right| = \sup_{u\in[0,t]} \frac{|\widehat G(u\mid x)-G_0(u\mid x)|}{\widehat G(u\mid x)G_0(u\mid x)} \le c_2^2(t) \sup_{u\in[0,t]} |\widehat G(u\mid x)-G_0(u\mid x)|. \]
    }Moreover, we can bound the total variation distance between hazard difference measures using
    \[ \int_0^t \left|d\Lambda_{0,T\mid X}(u\mid x)-d\widehat\Lambda_{T\mid X}(u\mid x)\right| \le \Lambda_{0,T\mid X}(t\mid x)+\widehat\Lambda_{T\mid X}(t\mid x) \le 2\log c_1(t). \]
    Therefore,
    \begin{align*}
        B_1(t,x;\widehat\eta) &\le c_1(t) \sup_{u\in[0,t]} \left|\frac{1}{\widehat G(u\mid x)}-\frac{1}{G_0(u\mid x)}\right|\int_0^t \left|d\Lambda_{0,T\mid X}(u\mid x)-d\widehat\Lambda_{T\mid X}(u\mid x)\right| \\
        &\leq 2 c_1(t) \log c_1(t) \sup_{u\in[0,t]} \left|\frac{1}{\widehat G(u\mid x)}-\frac{1}{G_0(u\mid x)}\right| \\
        &\leq 2 c_1(t) c_2^2(t) \log c_1(t) \sup_{u\in[0,t]} |\widehat G(u\mid x)-G_0(u\mid x)|. 
    \end{align*}
    Thus, $B_1(t,x;\widehat\eta) \stackrel{p}{\longrightarrow} 0$. Now suppose instead that
    \[ \sup_{u\in[0,t]} |\widehat S(u\mid x)-S_0(u\mid x)| \stackrel{p}{\longrightarrow} 0. \]
    By assumption~\ref{ass:positivity},
    \[ \sup_{u\in[0,t]} \left|\frac{\widehat S(t\mid x)}{\widehat S(u\mid x)}-\frac{S_0(t\mid x)}{S_0(u\mid x)}\right| \le c_1^2(t) \sup_{u\in[0,t]} \left|\widehat S(t\mid x)S_0(u\mid x)-S_0(t\mid x)\widehat S(u\mid x)\right|. \]
    Using the triangle inequality,
    \begin{align*}
        \left|\widehat S(t\mid x)S_0(u\mid x)-S_0(t\mid x)\widehat S(u\mid x)\right| &\le \left| \widehat S(t\mid x)-S_0(t\mid x) \right| S_0(u\mid x)\\
        &\phantom{======}\,\,+ \left| S_0(u\mid x)-\widehat S(u\mid x) \right| S_0(t\mid x) \\
        &\le 2\sup_{v\in[0,t]} |\widehat S(v\mid x)-S_0(v\mid x)|.
    \end{align*}
    Like before, we can bound the total variation of the hazard difference measures using
    \[ \int_0^t \left|d\Lambda_{0,C\mid X}(u\mid x)-d\widehat\Lambda_{C\mid X}(u\mid x)\right| \le \Lambda_{0,C\mid X}(t\mid x)+\widehat\Lambda_{C\mid X}(t\mid x) \le 2\log c_2(t). \]
    Therefore,
    \begin{align*}
        B_2(t,x;\widehat\eta) &\le c_2(t) \sup_{u\in[0,t]} \left|\frac{\widehat S(t\mid x)}{\widehat S(u\mid x)}-\frac{S_0(t\mid x)}{S_0(u\mid x)}\right|\int_0^t \left|d\Lambda_{0,C\mid X}(u\mid x)-d\widehat\Lambda_{C\mid X}(u\mid x)\right| \\
        &\leq 2c_2(t) \log c_2(t) \sup_{u\in[0,t]} \left|\frac{\widehat S(t\mid x)}{\widehat S(u\mid x)}-\frac{S_0(t\mid x)}{S_0(u\mid x)}\right| \\
        &\leq 4c_2(t) c_1^2(t) \log c_2(t) \sup_{u\in[0,t]} |\widehat S(u\mid x)-S_0(u\mid x)|.
    \end{align*}
    Thus, \(B_2(t,x;\widehat\eta)\stackrel{p}{\longrightarrow}0\). Combining the two cases proves that \(B(t,x;\widehat\eta)\stackrel{p}{\longrightarrow}0\).
\end{proof}

\newpage
\subsection{Proof of Theorem 2}
\label{app:sec:proof_consistency}

Our proof proceeds in 4 main steps. 

(i) We first show that the isotonic regression estimate at any specific time converges to the $L_2(P)$ projection of the true survival function. (ii) Next, we prove pointwise consistency at a specific time. (iii) Then, we prove uniform consistency at a specific time. (iv) Finally, we show how these uniformly consistent estimates at each time are also naturally non-increasing in time, leading to Dykstra's projection algorithm terminating after 1 step and forming our DR-ISR estimate.

\subsubsection{Convergence to $L_2(P)$ projection}
\label{app:sec:l2p_convergence_proof}

\paragraph{Additional Notation:}
For this part of the proof, we first fix a time $t\geq0$.

Let us define the $L_2(P)$ distance as usual
\[\|f-g\|_P = \left(\int(f(r)-g(r))^2p(r)dr\right)^\frac{1}{2}. \]

Let $\mathcal{M}$ be the class of non-increasing 1D functions. Define the $L_2(P)$ projection of $S_{0,r}(t\mid r)$ onto this class as
\[S^* = \underset{g \in \mathcal{M}}{\argmin} \int_0^1 (S_{0,r}(t \mid r) - g(r))^2 p(r)dr.\]
If the risk function is well specified such that the true survival probabilities are non increasing in the risk (assumption~\ref{ass:risk_correctness}), then this is exactly equal to $S_{0,r}(t\mid r)$. However, we will prove convergence of the isotonic regression estimator to $S^*$ in probability in the $L_2(P)$ norm without this assumption since this will be useful in our proof of Theorem~\ref{thm:calibration}.

Let $\hat{f}_n$ denote the IR solution when using $n$ calibration datapoints to regress the pseudo-outcomes onto the risks. That is,
\[\hat{f}_n = \underset{g \in \mathcal{M}}{\argmin}\sum_{i=1}^n(\tilde{S}_\text{DR}(t \mid X_i; \widehat{\eta}_n) - g(r(X_i)))^2,\]
where $\widehat{\eta}_n=(\widehat{S}_n,\widehat{G}_n)$ are the initial estimates of the survival and censoring models used to construct the pseudo-outcomes when $n$ calibration datapoints were available.

We define the errors as
\[\epsilon_{n,i} = \tilde{S}_\text{DR}(t \mid X_i; \widehat{\eta}_n) - S_{0,r}(t\mid r(X_i)).\]

From Lemma~\ref{lemma:bias_convergence}, we have that if either 
\[ \sup_{u \in [0,t]} |\widehat{S}(u \mid x) - S_0 (u \mid x)| \stackrel{p}{\longrightarrow} 0 \quad \text{or} \quad \sup_{u \in [0,t]} |\widehat{G}(u \mid x) - G_0 (u \mid x)| \stackrel{p}{\longrightarrow} 0, \]
then $\mathbb{E}[\epsilon_{n,i} \mid X_i] \to 0$ as $n \to \infty$ for all $x$.

Further, note that the errors are equal to the integral term in Eq.~\ref{eq:dr_pseudo_label}. That is,

\begin{equation*}
    \epsilon_{n,i} = \widehat{S}_n(t \mid X_i) - S_0(t \mid X_i) - \int_0^t \frac{\widehat{S}_n(t \mid X_i)}{\widehat{S}_n(u \mid X_i) \widehat{G_n}(u \mid X_i)} \,d\widehat{M}_{T \mid X} (u \mid X_i).
\end{equation*}

Now, $\widehat{S}_n, \widehat{G}_n$ are estimated independently from the calibration data, and the random variable $d\widehat{M}_{T|X}(u|X)$ only depends on $\widehat{S}_n$ and the observed outcome for the $i^{th}$ calibration data point. Hence, the errors across the calibration data are independent and identically distributed at any fixed $n$. 

Lastly, since $t$ is fixed, the pseudo-outcomes are bounded by assumption~\ref{ass:positivity}, and hence the errors are bounded as well. Let the errors be bounded in absolute value by $M$.

\paragraph{Step 1:}
Let $\mu_{n,i} = \mathbb{E}[\epsilon_{n,i} \mid X_i]$ be the conditional bias and $\nu_{n,i} = \epsilon_{n,i} - \mu_{n,i}$ be the zero-mean stochastic noise. 

Note that given $X_i$, $\epsilon_{n,i}$ only depends on the realized outcome for $X_i$. Since the calibration data points are all independent, the conditional biases are also independent.

Since the $\mu_{n,i}$ and $\epsilon_{n,i}$ are independent across the calibration data, the zero-mean stochastic noise $\nu_{n,i}$ is also independent across the calibration data

We know $|\mu_{n,i}| \le M$ and $\mu_{n,i} \xrightarrow{p} 0$. By the bounded convergence theorem, we can interchange the limit and integral, proving that $\mathbb{E}[|\mu_{n,i}|] \to 0$ for all $i$. 

By Markov's inequality, the empirical average of the absolute conditional bias goes to zero in probability 
\begin{equation*}
    P\left( \frac{1}{n}\sum_{i=1}^n |\mu_{n,i}| > \delta \right) \le \frac{\frac{1}{n}\sum_{i=1}^n \mathbb{E}[|\mu_{n,i}|]}{\delta} \leq \frac{\max_{1\leq i\leq n}\mathbb{E}[|\mu_{n,i}|]}{\delta} \to 0.
\end{equation*}
Thus, $\frac{1}{n}\sum_{i=1}^n |\mu_{n,i}| \xrightarrow{p} 0$. 

The remaining stochastic noise $\nu_{n,i}$ forms a triangular array satisfying $\mathbb{E}[\nu_{n,i} \mid X_i] = 0$, is conditionally independent across $i$, and is strictly bounded in $[-2M, 2M]$.

\paragraph{Step 2:} First define the empirical risk over the noisy data, shifted by the errors as
\[ \hat{R}_n(g) := \frac{1}{n} \sum_{i=1}^n \left[ (\tilde{S}_\text{DR}(t \mid X_i; \widehat{\eta}_n) - g(r(X_i)))^2 - \epsilon_{n,i}^2 \right].\]
Also define the empirical risk over the noiseless data as 
\[\tilde{R}_n(g) := \frac{1}{n} \sum_{i=1}^n (S_{0,r}(t\mid r(X_i)) - g(r(X_i)))^2.\]
Finally, we have the population risk
 \[R(g) := \mathbb{E}_{X\sim P_X}[(S_{0,r}(t\mid r(X)) - g(r(X)))^2].\]

Because $S_0$ is bounded by $1$, $S^*$ is also bounded by $1$. Furthermore, since $|\tilde{S}_\text{DR}(t \mid X_i; \widehat{\eta}_n)| \leq 1 + M$, the isotonic regression estimator $\hat{f}_n$ is bounded by $1 + M$. Let $K = 1 + M$. We can restrict our search space to $\mathcal{M}^K = \{g \in \mathcal{M} : \|g\|_\infty \le K\}$.

Then for any $g\in\mathcal{M}^K$,
\begin{align*} 
    \hat{R}_n(g) - \tilde{R}_n(g) &= \frac{1}{n} \sum_{i=1}^n \left( (S_{0,r}(t\mid r(X_i)) + \epsilon_{n,i} - g(r(X_i)))^2 - \epsilon_{n,i}^2 - (S_{0,r}(t\mid r(X_i)) - g(r(X_i)))^2 \right) \\ 
    &= \frac{2}{n} \sum_{i=1}^n \epsilon_{n,i}(S_{0,r}(t\mid r(X_i)) - g(r(X_i))) .
\end{align*}

Decomposing $\epsilon_{n,i} = \mu_{n,i} + \nu_{n,i}$, we can bound the bias portion uniformly over all $g\in \mathcal{M}^K$ as
\begin{equation*}
    \sup_{g \in \mathcal{M}^K} \left| \frac{2}{n} \sum_{i=1}^n \mu_{n,i}(S_{0,r}(t\mid r(X_i)) - g(r(X_i))) \right| \leq 2(1 + K)\frac{1}{n} \sum_{i=1}^n |\mu_{n,i}|
\end{equation*}
From Step 1, this bias term converges to $0$ in probability.

\paragraph{Step 3:}
Note that $\mathcal{M}^K$ is a class of bounded monotonic functions. By Theorem 2.7.5 of van der Vaart and Wellner (1996)~\citep{van1996weak}, 
\[\log N_{[]}(\epsilon, \mathcal{M}^K, L_1(P)) \leq \frac{K\cdot C}{\epsilon},\]
where $N_{[]}$ is the bracketing number and $C$ is some global constant. Therefore, the bracketing number is finite for all $\epsilon>0$.

Define the class of difference functions $\mathcal{H} = \{ h_g(x) = S_{0,r}(t \mid r(x)) - g(r(x)) : g \in \mathcal{M}^K \}$. Since this is simply a bounded shift of $\mathcal{M}^K$, its bracketing number is also finite. 

Since the multipliers $\nu_{n,i}$ form a triangular array of conditionally independent, mean-zero random variables strictly bounded by $2M$, and $\mathcal{H}$ is a Glivenko-Cantelli class, the multiplier empirical process converges uniformly to zero in probability
\begin{equation*}
    \sup_{g \in \mathcal{M}^K} \left| \frac{2}{n} \sum_{i=1}^n \nu_{n,i}(S_{0,r}(t\mid r(X_i)) - g(r(X_i))) \right| = \sup_{h_g \in \mathcal{H}} \left| \frac{2}{n} \sum_{i=1}^n \nu_{n,i}h_g(X_i) \right| \xrightarrow{p} 0.
\end{equation*}
Combined with the uniform bound on the bias from Step 2, we have
\begin{equation*}
    \sup_{g \in \mathcal{M}^K} |\hat{R}_n(g) - \tilde{R}_n(g)| \xrightarrow{p} 0.
\end{equation*}

Now define $ \mathcal{F} = \{ h_g(x) = (S_{0,r}(t \mid r(x)) - g(r(x)))^2 : g \in \mathcal{M}^K \}$. 

For any two $g_1,g_2\in\mathcal{M}^K$,
\begin{align*} 
    |h_{g_1}(x) - h_{g_2}(x)| &= |(S_{0,r} - g_1(x))^2 - (S_{0,r} - g_2(x))^2| \\ 
    &= |2S_{0,r} - g_1(x) - g_2(x)| \cdot |g_2(x) - g_1(x)| \\ 
    &\leq 2(1+K) \cdot |g_2(x) - g_1(x)| .
\end{align*}
Therefore, if $[l,u]$ is an $\epsilon$ bracket for $\mathcal{M}^K$, it is an $\epsilon/(2(1+K))$ bracket for $\mathcal{F}$. Therefore, the bracketing number for $\mathcal{F}$ is also finite for all $\epsilon>0$. 

By Theorem 2.4.1 of van der Vaart and Wellner (1996)~\citep{van1996weak}, $\mathcal{F}$ is Glivenko-Cantelli. That is,
\begin{equation*}
    \sup_{g \in \mathcal{M}^K} |\tilde{R}_n(g) - R(g)| \xrightarrow{a.s.} 0.
\end{equation*}

\paragraph{Step 4:}
From Steps 2 and 3,
\begin{align*} 
    \sup_{g \in \mathcal{M}^K} |\hat{R}_n(g) - R(g)| &\leq \sup_{g \in \mathcal{M}^K} |\hat{R}_n(g) - \tilde{R}_n(g)| + \sup_{g \in \mathcal{M}^K} |\tilde{R}_n(g) - R(g)| \\ 
    &\xrightarrow{p}  0 + 0 .
\end{align*}

Now, $\hat{f}_n$ is the minimizer of $\hat{R}_n(g)$ over $\mathcal{M}^K$. Hence, $\hat{R}_n(\hat{f}_n) \leq \hat{R}_n(S^*)$. Also, $S^*$ is minimizer of $R(g)$ over $\mathcal{M}^K$. Hence, $R(S^*) \leq R(\hat{f}_n)$. Combining, 
\begin{align*} 
    0 &\leq R(\hat{f}_n) - R(S^*) \\ 
    &= (R(\hat{f}_n) - \hat{R}_n(\hat{f}_n)) + (\hat{R}_n(\hat{f}_n) - \hat{R}_n(S^*)) + (\hat{R}_n(S^*) - R(S^*)) \\ 
    &\leq (R(\hat{f}_n) - \hat{R}_n(\hat{f}_n)) + 0 + (\hat{R}_n(S^*) - R(S^*)) \\ 
    &\leq 2 \sup_{g \in \mathcal{M}^K} |\hat{R}_n(g) - R(g)| \\ 
    &\xrightarrow{p} 0 .
\end{align*}
Hence, $R(\hat{f}_n) \xrightarrow{p} R(S^*)$.

\paragraph{Step 5:}
Note that $\mathcal{M}$ is a closed convex cone. Hence, if $S^*$ is the projection of $S_{0,r}$ onto $\mathcal{M}$, then for any $g\in\mathcal{M}$,
\begin{equation*}
    \|S_{0,r}(t \mid r) - g(r)\|_P^2 \geq \|S_{0,r}(t \mid r) - S^*(r)\|_P^2 + \|g(r) - S^*(r)\|_P^2.
\end{equation*}

Note that $R(g) = \|g(r) - S_{0,r}(t \mid r)\|_P^2$ and pick $g=\hat{f}_n$, to get
\begin{equation*}
    \|\hat{f}_n(r) - S^*(r)\|_P^2 \le R(\hat{f}_n) - R(S^*).
\end{equation*}

Combining with Step 4,
\begin{equation*}
    \|\hat{f}_n - S^*\|_P \xrightarrow{p} 0.
\end{equation*}
In particular, under assumption~\ref{ass:risk_correctness}, we have that $\Sor$ is non-increasing itself and hence $S^*=\Sor$. Therefore,
\begin{equation*}
    \|\hat{f}_n(r) - \Sor\|_P \xrightarrow{p} 0.
\end{equation*}

\subsubsection{Pointwise Convergence}
\label{app:sec:pointwise_convergence}
We now show that pointwise convergence also holds. That is,
\[\lim_{n\rightarrow\infty}|\hat{f}_n(r)-S_{0,r}(t\mid r)|\xrightarrow{p}0.\]
We proceed by contradiction. Let pointwise convergence fail at $r_0$.
Since 
\begin{equation*}
    P(|\hat{f}_n(r_0) - S_{0,r}(t\mid r_0)| \geq \epsilon) = P(\hat{f}_n(r_0) - S_{0,r}(t\mid r_0) \geq \epsilon) + P(\hat{f}_n(r_0) - S_{0,r}(t\mid r_0) \leq -\epsilon),
\end{equation*}
for pointwise convergence to fail, there exists $\epsilon>0$ such that at least one of $P(\hat{f}_n(r_0) - S_{0,r}(t\mid r_0) \geq \epsilon) \nrightarrow 0 $ or $P(\hat{f}_n(r_0) - S_{0,r}(t\mid r_0) \leq -\epsilon) \nrightarrow 0$ must be true.

We consider the case that $P(\hat{f}_n(r_0) - S_{0,r}(t\mid r_0) \geq \epsilon) \nrightarrow 0$, the contradiction for the other case will follow symmetrically.

Since $S_{0,r}(t\mid r)$ is continuous at $r_0$, there exists $\delta>0$ such that for all $r\in[r_0-\delta,r_0]$,
\begin{equation*}
    S_{0,r}(t\mid r) - S_{0,r}(t\mid r_0) < \frac{\epsilon}{2}.
\end{equation*}
Since $\hat{f}_n$ in non-increasing, for all $x\in[r_0-\delta,r_0]$, we have $\hat{f}_n(r) \ge \hat{f}_n(r_0)$. If, $\hat{f}_n(r_0) \geq S_{0,r}(t\mid r_0) + \epsilon$ then, for all $r \in [r_0-\delta, r_0]$,
\begin{equation*}
    \hat{f}_n(r) \geq S_{0,r}(t\mid r_0) + \epsilon.
\end{equation*}
Combining the two,
\begin{equation*}
    \hat{f}_n(r) - S_{0,r}(t\mid r) > (S_{0,r}(t\mid r_0) + \epsilon) - \left(S_{0,r}(t\mid r_0) + \frac{\epsilon}{2}\right) = \frac{\epsilon}{2} \quad \text{for all } r \in [r_0-\delta, r_0].
\end{equation*}
But this implies the $L_2(P)$ distance is then at least
\begin{equation*}
    \|\hat{f}_n(r) - S_{0,r}(t\mid r)\|_P^2 \geq \int_{r_0-\delta}^{r_0} (\hat{f}_n(r) - S_{0,r}(t\mid r))^2 p(r) dr \geq \frac{\epsilon^2}{4} \int_{r_0-\delta}^{r_0} p(t) dt > 0,
\end{equation*}
since $p(t)$ is strictly positive by assumption~\ref{ass:continuous_density}. 

Hence,
\begin{align*}
     P(\|\hat{f}_n(r) - S_{0,r}(t\mid r)\|_P^2 > 0) \geq P(\hat{f}_n(r_0) - S_{0,r}(t\mid r_0) \geq \epsilon).
\end{align*}
By our assumption then
\begin{equation*}
    P(\hat{f}_n(r_0) - S_{0,r}(t\mid r_0) \geq \epsilon) \nrightarrow 0 \implies P(\|\hat{f}_n(r) - S_{0,r}(t\mid r)\|_P^2 > \epsilon) \nrightarrow 0,
\end{equation*}
leading to a contradiction. 
\subsubsection{Uniform Convergence}
From the previous part, we have that pointwise convergence holds at all $r\in(a,b)$. 

Consider $0<a<b<1$. Since $S_{0,r}(t\mid r)$ is continuous on $[a,b]$ and hence we can pick $K$ points such that $a=u_0<u_1<...<u_M=b$
such that
\begin{equation*}
    S_{0,r}(t\mid u_k) - S_{0,r}(t\mid u_{k-1}) \leq \epsilon \quad \text{for all } k = 1, \dots, M
\end{equation*}
for some $\epsilon>0$. 

Fix a $\delta>0$. From the definition of pointwise convergence, there exist $N_1,...,N_M$ such that for any $u_m$
\begin{equation*}
    P(|\hat{f}_n(u_m)-S_{0,r}(t\mid u_m)|>\epsilon) \leq \frac{\delta}{M+1} \quad \forall \ n>N_m.
\end{equation*}
Let $N=\max\{N_1,...,N_m\}$. Then for any $n>N$, by the union bound and pointwise convergence,
\begin{equation*}
    P\left( \max_{0 \leq m \leq M} |\hat{f}_n(u_m) - S_{0,r}(t\mid u_m)| > \epsilon \right) \leq \sum_{m=0}^M P\left( |\hat{f}_n(u_m) - S_{0,r}(t\mid u_m)| >\epsilon \right) \leq \delta.
\end{equation*}
Now, consider some $r\in(a,b)$. Let $u_{m-1}\leq x \leq u_m$ be the interval $r$ is in. $\hat{f}_n$ is non-increasing so we have 
\begin{align*}
    |\hat{f}_n(r)-S_{0,r}(t\mid r)| \leq \max\{\hat{f}_n(u_{m-1})-S_{0,r}(t\mid r),-\hat{f}_n(u_{m})+S_{0,r}(t\mid r)\}.
\end{align*}
Note that by construction and non-decreasingness of $S_{0,r}(t\mid r)$
\begin{align*}
    \hat{f}_n(u_{m-1})-S_{0,r}(t\mid r) &= \hat{f}_n(u_{m-1})-S_{0,r}(t\mid u_{m-1})-\left(S_{0,r}(t\mid r)-S_{0,r}(t\mid u_{m-1})\right) \\
    &\leq \hat{f}_n(u_{m-1})-S_{0,r}(t\mid u_{m-1}) +\epsilon,
\end{align*}
and
\begin{align*}
    -\hat{f}_n(u_{m})+S_{0,r}(t\mid r) &= -\hat{f}_n(u_{m})+S_{0,r}(t\mid u_m) + (S_{0,r}(t\mid r)- S_{0,r}(t\mid u_m)) \\
    &\leq -\hat{f}_n(u_{m})+S_{0,r}(t\mid u_m) + \epsilon.
\end{align*}
Combining, 
\begin{align*}
    |\hat{f}_n(r)-S_{0,r}(t\mid r)| &\leq \max\{\hat{f}_n(u_{m-1})-S_{0,r}(t\mid u_{m-1}),-\hat{f}_n(u_{m})+S_{0,r}(t\mid u_m)\} + \epsilon \\
    &\leq \max_{0 \leq m \leq M} |\hat{f}_n(u_m) - S_{0,r}(t\mid u_m)| + \epsilon,
\end{align*}
where the upper bound holds for all $x$. 

So, for any $n>N$,
\begin{align*}
    P\left(\sup_{r\in(a,b)}|\hat{f}_n(r)-S_{0,r}(t\mid r)|>2\epsilon\right) &\leq P\left(\max_{0 \leq m \leq M} |\hat{f}_n(u_m) - S_{0,r}(t\mid u_m)| + \epsilon>2\epsilon\right) \\
    &=P\left(\max_{0 \leq m \leq M} |\hat{f}_n(u_m) - S_{0,r}(t\mid u_m)| >\epsilon\right) \\
    &\leq \delta.
\end{align*}
Since $\epsilon,\delta>0$ were arbitrary, we have the desired uniform convergence
\begin{equation*}
    \sup_{r\in(a,b)}|\hat{f}_n(r)-S_{0,r}(t\mid r)| \xrightarrow{p} 0.
\end{equation*}
\subsubsection{Convergence of Dykstra's}
\label{app:sec:dykstras_convergence}
Lastly, we show how the $\hat{f}_n$ isotonic regression estimators constructed at each fixed $t$ from the previous part are also non-increasing in time. Hence, Dykstra's algorithm to find the solution to the optimization problem from Eq.~\ref{eqn:DR_estimation} converges in one step after ensuring monotonicity along risks itself.

Let $t_j\leq t_i$. Let $\hat{f}_n^i$ and $\hat{f}^j_n$ be the isotonic regression estimate from time $t_i$ and $t_j$ respectively.  From Sec.~\ref{app:sec:l2p_convergence_proof}, we have that $\|\hat{f}^i_n-S^{i*}\|_P\xrightarrow{p}0$ and $\|\hat{f}^j_n-S^{j*}\|_P\xrightarrow{p}0$ where
    \begin{equation*}
        S^{i*} = \underset{g \in \mathcal{M}}{\argmin} \int (S_{0,r}(t_i\mid r) - g(r))^2 dp(r)dr,
    \end{equation*}
    and
    \begin{equation*}
        S^{i*} = \underset{g \in \mathcal{M}}{\argmin} \int (S_{0,r}(t_j\mid r) - g(r))^2 dp(r)dr.
    \end{equation*}

We now proceed by contradiction.

Suppose the set $A = \{x \mid S^{i*}(r) > S^{j*}(r)\}$ has positive measure under $P$. Intuitively, this is the set where a violation of monotonicity along time occurs.
    
Define the functions $g_{\min}(x) = \min(S^{i*}(r), S^{j*}(r))$ and $g_{\max}(r) = \max(S^{i*}(r), S^{j*}(r))$. Because the class of monotonically increasing functions $\mathcal{M}$ is closed under minimum and maximum operations, both $g_{\min}$ and $g_{\max}$ are in $\mathcal{M}$.
    
By the definition of $S^{i*}$ and $S^{j*}$ as the optimal $L_2(P)$ projections onto $\mathcal{M}$, replacing them with $g_{\min}$ and $g^j$ respectively, cannot decrease the objective function. Therefore, we have
\begin{align*}
    \Delta_{\min} &= \int (S_{0,r}(t_i\mid r) - g_{\min}(r))^2 p(r)dr - \int (S_{0,r}(t_i\mid r) - S^{i*}(r))^2 p(r)dr \geq 0, \\
    \Delta_{\max} &= \int (S_{0,r}(t_j\mid r) - g_{\max}(r))^2 p(r)dr - \int (S_{0,r}(t_j\mid r) - S^{j*}(r))^2p(r)dr \geq 0.
\end{align*}

Observe that on the complement $A^c$, $S^{i*}(r) \leq S^{j*}(r)$, which implies $g_{\min} = S^{i*}$ and $g_{\max} = S^{j*}$. Thus, the integrands of $\Delta_{\min}$ and $\Delta_{\max}$ evaluate to exactly zero outside of $A$. 

We can therefore restrict the sum of the differences strictly to the set $A$, where $g_{\min} = S^{j*}$ and $g_{\max} = S^{i*}$,
\begin{align*}
    \Delta_{\min} + \Delta_{\max} &= \int_A \Big[ (S_{0,r}(t_i\mid r) - S^{j*}(r))^2 - (S_{0,r}(t_i\mid r)-S^{i*}(r))^2\Big]p(r)dr \\
    &+ \int_A\Big[ (S_{0,r}(t_j\mid r) - S^{i*}(r))^2 - (S_{0,r}(t_j\mid r) - S^{j*}(r))^2 \Big] p(r)dr.
\end{align*}

Expanding the squares and simplifying the integrand yields
\begin{align*}
    &S_{0,r}(t_i\mid r)^2 - 2S_{0,r}(t_i\mid r) S^{j*}(r) + S^{j*}(r)^2 - S_{0,r}(t_i\mid r)^2 + 2S_{0,r}(t_i\mid r) S^{i*}(r) - S^{i*}(r)^2 \\
    + &S_{0,r}(t_j\mid r)^2 - 2S_{0,r}(t_j\mid r) S^{i*}(r) + S^{i*}(r)^2 - S_{0,r}(t_j\mid r)^2 + 2S_{0,r}(t_j\mid r) S^{j*}(r) - S^{j*}(r)^2 \\
    &= 2S_{0,r}(t_i\mid r) S^{i*}(r) - 2S_{0,r}(t_i\mid r) S^{j*}(r) - 2S_{0,r}(t_j\mid r) S^{i*}(r) + 2S_{0,r}(t_j\mid r) S^{j*}(r) \\
    &= 2(S_{0,r}(t_i\mid r) - S_{0,r}(t_j\mid r))(S^{i*}(r) - S^{j*}(r)).
\end{align*}

Substituting this back into our integral, we get
\begin{equation*}
    \Delta_{\min} + \Delta_{\max} = \int_A 2(S_{0,r}(t_i\mid r) - S_{0,r}(t_j\mid r)(S^{i*}(r) - S^{j*}(r)) p(r) dr.
\end{equation*}

We now analyze the signs of the terms in the integrand on the set $A$.

First, $S_{0,r}(t_i\mid r) \leq S_{0,r}(t_j\mid r)$ since survival curves are non-increasing, so $S_{0,r}(t_i\mid r) - S_{0,r}(t_j\mid r) \leq 0$. Second, $S^{i*}(r) > S^{j*}(r)$ by the definition of $A$, so $S^{i*}(r) - S^{j*}(r) > 0$. Lastly, $p(r) > 0$ strictly everywhere, by assumption.

Multiplying these terms together, the integrand is non-positive everywhere on $A$, which means the entire integral evaluates to $\Delta_{\min} + \Delta_{\max} \leq 0$. However, since $\Delta_{\min} \geq 0$ and $\Delta_{\max} \geq 0$, we must have $\Delta_{\min} + \Delta_{\max} = 0$. 

This requires the integrand to be zero almost everywhere on $A$. Since $S^{i*} - S^{j*} > 0$ and $p(r) > 0$ on $A$, it must hold that $S^i(x) = S^j(x)$ almost everywhere on $A$. But if $S_{0,r}(t_i\mid r) = S_{0,r}(t_j\mid r)$ on $A$, their unique $L_2(P)$ projections onto $\mathcal{M}$ must be identical on $A$, which contradicts the definition of $A$ where $S^{i*} > S^{j*}$.

Therefore, the set $A$ must have a measure of zero under $P$. Because $p(r)$ is strictly positive on its domain, $S^{i*}(x)\leq S^{j*}(x)$ for almost all $x$.

Hence, the solution to the optimization problem defining $\widehat{S}^\mathrm{DR}_\mathrm{ISR}$ (Eq.~\ref{eqn:DR_estimation}), is asymptotically found by simply ensuring monotonicity along risks. Hence, to predict the survival probability at any $t\in\mathcal{T}$, it suffices to find the isotonic regression estimator enforcing monotonicity along risks for the pseudo-outcomes constructed over calibration data at $t$.

By assumption~\ref{ass:risk_correctness} we get $S^*=S_{0,r}(t\mid r)$ as the asymptotic solution enforcing non-increasingness along risks. Combining with lemma~\ref{lemma:risk_score_sufficiency} which ensures $S_{0}(t\mid x)= S_{0,r}(t\mid r(x))$, we have the desired result
\begin{equation}
    \sup_{x\in\mathcal{X}'} |\widehat{S}^{\mathrm{DR}}_{\mathrm{ISR}}(t,x)-S_0(t\mid x)|\xrightarrow{p}0, \qquad \text{for all } t \geq 0,
\end{equation}
where $\mathcal{X}'= \{x \in \mathcal{X} : r(x) \in \mathrm{int}(r(\mathcal{X}))\}$.

\newpage
\subsection{Proof of Theorem 3}

Once again, consider a fixed time $t\geq0$.

Define $G(x) = \int_{0}^x (1-\Sor)p(r)dr$ and let $F(x) = \int_0^xp(r)dr$ be the CDF of $P$. 

Consider the parametric curve $C=\{(G(x),F(x)\:x\in[0,1]\}$. 

Then, let $G^*(x)$ be such that $C^*=\{(G^*(x),F(x)\:x\in[0,1]\}$ is the greatest convex minorant (GCM) of $C$.

As before, let
\[S^* = \underset{g \in \mathcal{M}}{\argmin} \int (S_{0,r}(t \mid r) - g(r))^2 p(r)dr\]
be the $L_2(P)$ projection of $S_{0,r}(t \mid r)$ onto the class of non-increasing functions. 

Further define $f^*=1-S^*$. $f^*$ can then be seen the $L_2(P)$ projection of the survival CDF, $1-S_{0,r}(t \mid r)$ onto the class of non-decreasing functions.

It is well known that $f^*(r)$ is the subdifferential of $C^*$ at $F(r)$~\citep{barlow1972IRBook, groeneboom2014IRBook}. Since $C^*$ is convex, the subdifferential is guaranteed to not be empty over the interior $(0,1)$. 

For ease of exposition, we consider the unique $f^*$ which is defined as the right derivative of $C^*$ at $F(r)$. Notably, implementations of isotonic regression algorithms, including ours, use this convention. We note, however, that the proof below easily extends without this right derivative assumption.

\paragraph{Case I: $f^*(r)=q$ for multiple $r$}

Let $[a,b)$ be the maximal interval where $f^*(r)=q$. Note that $a\neq b$ since  $f^*(r)=q$ for multiple $r$. Also, note that by monotonicity of $f^*$, $f^*(r)<q$ for $r<a$ and $f^*(r)>q$ for $r>b$. 

Since $f^*(r)$ is constant over this interval, the portion of the curve $C^*$ between $F(a)$ and $F(b)$ is a straight line. By the properties of the GCM, this implies that $C^*$ touches $C$ at $F(a)$ and $F(b)$. Hence, $G(a) = G^*(a)$ and $G(b) = G^*(b)$. 

Since the portion of the curve $C^*$ between $F(a)$ and $F(b)$ is a straight line, we can write the vertical change as the slope (i.e $q$) times the horizontal change (i.e. $(F(b)-F(a))$) to get
\begin{align*}
    &G(b)-G(a) = G^*(b)-G^*(a) = q(F(b)-F(a)) \\
    \implies & \int_a^b (1-S_{0,r}(t \mid r))p(r)dr = q\int_a^b p(r)dr \\
    \implies & \frac{\int_a^b (1-S_{0,r}(t \mid r))p(r)dr}{\int_a^b p(r)dr} = q \hspace{60pt} \left(\int_a^b p(r)dr>0 \textrm{ since } p(r)>0 \textrm{ everywhere}\right)\\
    \implies& \mathbb{E}[1-S_{0,r}(t \mid r)|f^*(r)=q] = q,
\end{align*}
where the expectation is taken treating $r$ as a random variable and conditioning on the event $f^*(r)=q$.

\paragraph{Case II: $f^*(r) = q$ for a unique $r$}

Let $r_0$ be the unique point where $f^*(r_0)=q$. Since $f^*(r_0)$ is the (right) derivative of $C^*$ at $F(r_0)$, $C^*$ is strictly increasing at $F(r_0)$. 

By the properties of the GCM, the GCM must touch $C$ at $F(r_0)$, i.e. $G(F(r_0)) = G^*(F(r_0))$. Now consider some $r>r_0$. Since $G^*$ is below $G$, we have
\begin{align*}
    G^*(F(r)) \leq G(F(r)).
\end{align*}
Subtract $G(F(r_0)) = G^*(F(r_0))$ from both sides and divide by $F(r_0)-F(r)$ ($F(r_0)-F(r)>0$ since the density is positive everywhere) to get
\begin{equation*}
    \frac{G^*(F(r)) - G^*(F(r_0))}{F(r)-F(r_0)} \leq \frac{G(F(r)) - G(F(r_0))}{F(r)-F(r_0)}.
\end{equation*}
Taking limits of $r\rightarrow r_0^+$,
\begin{equation*}
    q = f^*(r_0) = \lim_{r\rightarrow r_0^+} \frac{G^*(F(r)) - G^*(F(r_0))}{F(r)-F(r_0)} \leq \lim_{r\rightarrow r_0^+} \frac{G(F(r)) - G(F(r_0))}{F(r)-F(r_0)}.
\end{equation*}
Next, consider some $r<r_0$. Since $G^*$ is below $G$, we have
\begin{align*}
    G^*(F(r)) \leq G(F(r)).
\end{align*}
Subtract $G(F(r_0)) = G^*(F(r_0))$ from both sides and divide by $F(r_0)-F(r)$ ($F(r_0)-F(r)>0$ since the density is positive everywhere and $r<r_0$) to get
\begin{align*}
    &\frac{G^*(F(r)) - G^*(F(r_0))}{F(r_0)-F(r)} \leq \frac{G(F(r)) - G(F(r_0))}{F(r_0)-F(r)} \\
    \implies& \frac{G(F(r)) - G(F(r_0))}{F(r)-F(r_0)} \leq \frac{G^*(F(r)) - G^*(F(r_0))}{F(r)-F(r_0)}.
\end{align*}
Taking limits of $r\rightarrow r_0^-$,
\begin{equation*}
    \lim_{r\rightarrow r_0^-} \frac{G(F(r)) - G(F(r_0))}{F(r)-F(r_0)} \leq \lim_{r\rightarrow r_0^-} \frac{G^*(F(r)) - G^*(F(r_0))}{F(r)-F(r_0)}.
\end{equation*}
Since $C^*$ is a convex curve, its derivative is non decreasing. In particular, the left derivative must be at most the right derivative at any point. That is,
\begin{equation*}
    \lim_{r\rightarrow r_0^-} \frac{G^*(F(r)) - G^*(F(r_0))}{F(r)-F(r_0)} \leq \lim_{r\rightarrow r_0^+} \frac{G^*(F(r)) - G^*(F(r_0))}{F(r)-F(r_0)}.
\end{equation*}
Putting it all together,
\begin{align*}
    \lim_{r\rightarrow r_0^-} \frac{G(F(r)) - G(F(r_0))}{F(r)-F(r_0)} &\leq \lim_{r\rightarrow r_0^-} \frac{G^*(F(r)) - G^*(F(r_0))}{F(r)-F(r_0)} \\
    &\leq \lim_{r\rightarrow r_0^+} \frac{G^*(F(r)) - G^*(F(r_0))}{F(r)-F(r_0)} = f^*(r_0) = q \\
    & \leq \lim_{r\rightarrow r_0^+} \frac{G(F(r)) - G(F(r_0))}{F(r)-F(r_0)}.
\end{align*}
Next, note that 
\begin{align*}
    \lim_{r\rightarrow r_0^-} \frac{G(F(r)) - G(F(r_0))}{F(r)-F(r_0)} &= \lim_{r\rightarrow r_0^-} \frac{(S_{0,r}(t\mid r))p(r)}{p(r)} & \textrm{(L'Hôspital's Rule)} \\
    &= 1-S_{0,r}(t\mid r_0^-).
\end{align*}
Similarly,
\begin{align*}
    \lim_{r\rightarrow r_0^+} \frac{G(F(r)) - G(F(r_0))}{F(r)-F(r_0)} = 1-S_{0,r}(t\mid r_0^+).
\end{align*}
Therefore,
\begin{equation*}
    1-S_{0,r}(t\mid r_0^-) \leq f^*(r_0) = q \leq S_{0,r}(t\mid r_0^+).
\end{equation*}
By assumption~\ref{ass:continous_survival}, $S_{0,r}(t\mid r)$ is continuous at $r_0$, so $1-S_{0,r}(t\mid r_0^-) = 1-S_{0,r}(t\mid r_0^+)$ and hence $f^*(r_0) = q =  1-S_{0,r}(t\mid r_0)$. 

Finally,
\begin{equation*}
    \mathbb{E}[1-S_{0,r}(t\mid r)|f^*(r)=q] = \mathbb{E}[1-S_{0,r}(t\mid r)|r=r_0] = 1-S_{0,r}(t\mid r_0) = q.
\end{equation*}

\paragraph{Combining Cases}

Together, from Case I and II, we have shown that 
\begin{equation*}
    \mathbb{E}[1-S_{0,r}(t\mid r)|f^*(r)=q] = q,
\end{equation*}
for all $q$ that are in the range of $f^*$.

Since $f^*=1-S^*$, we can define $p=1-q$ to get
\begin{equation*}
    \mathbb{E}[S_{0,r}(t\mid r)|S^*(r)=p] = p.
\end{equation*}
Now, from the proof of Theorem~\ref{thm:consistency}, from Sec.~\ref{app:sec:l2p_convergence_proof} we have that the isotonic regression fit at any specific time $t$ converges to $S^*$ in the $L_2(P)$ norm. 

We can proceed exactly as in Sec.~\ref{app:sec:pointwise_convergence} to show that this convergence is also pointwise thanks to monotonicity of $S^*$ and assumptions~\ref{ass:continuous_density}, ~\ref{ass:continous_survival},

Finally from Sec.~\ref{app:sec:dykstras_convergence}, we see that the $\widehat{S}_\textrm{ISR}^\textrm{DR}$ asymptotically equals the $S^*$ estimates from time $t$. 

Combining, we have that the pointwise limit of the DR-ISR estimator $\widehat{S}^\infty_\mathrm{DR\text{-}ISR}(t,x)  = \lim\limits_{n\rightarrow\infty}\widehat{S}^{\mathrm{DR}}_{\mathrm{ISR}}(t,x)$ is simply 
\begin{equation*}
    \widehat{S}^\infty_\mathrm{DR\text{-}ISR}(t,x) = S^*_t(r(X)),
\end{equation*}
where,
\[S^*_t = \underset{g \in \mathcal{M}}{\argmin} \int (S_{0,r}(t \mid r) - g(r))^2 p(r)dr.\]
Therefore,
\begin{equation*}
    \mathbb{E}\left[S_{0,r}(t\mid r(X))\ \middle|\ \widehat{S}^\infty_\mathrm{DR\text{-}ISR}(t,X)=p\right] = p.
\end{equation*}
Recall that $\Sor = \mathbb{P}(T>t|r(X)=r) = \mathbb{E}[\mathbf{1}\{T>t\}|r(X)=r]$. 

Substituting in the equation above and applying tower rule, we have that at any time $t\in\mathcal{T}$,
\begin{equation*}
    \mathbb{E}\left[\mathbf{1}\{T>t\}\ \middle|\ \widehat{S}^\infty_\mathrm{DR\text{-}ISR}(t,X)=p\right] = p.
\end{equation*}
Thus, we have the desired result
\begin{equation*}
    \mathbb{P} \big( T>t \ \big| \ \widehat{S}^\infty_\mathrm{DR\text{-}ISR} (t,X) = p \big) = p, \qquad \text{for all } t \in \mathcal{T},
\end{equation*}
for all $p$ in the range of $\widehat{S}^\infty_\mathrm{DR\text{-}ISR}$.
\newpage

\section{Additional estimators}
\label{app:sec:new_estimators}

We now present the motivation behind the RW-ISR and DR-ISR estimators as developed in Sec~\ref{sec:estimator_construction} in full generality. This allows us to identify three new estimators RW$^+$-ISR, HT-ISR and HT$^+$-ISR.

We have considered two broad approaches to calibration under censoring. The first uses constrained regression on the uncensored binary survival indicators in the calibration data. In this case, censoring is handled by weighting the optimization objective to account for unequal probabilities of observing $\mathbf{1}\{T_i>t\}=Z_i(t)$. The second method uses a pseudo-outcome approach, which first constructs pointwise estimates of the conditional survival probability on the calibration set. These estimates are then projected onto a family of survival functions that are monotone in both risk and time.

Because these two approaches incorporate censoring in fundamentally different ways, we develop the resulting estimators separately.

\subsection{Weighted isotonic regression using survival indicators}
Under right-censoring, the survival indicators $Z_i(t)$ are only partially observed. To account for this, we use inverse probability of censoring weighting (IPCW), which reweights observations according to their probability of remaining uncensored.

Let $O_i = (X_i, Y_i, \delta_i)$ denote the observed data, and as before define the censoring analog of the survival function $G(t \mid x) = \mathbb{P}(C > t \mid X = x)$. We also define the IPCW weight $w(O) = \delta / G(Y \mid X)$. For a fixed time $t$, consider the weighted regression problem
\[ f_t = \argmin_{f \in \mathcal{F}} \mathbb{E} \left[ w(O) (f(r(X)) - \mathbf{1}[Y > t])^2 \right]. \]
Note that the population minimizer is the well-known Hajek estimator of the survival probability.
\[ f_t(r(x)) = \frac{ \mathbb{E}  \left[ w(O) \, \mathbf{1}[Y > t] \mid r(X) = r(x) \right]}{\mathbb{E}[w(O) \mid r(X) = r(x)]} = \mathbb{P}(T > t \mid r(X) = r(x)) = S(t \mid x). \]

There are two additional practical considerations. First, neither the risk score $r$ nor the censoring function $G$ is known, so both must be learned on the training data, for example using a Cox model. We then treat these estimates as fixed and solve the corresponding weighted optimization problem on the independent calibration set. Second, we impose a monotonicity constraint in the risk score so that the resulting estimator defines a valid survival curve. The calibration step solves,
\begin{equation} \label{app:eqn:hajek_optimization}
    \widehat{S}_\textrm{ISR}^\textrm{RW}(t \mid x) = f_t(\widehat{r}(x)), \quad \text{where } f_t = \argmin_{f \in \mathcal{F}_{\leq}} \sum_{O_i \in \mathcal{D}_\text{cal}} \frac{\delta_i (f(\widehat{r}(X_i)) - \mathbf{1}[Y_i > t])^2}{\widehat{G}(Y_i \mid X_i)}.
\end{equation}

The RW-ISR estimator does not make full use of the available information, since the survival indicator is known not only when $\delta = 1$, but also for censored individuals with $Y > t$. We can modify the IPCW weight to account for all samples for which $Z_i(t)$ is observed. For a fixed time $t$, consider the weighted regression problem,
\[ f^+_t = \argmin_{f \in \mathcal{F}} \mathbb{E} \left[ w^+(O) (f(r(X)) - \mathbf{1}[Y > t])^2 \right], \quad \text{where } w^+(O) = \frac{\mathbf{1}[Y > t]}{G(t \mid X)} + \frac{\delta \mathbf{1}[Y \leq t]}{G(Y \mid X)}. \]
The survival probability remains the population minimizer.
\[ f^+_t(r(x)) = \frac{ \mathbb{E}  \left[ w^+(O) \, \mathbf{1}[Y > t] \mid r(X) = r(x) \right]}{\mathbb{E}[w^+(O) \mid r(X) = r(x)]} = \mathbb{P}(T > t \mid r(X) = r(x)) = S(t \mid x). \]
Since this estimator improves upon the RW-ISR estimator, we refer to it as RW$^+$-ISR. In practice, we use the same calibration procedure as before, but with the modified weight. Specifically,
\begin{equation}  \label{app:eqn:hajek+_optimization}
    \widehat{S}_\textrm{ISR}^\textrm{RW+}(t \mid x) = f^+_t(\widehat{r}(x)), \quad \text{where } f^+_t = \argmin_{f \in \mathcal{F}_{\leq}} \sum_{O_i \in \mathcal{D}_\text{cal}} w^+(O_i) (f(\widehat{r}(X_i)) - \mathbf{1}[Y_i > t])^2.
\end{equation}
Intuitively, both RW-ISR and RW$^+$-ISR use the censoring model to determine how to weight the calibration data in order to correct for missing survival indicators induced by censoring. 

\subsection{Isotonic projection of pseudo-outcomes}

We now describe our second calibration strategy as used by DR-ISR. Rather than incorporating censoring directly into the calibration objective, this approach first constructs pointwise pseudo-outcomes for the conditional survival probability on the calibration set. 

For each $O_i \in \mathcal{D}_{\mathrm{cal}}$ and $t \in \mathcal{T}$, let $\tilde{S}(t \mid O_i; \eta)$ denote a pseudo-outcome, where $\eta$ denotes the nuisance parameters. Because these pseudo-outcomes are constructed pointwise, they may not satisfy the structural properties of a valid survival function. We therefore project them onto a class of functions that is monotone in both risk and time. Consider the optimization problem
\[ f = \argmin_{f \in \mathcal{F}^{(2)}} \sum_{t \in \mathcal{T}} \mathbb{E} \left[ \left( f(t, r(X)) - \tilde{S}(t \mid O; \eta) \right)^2 \right]. \]
For conditionally unbiased pseudo-outcomes in the sense that $\mathbb{E}[\tilde{S}(t \mid O; \eta) \mid X = x] = S(t \mid x)$, the population minimizer agrees with the true survival function on the grid $\mathcal{T}$. 

In practice, the risk scores and nuisance quantities are first estimated on the training data. In the calibration step, we use the plug-in nuisance estimate $\widehat{\eta}$ and replace the population objective by its empirical analogue on the calibration set. The calibration step solves
\begin{equation} \label{app:eqn:pseudo_outcome_optimization}
    \widehat{S}(t \mid x) = f(t, \widehat{r}(x)), \quad \text{where } f = \argmin_{f \in \mathcal{F}^{(2)}_\leq} \sum_{t \in \mathcal{T}} \: \sum_{X_i \in \mathcal{D}_\text{cal}} \left[ \left( f(t, \widehat{r}(X_i)) - \tilde{S}(t \mid O_i; \widehat{\eta}) \right)^2 \right].
\end{equation}

\textit{Remark.} Survival models such as the Cox model already produce structurally valid survival curves, but their survival probabilities are often not well calibrated in practice. The pseudo-outcomes provide pointwise estimates that may be statistically more efficient than the initial survival model, while the projection step restores the structural constraints required of a valid survival function. 

We next describe two ways to construct conditionally unbiased pseudo-outcomes using IPCW and Augmented IPCW (AIPCW).

\paragraph{Horvitz--Thompson.} The first method uses IPCW in the same spirit as the RW-ISR estimator. For each $t \in \mathcal{T}$ and $O_i \in \mathcal{D}_{\mathrm{cal}}$, we define the pointwise pseudo-outcomes
\[ \tilde{S}_\textrm{ISR}^\textrm{HT}(t \mid O_i; \eta) = \frac{\delta_i \mathbf{1}[Y_i > t]}{G(Y_i \mid X_i)}, \qquad \tilde{S}^\text{HT+}_\text{ISR}(t \mid O_i; \eta) = \frac{\mathbf{1}[Y_i > t]}{G(t \mid X_i)}, \qquad \text{where } \eta = G. \]
The corresponding distributional estimators in Equation~\ref{app:eqn:pseudo_outcome_optimization} are denoted $\widehat{S}_\text{HT}$ and $\widehat{S}_\text{HT+}$. Because the pseudo-outcomes take the well-known Horvitz--Thompson (HT) form, we refer to them as the HT-ISR and HT$^+$-ISR distributional estimators, respectively. As in the RW-ISR setting, HT$^+$-ISR improves on the initial construction by incorporating censored observations for which the survival indicator is still known. If multiple calibration points share the same covariate value (e.g. if the covariates are discrete), we define a separate pseudo-outcome for each observation.

\paragraph{Doubly robust.} The RW and Horvitz--Thompson constructions above rely only on inverse probability of censoring weighting, and therefore use only the censoring model to correct for missing survival indicators. A different strategy is to combine the censoring model with an outcome model for the survival function itself. This leads to the augmented IPCW (AIPCW) construction.
\begin{equation*}
    \tilde{S}_\textrm{ISR}^\textrm{DR}(t \mid O_i; \eta) = S(t \mid X_i) - \int_0^t \frac{S(t \mid X_i)}{S(u \mid X_i) G(u \mid X_i)} \,dM_{T \mid X} (u \mid X_i), \qquad \text{where } \eta = (S, G).
\end{equation*}
Here, $dM_{T \mid X} (u \mid x) = dN_T(u) - \mathbf{1}[Y \geq u] \, d\Lambda_{T \mid X} (u \mid x)$ denotes the event martingale difference process, where $dN_T(u) = \mathbf{1}[Y \in du \:\wedge\: \delta = 1]$ is the event increment counting process and $\Lambda_{T \mid X} (u \mid x) = -\log S(u \mid x)$ is the cumulative hazard function as described in Sec.~\ref{sec:estimator_construction}

The AIPCW construction has an attractive theoretical advantage over IPCW: it is conditionally unbiased if either the survival model or the censoring model is correctly specified.

\newpage

\section{Algorithmic Implementation}

\subsection{DR-ISR Algorithm}
For completeness, we present an expanded version of algorithm~\ref{alg:dr} to fit and make predictions from the DR-ISR estimator. Both algorithm~\ref{alg:dr} and algorithm~\ref{app:alg:dr} contain the exact same operations but algorithm~\ref{app:alg:dr} improves legibility by spelling out the operations more clearly. In practice, we also clip the probabilities returned by $\widehat{S},\widehat{G}$ at some small $\epsilon$ to ensure stability. Finally, to construct the time grid, we use 10000 evenly spaced points between $t=0$ and the maximum observed time in the calibration dataset. We additionally augment this time grid with all the observed times in the training and calibration datasets.
\begin{algorithm}
    \caption{Construction of DR-ISR calibrated survival curve estimates }\label{app:alg:dr}
    \begin{algorithmic}[1]
        \Require Survival $\widehat{S}$, censoring $\widehat{G}$, and risk $r_\theta$ models from an initial Deep Cox model
\Require Calibration dataset $\mathcal{D}_{\text{cal}} = \{(X_i, Y_i, \delta_i)\}_{i=1}^{n_{\text{cal}}}$ sorted by increasing risk
\Require Time grid $\mathcal{T} = \{t_1, \ldots, t_K\}$

\Statex \Comment{Doubly robust point estimation}
\For{$j \in [n_\text{cal}]$}
    \State $\widehat{S}(t_0 \mid X_j) \gets 1$
    \For{$i \in [K]$}
        \State $d\Lambda_{T \mid X} (t_i \mid X_j) \gets 1 - \widehat{S}(t_i \mid X_j) / \widehat{S}(t_{i-1} \mid X_j)$
        \State $dM_{T \mid X} (t_i \mid X_j) \gets \mathbf{1}[Y_j = t_i, \delta_j = 1] - \mathbf{1}[Y_j > t_i] \, d\Lambda_{T \mid X} (t_i \mid X_j)$
        \State $\tilde{S}^\mathrm{DR}[i,j] \gets \widehat{S}(t_i \mid X_j) - \sum_{k=1}^i \frac{\widehat{S}(t_i \mid X_j)}{\widehat{S}(t_k \mid X_j) \widehat{G}(t_k \mid X_j)} dM_{T \mid X} (t_k \mid X_j)$ \Comment{Eq.~\ref{eq:dr_pseudo_label}}
    \EndFor
\EndFor

\Statex \Comment{Dykstra's algorithm initialization}
\State $S \gets \tilde{S}_{\text{DR}}$
\State $\epsilon_{\text{risk}} \gets \mathbf{0}^{n_{\text{cal}} \times |\mathcal{T}|}$
\State $\epsilon_{\text{time}} \gets \mathbf{0}^{n_{\text{cal}} \times |\mathcal{T}|}$

\While{not converged}
    \Statex \hspace{1em} \Comment{Monotonicity in risk}
    \State $U \gets S + \epsilon_{\text{risk}}$
    \For{$j \in [|\mathcal{T}|]$}
        \State $S[:, j] \gets \texttt{PAVA}(U[:, j], \text{ decreasing}=\text{True})$
    \EndFor
    \State $\epsilon_{\text{risk}} \gets U - S$

    \Statex \hspace{1em} \Comment{Monotonicity in time}
    \State $V \gets S + \epsilon_{\text{time}}$
    \For{$i \in [n_{\text{cal}}]$}
        \State $S[i, :] \gets \texttt{PAVA}(V[i, :], \text{ decreasing}=\text{True})$
    \EndFor
    \State $\epsilon_{\text{time}} \gets V - S$
\EndWhile

\Statex \Comment{Prediction}
\State $\hat{f} \gets \text{Interpolate}(S)$ \Comment{Linearly interpolate to form 2D function}
\State $\widehat{S}_{\mathrm{ISR}}^\mathrm{DR}(t \mid X_{\text{test}}) \gets \hat{f}(t, r_\theta(X_{\text{test}}))$
    \end{algorithmic}
\end{algorithm}
\newpage
\subsection{RW-ISR Estimator}
We present the pseudocode for construction of the RW-ISR estimator below in algorithm~\ref{app:alg:hajek}. As in the DR-ISR case, we also clip the probabilities returned by $\widehat{G}$ at some small $\epsilon$ to ensure stability. To construct the time grid, we use 10000 evenly spaced points between $t=0$ and the maximum observed time in the calibration dataset. We additionally augment this time grid with all the observed times in the training and calibration datasets.
\begin{algorithm}
    \caption{Construction of RW-ISR calibrated survival curve estimates }\label{app:alg:hajek}
    \begin{algorithmic}[1]
        \Require Survival $\widehat{S}$, censoring $\widehat{G}$, and risk $r_\theta$ models from an initial Deep Cox model
\Require Calibration dataset $\mathcal{D}_{\text{cal}} = \{(X_i, Y_i, \delta_i)\}_{i=1}^{n_{\text{cal}}}$ sorted by increasing risk
\Require Time grid $\mathcal{T} = \{t_1, \ldots, t_K\}$

\Statex \Comment{Weight estimation}
\For{$j \in [n_\text{cal}]$}
    \If{$\delta_j=1$}
        \State $W[j]\gets\frac{1}{\widehat{G}(Y_i|X_i)}$
    \Else
        \State $W[j]\gets 0$ \Comment{Assign zero weight to censored patients}
    \EndIf
\EndFor

\Statex \Comment{Perform Isotonic Regression at each time}
\For{$i\in[K]$}
    \For{$j\in[n_\text{cal}]$} \Comment{Construct Binary Outcomes}
        \State $Z[j]\gets\mathbf{1}\{Y_i>t_i\}$
    \EndFor
    \State $S[i; :] \gets\texttt{PAVA}(Z,\text{ weights=}W, \text{ decreasing=True})$
\EndFor

\Statex \Comment{Prediction}
\State $\hat{f} \gets \text{Interpolate}(S)$ \Comment{Linearly interpolate to form 2D function}
\State $\widehat{S}_{\mathrm{ISR}}^\mathrm{RW}(t \mid X_{\text{test}}) \gets \hat{f}(t, r_\theta(X_{\text{test}}))$
    \end{algorithmic}
\end{algorithm}

\clearpage
\newpage

\section{Experiment details}
\label{app:sec:experiment_details}

\subsection{Quantile loss details} \label{app:quantil_loss}
\label{app:sec:metrics}
We first derive the form of the quantile loss under censoring. Recall the standard pinball loss for a quantile prediction $q$ at level $\tau$ in the absence of censoring is
\[\rho_\tau(T, q)=(1-\tau)(q-T)\mathbf{1}\{T\le q\}+\tau(T-q)\mathbf{1}\{T>q\}.\]
Since survival times are nonnegative, each term can be written as an integral over threshold indicators.
\[\rho_\tau(T, q) = (1-\tau)\int_0^q \mathbf{1}[T\le t]\,dt+\tau\int_q^{\infty}\mathbf{1}[T>t]\,dt.\]
Under conditional independent censoring, we can estimate the threshold indicators using IPCW.
\[\mathbf{1}[T\le t]\leadsto \frac{\delta\mathbf{1}[Y\le t]}{\widehat G(Y\mid X)},\qquad \mathbf{1}[T>t]\leadsto \frac{\mathbf{1}[Y>t]}{\widehat G(t\mid X)}.\]
In practice, we also truncate the second integral at the maximum observed time $t_\mathrm{max}$. Substituting these terms into the integral representation and using the predicted quantile $\hat q_\tau(X_i)$ gives the censored quantile score
\[\mathrm{QS}(\tau) = \frac{1}{|\mathcal{I}_{\mathrm{test}}|}\sum_{i\in\mathcal{I}_{\mathrm{test}}}\left[(1-\tau)\int_0^{\hat q_\tau(X_i)}\frac{\delta_i\mathbf{1}[Y_i\le t]}{\widehat G(Y_i\mid X_i)}\,dt+\tau\int_{\hat q_\tau(X_i)}^{t_\mathrm{max}}\frac{\mathbf{1}[Y_i>t]}{\widehat G(t\mid X_i)}\,dt\right].\]
Note that if $\widehat{S}(t\mid X_i)$ does not reach the $1-\tau$ level by $t_{\max}$, then $\widehat{q}_\tau(X_i)$ is ill-defined and hence the $\mathrm{QS}(\tau)$ cannot be calculated.

This stems from a common problem in survival analysis where the survival curves may not reach 0. Intuitively, for a variety of time-to-event outcomes, the event may never happen and hence even the true survival curve may never reach 0. Additionally, when estimating the survival curve, we may have no information regarding how the survival probabilities behave after a maximum time $t_{\max}$. For example, for the Kaplan-Meier estimator, if the largest observed times correspond to censored instances, the estimated survival curve will never reach 0. Cox models, when combined with the Breslow estimator, also suffer from this issue.

When absolutely necessary to obtain survival curve estimates that go to zero (for example to calculate a mean predicted survival time), a variety of extrapolation methods can be applied such as fitting a parametric tail beyond the maximum time or using external data. Unfortunately, however, there is no principled and well-accepted way to extrapolate. See ~\citep{Latimer2013Extrapolation} which presents an official report made to the government of the United Kingdom reviewing suitability of different extrapolation techniques when making decisions regarding efficacy of medical treatments. Ultimately, the appropriate extrapolation method depends on the problem context and downstream analysis. 

For our estimators too, predicted survival curves may not reach 0 and are constant beyond the maximum observed time in the calibration dataset. Thus, $\widehat{q}_\tau(X_i)$ can often be ill-defined at higher quantiles. When evaluating the quantile score, we simply exclude such instances from evaluation. In particular, if the quantile score is ill-defined for a test point for even one of the methods under comparison, we exclude that instance from evaluation for all methods to allow for fair comparison. Figure~\ref{app:fig:quantile_losses_pathology_and_rna} reports the proportion of patients that the quantile losses were evaluated over for the real data pathology and RNA models.

\subsection{Integrated Brier Score details} \label{app:ibs}

We discuss the properties of the IBS and then derive its connection to the integrated quantile loss. We will first find it useful to relate the IBS to the continuous ranked probability score (CRPS), which measures the calibration and sharpness of a probabilistic forecast in the uncensored setting~\citep{hersbach2000decomposition}. For a predictive CDF $F$ of a generic response $Y$, the CRPS is
\[\mathrm{CRPS}(F,Y)=\int_0^\infty \left(F(z)-\mathbf{1}[Y\le z]\right)^2\,dz.\]
The CRPS is a proper scoring rule, in the sense that its expectation is minimized by the true CDF. In the context of survival analysis, the estimated CDF is $\widehat{F}(t \mid X) = 1 - \widehat{S}(t \mid X)$. Thus, we can express the scoring rule as
\[\mathrm{CRPS}(S,T)=\int_0^\infty \left( \mathbf{1}[T\le t] \widehat S(t\mid X)^2 + \mathbf{1}[T>t]\big(1-\widehat S(t\mid X)\big)^2 \right) \,dt.\]
Thus, in the absence of censoring, the CRPS is the integrated Brier score over time. Under right censoring, the threshold indicators can be estimated using IPCW.
\[\mathbf{1}[T\le t]\leadsto \frac{\delta\mathbf{1}[Y\le t]}{\widehat G(Y\mid X)}, \qquad \mathbf{1}[T>t]\leadsto \frac{\mathbf{1}[Y>t]}{\widehat G(t\mid X)}. \]
Substituting these IPCW terms gives,
\[ \int_0^\infty \bigg[ \frac{\delta_i \mathbf{1}[Y_i \leq t] \widehat S(t \mid X_i)^2}{\widehat G(Y_i \mid X_i)} + \frac{\mathbf{1}[Y_i > t] (1 - \widehat S(t \mid X_i))^2}{\widehat G(t \mid X_i)} \bigg] \, dt . \]
Truncating the integral at the maximum observed time $t_\text{max}$ and normalizing recovers the IBS.
\[ \mathrm{IBS} = \frac{1}{|\mathcal{I}_{\mathrm{test}}|} \sum_{i \in \mathcal{I}_{\mathrm{test}}} \frac{1}{t_{\max}}\int_0^{t_{\max}} \bigg[ \frac{\delta_i \mathbf{1}[Y_i \leq t] \widehat S(t \mid X_i)^2}{\widehat G(Y_i \mid X_i)} + \frac{\mathbf{1}[Y_i > t] (1 - \widehat S(t \mid X_i))^2}{\widehat G(t \mid X_i)} \bigg] \, dt . \]
The connection between the CRPS and IBS highlights several shared properties. In particular, it makes clear that IBS assesses both calibration and sharpness of the predicted survival distribution, as noted in the literature~\citep{degroot1983comparison,Qi2024ConformalizedDistributions}. 

In the uncensored setting, the CRPS is also known to admit an equivalent representation as an integral of quantile loss over all quantile levels~\citep{gneiting2011comparing}. We next show that an analogous identity holds for the IBS in the right-censored setting if $\hat q_\tau(X_i)\le t_{\max}$ for all $\tau\in[0,1]$ and all $i\in\mathcal{I}_{\mathrm{test}}$. Integrating over $\tau$ in the quantile loss and exchanging the order of integration gives
\[\int_0^1\mathrm{QS}(\tau)\,d\tau=\frac{1}{|\mathcal{I}_{\mathrm{test}}|}\sum_{i\in\mathcal{I}_{\mathrm{test}}}\int_0^{t_{\max}}\left[\frac{\delta_i\mathbf{1}[Y_i\le t]}{\widehat G(Y_i\mid X_i)}A_i(t)+\frac{\mathbf{1}[Y_i>t]}{\widehat G(t\mid X_i)}B_i(t)\right]dt,\]
where
\[A_i(t)=\int_0^1(1-\tau)\mathbf{1}[t\le \hat q_\tau(X_i)]\,d\tau,\qquad B_i(t)=\int_0^1\tau\mathbf{1}[\hat q_\tau(X_i)\le t]\,d\tau.\]
Since $t\le \hat q_\tau(X_i)$ is equivalent to $\tau\ge \widehat F(t\mid X_i)$, we have
\[A_i(t)=\int_{\widehat F(t\mid X_i)}^1(1-\tau)\,d\tau=\frac{1}{2}\left(1-\widehat F(t\mid X_i)\right)^2=\frac{1}{2}\widehat S(t\mid X_i)^2.\]
Similarly, since $\hat q_\tau(X_i)\le t$ is equivalent to $\tau\le \widehat F(t\mid X_i)$,
\[B_i(t)=\int_0^{\widehat F(t\mid X_i)}\tau\,d\tau=\frac{1}{2}\widehat F(t\mid X_i)^2=\frac{1}{2}\left(1-\widehat S(t\mid X_i)\right)^2.\]
Substituting these two identities yields
{\small \[\int_0^1\mathrm{QS}(\tau)\,d\tau=\frac{1}{2|\mathcal{I}_{\mathrm{test}}|}\sum_{i\in\mathcal{I}_{\mathrm{test}}}\int_0^{t_{\max}}\left[\frac{\delta_i\mathbf{1}[Y_i\le t]\widehat S(t\mid X_i)^2}{\widehat G(Y_i\mid X_i)}+\frac{\mathbf{1}[Y_i>t]\big(1-\widehat S(t\mid X_i)\big)^2}{\widehat G(t\mid X_i)}\right]dt.\]
}Comparing this expression with the definition of IBS gives
\[\mathrm{IBS}=\frac{2}{t_{\max}}\int_0^1\mathrm{QS}(\tau)\,d\tau.\]
Thus, IBS is the integrated quantile score over all quantile levels, up to scaling.

\subsection{Compute resources}
\label{app:sec:compute_resources}
All experiments were conducted using on an Ubuntu 22.04.5 Linux cluster. 10 Intel(R) Xeon(R) E5-2670 v2 @ 2.50GHz CPUs were available along with two NVIDIA Tesla V100 GPUs. 200TB of long term storage was available and 300GB of RAM.

Our simulation experiments need only about 2GB of RAM and testing all methods for a single seed for a single setting takes about 2 minutes on a single CPU.

For the real data experiments, training the Deep Cox models for single seed for a single experiment takes under 2 minutes on a single GPU with no more than 30GB of combined GPU memory and CPU RAM needed. Running all calibration procedures for a single seed takes no more than 2 minutes on a single CPU. 


\subsection{Synthetic Data Details}
\label{app:sec:simulations}
For completeness, we reproduce details from~\citep{Gui2024AdaptiveCutoffs} on the data generating process for the synthetic data experiments in Sec.~\ref{sec:simulations}.

The datapoints $(X_i,Y_i,\delta_i)$ are i.i.d. for all 6 settings. We use $2500$ points for each of the training, calibration datasets and $5000$ points for the test datasets. For each setting, $P_X \sim \text{Unif}([0, 4]^p)$, $P_{T|X} \sim \text{LogNormal}(\mu(x), \sigma^2(x))$, where $p$ is the covariate dimension.

Settings 1--2 are univariate with independent censoring, while Settings 3--4 involve covariate-dependent censoring. Settings 5-6 have high-dimensional covariates ($p = 10$) and covariate-dependent censoring. Table~\ref{app:tab:synthetic_data_generating_process} summarizes the parameters for each setup.

\begin{table}[h!]
\caption{Data generating mechanism for synthetic data experiments. Reproduced from~\citep{Gui2024AdaptiveCutoffs}.}
\label{app:tab:synthetic_data_generating_process}
\centering
\begin{tabular}{ccccc}
\toprule
Setting & $p$ & $\mu(x)$ & $\sigma(x)$ & $P_{C|X}$ \\
\midrule
1 & 1  & $0.632x$                                      & 2                   & $\text{Exp}(0.1)$                                          \\
2 & 1  & $3 \cdot 1\{x > 2\} + x \cdot 1\{x \le 2\}$   & 0.5                 & $\text{Exp}(0.1)$                                          \\
3 & 1  & $2 \cdot 1\{x > 2\} + x \cdot 1\{x \le 2\}$   & 0.5                 & $\text{Exp}\left(0.25 + \frac{6+x}{100}\right)$            \\
4 & 1  & $3 \cdot 1\{x > 2\} + 1.5x \cdot 1\{x \le 2\}$& 0.5                 & $\text{LogNormal}\left(2 + \frac{2-x}{50},\, 0.5\right)$   \\
5 & 10 & $0.126(x_1 + \sqrt{x_3x_5}) + 1$              & 1                   & $\text{Exp}\left(\frac{x_{10}}{10} + \frac{1}{20}\right)$  \\
6 & 10 & $0.126(x_1 + \sqrt{x_3x_5}) + 1$              & $\frac{x_2 + 2}{4}$ & $\text{Exp}\left(\frac{x_{10}}{10} + \frac{1}{20}\right)$  \\
\bottomrule
\end{tabular}
\end{table}

For each setting, we train a linear Cox model to estimate the risk score using the \texttt{lifelines} package. We use Breslow's estimate to obtain the probabilities needed for the survival and censoring functions $\widehat{S}$ and $\widehat{G}$ for our calibration algorithms. We clip the probabilities returned from the Cox model if they are below $10^{-4}$ for numerical stability. We repeat the training and calibration procedure for each setting $100$ times using different random seeds for data generation.

\subsection{Real-data foundation models}
\label{app:sec:real_data_details}
\subsubsection{Pathology Reports Model}
\paragraph{Data:} We use the TCGA dataset which consists of cancer patients with various forms of data and clinical outcomes. For this experiment, we use text-based pathology reports as provided by~\citep{Kefelli2024TCGAreports} to predict the progression free survival (PFS). The PFS times we use were calculated and provided by~\citep{Liu2024TCGAOutcomes}. After selecting those patients which have both PFS times and pathology reports available, we are left with 9517 patients. For these selected patients, $59.55\%$ of the PFS outcomes were censored. We randomly split the data into $20\%$ for testing, $20\%$ for both calibration and validation, and $60\%$ for Deep Cox model training.

\paragraph{Model Training:} We make use of the Bio + Clinical BERT Model from~\citep{alsentzer2019BioBERT} which was trained on all notes from MIMIC III~\citep{Johnson2016MIMIC3}, a database containing electronic health records from ICU patients at the Beth Israel Hospital in Boston, MA. We extract 768 dimensional features by extracting the final token embedding from the last layer of the Bio + Clinical BERT Model. We then train a three layer MLP on top of these embeddings with 128 and 64 hidden units. We use ReLU activation and dropout at a rate of 0.1 between layers. We train the model to maximize the Cox partial likelihood as in~\citep{Katzman2018DeepSurv}. We use Adam~\citep{kingma2015Adam} as an optimizer with a learning rate of 0.001 and betas=(0.9, 0.999) with a batch size of 256. We train for up to 50 epochs. We use early stopping by tracking the loss on the validation set and stopping training if it does not improve for 5 consecutive epochs. We use the model parameters from the epoch with the lowest validation loss. We follow exactly this same procedure to train a model to predict the censoring time too. Over 100 seeds controlling the dataset splits and stochasticity in model training, we obtain a mean C-index of $0.643\pm0.002$ for the time to event Deep Cox model and $0.672\pm0.002$ for the censoring Deep Cox model ($\pm$ 2 standard error reported). We use Breslow's estimator~\citep{Breslow1972BreslowEstimator} on top of these Deep Cox models to obtain the probabilities needed for $\widehat{S},\widehat{G}$ for our calibration algorithms. We clip the probabilities returned from the Deep Cox $\widehat{S},\widehat{G}$ if they are below $0.01$ for numeric stability.

\subsubsection{RNA-seq Model}
\paragraph{Data:} We again use the TCGA dataset. For this experiment, we use bulk-RNA samples as provided by~\citep{Vivian2017Toil} to predict the PFS. The PFS times we use were calculated and provided by~\citep{Liu2024TCGAOutcomes}. After selecting those patients which have both PFS times and RNA-seq samples available, we are left with 10272 patients. For these selected patients, $56.90\%$ of the PFS outcomes were censored. We randomly split the data into $20\%$ for testing, $20\%$ for calibration and validation both, and $60\%$ for Deep Cox model training.

\paragraph{Model Training:} We make use of the BulkFormer 37M Model from~\citep{Kang2025BulkFormer}. BulkFormer was pretrained on over 500,000 human bulk transcriptomic profiles. It incorporates a hybrid encoder architecture, combining a graph neural network to capture explicit gene-gene interactions. We extract 131 dimensional features by taking the mean of the final gene embeddings for all the genes that can be handled by BulkFormer and are present in the compiled data from~\citep{Vivian2017Toil} as suggested by~\citep{Kang2025BulkFormer}. We then train a three layer MLP on top of these embeddings with 128 and 64 hidden units. We use ReLU activation and dropout at a rate of 0.1 between layers. We train the model to maximize the Cox partial likelihood as in~\citep{Katzman2018DeepSurv}. We use Adam~\citep{kingma2015Adam} as an optimizer with a learning rate of 0.001 and betas=(0.9, 0.999) with a batch size of 256. We train for up to 50 epochs. We use early stopping by tracking the loss on the validation set and stopping training if it does not improve for 5 consecutive epochs. We use the model parameters from the epoch with the lowest validation loss. We follow exactly this same procedure to train a model to predict the censoring time too. Over 100 seeds controlling the dataset splits and stochasticity in model training, we obtain a mean C-index of $0.686\pm0.002$ for the time to event Deep Cox model and $0.565\pm0.002$ for the censoring Deep Cox model ($\pm$ 2 standard error reported). We use Breslow's estimator~\citep{Breslow1972BreslowEstimator} on top of these Deep Cox models to obtain the probabilities needed for $\widehat{S},\widehat{G}$ for our calibration algorithms. We clip the probabilities returned from the Deep Cox $\widehat{S},\widehat{G}$ if they are below $0.01$ for numeric stability.

\newpage

\section{Additional experiments}

\subsection{Full simulations results}
We present the AUPIT and IBS metrics for the expanded set of estimators across the six simulation settings in Figure~\ref{app:fig:simulation_aupit_ibs_all_estimators}. In Table~\ref{tab:simulation_quantile_results_all_estimators}, we further examine quantile loss across levels $10\%$ through $90\%$. The fraction of patients included in the computation of the quantile loss is shown in Figure~\ref{app:fig:simulation_quantile_included}. 

The DR-ISR method continues to perform best empirically on the AUPIT and IBS metrics, followed by RW$^+$-ISR and HT-ISR. Our proposed estimators also improve quantile loss across all quantile levels in Settings 2--4. In the remaining settings, the differences across methods are small, and no estimators offers a significant advantage over the uncalibrated Cox baseline.

\begin{figure}[h]
    \centering
    \includegraphics[width=1\linewidth]{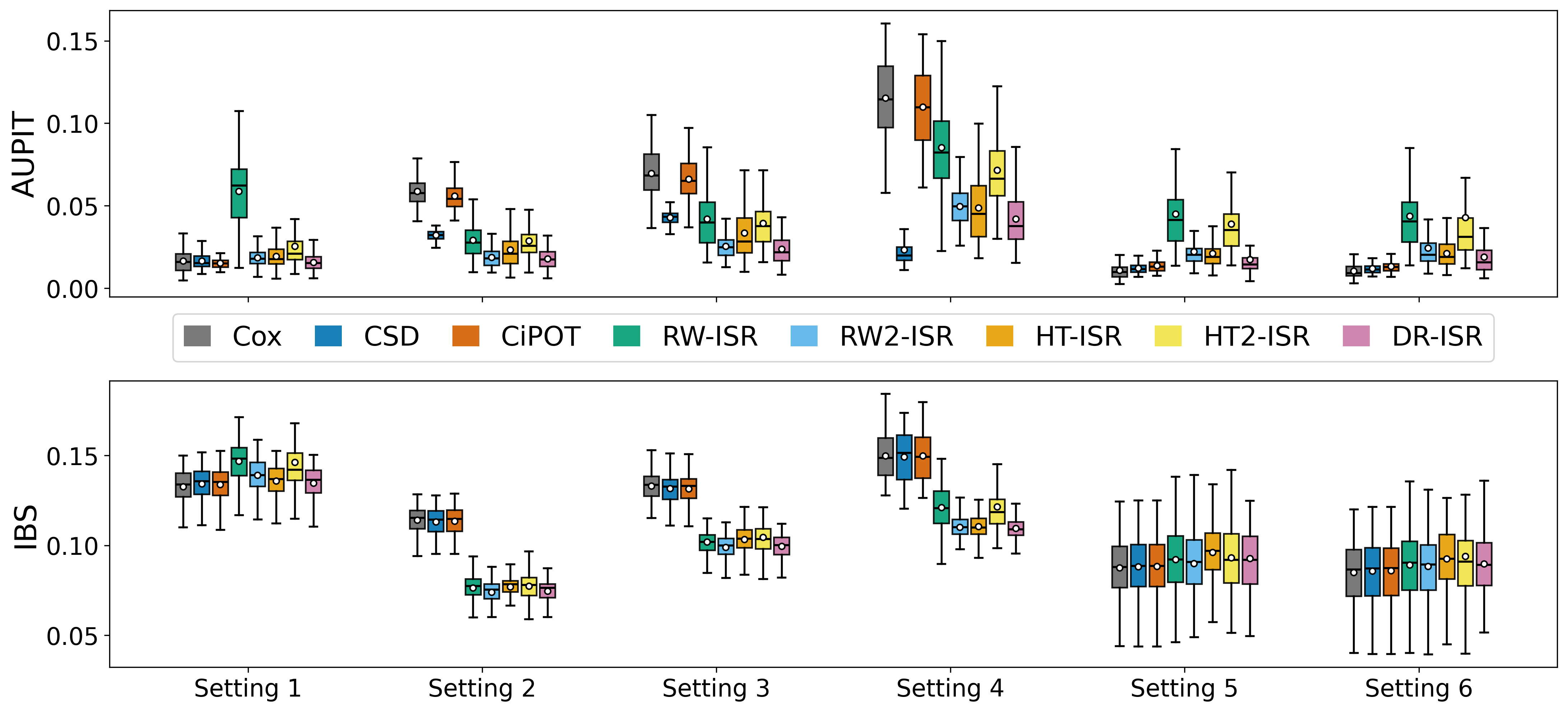}
    \caption{AUPIT and IBS across the six simulation settings for all estimators. Each boxplot summarizes variation over simulation seeds, where the center line indicates the median and the white dot indicates the mean.}
    \label{app:fig:simulation_aupit_ibs_all_estimators}
\end{figure}






\begin{table}[t]
  \caption{Quantile loss from $10\%$ to $90\%$ on the test set across the six simulation settings. Entries report the mean over 100 seeds $\phantom{}\pm2$ standard errors. Bold values indicate the lowest quantile loss for a given setting and quantile when strictly lower than Cox after rounding. For brevity, we remove the ISR suffix on our methods.}
  \vspace{2pt}
  \label{tab:simulation_quantile_results_all_estimators}
  \centering
  { \tiny
  \setlength{\tabcolsep}{2pt}
  \begin{tabular}{llccccccccc}
    \toprule
    & Method & $10\%$ & $20\%$ & $30\%$ & $40\%$ & $50\%$ & $60\%$ & $70\%$ & $80\%$ & $90\%$ \\
    \midrule
    \multirow{8}{*}{\rotatebox[origin=c]{90}{Setting 1}} & \smash{Cox} & \multicolumn{1}{c|}{$3.39 \pm 0.08$} & \multicolumn{1}{c|}{$6.73 \pm 0.16$} & \multicolumn{1}{c|}{$9.99 \pm 0.24$} & \multicolumn{1}{c|}{$13.1 \pm 0.32$} & \multicolumn{1}{c|}{$15.9 \pm 0.42$} & \multicolumn{1}{c|}{$17.8 \pm 0.58$} & \multicolumn{1}{c|}{$16.4 \pm 0.78$} & \multicolumn{1}{c|}{$10.7 \pm 0.87$} & $4.27 \pm 0.53$ \\
    & \smash{CSD} & \multicolumn{1}{c|}{$3.39 \pm 0.08$} & \multicolumn{1}{c|}{$6.73 \pm 0.16$} & \multicolumn{1}{c|}{$9.99 \pm 0.24$} & \multicolumn{1}{c|}{$13.1 \pm 0.32$} & \multicolumn{1}{c|}{$15.9 \pm 0.42$} & \multicolumn{1}{c|}{$17.8 \pm 0.58$} & \multicolumn{1}{c|}{$\textbf{16.3} \pm 0.78$} & \multicolumn{1}{c|}{$10.7 \pm 0.86$} & $4.28 \pm 0.53$ \\
    & \smash{CiPOT} & \multicolumn{1}{c|}{$3.39 \pm 0.08$} & \multicolumn{1}{c|}{$6.73 \pm 0.16$} & \multicolumn{1}{c|}{$9.99 \pm 0.24$} & \multicolumn{1}{c|}{$13.1 \pm 0.32$} & \multicolumn{1}{c|}{$15.9 \pm 0.42$} & \multicolumn{1}{c|}{$17.8 \pm 0.58$} & \multicolumn{1}{c|}{$\textbf{16.3} \pm 0.78$} & \multicolumn{1}{c|}{$10.7 \pm 0.86$} & $4.27 \pm 0.53$ \\
    & \smash{RW} & \multicolumn{1}{c|}{$3.40 \pm 0.08$} & \multicolumn{1}{c|}{$6.75 \pm 0.16$} & \multicolumn{1}{c|}{$10.1 \pm 0.25$} & \multicolumn{1}{c|}{$13.2 \pm 0.32$} & \multicolumn{1}{c|}{$16.1 \pm 0.42$} & \multicolumn{1}{c|}{$18.0 \pm 0.58$} & \multicolumn{1}{c|}{$16.6 \pm 0.79$} & \multicolumn{1}{c|}{$10.9 \pm 0.88$} & $4.42 \pm 0.55$ \\
    & \smash{RW$^+$} & \multicolumn{1}{c|}{$3.39 \pm 0.08$} & \multicolumn{1}{c|}{$6.74 \pm 0.16$} & \multicolumn{1}{c|}{$10.1 \pm 0.24$} & \multicolumn{1}{c|}{$13.2 \pm 0.33$} & \multicolumn{1}{c|}{$16.1 \pm 0.44$} & \multicolumn{1}{c|}{$17.9 \pm 0.59$} & \multicolumn{1}{c|}{$16.6 \pm 0.83$} & \multicolumn{1}{c|}{$11.0 \pm 0.90$} & $4.32 \pm 0.54$ \\
    & \smash{HT} & \multicolumn{1}{c|}{$3.39 \pm 0.08$} & \multicolumn{1}{c|}{$6.73 \pm 0.16$} & \multicolumn{1}{c|}{$10.00 \pm 0.24$} & \multicolumn{1}{c|}{$13.1 \pm 0.32$} & \multicolumn{1}{c|}{$15.9 \pm 0.42$} & \multicolumn{1}{c|}{$17.8 \pm 0.58$} & \multicolumn{1}{c|}{$16.4 \pm 0.79$} & \multicolumn{1}{c|}{$10.7 \pm 0.87$} & $4.31 \pm 0.53$ \\
    & \smash{HT$^+$} & \multicolumn{1}{c|}{$3.79 \pm 0.26$} & \multicolumn{1}{c|}{$7.08 \pm 0.27$} & \multicolumn{1}{c|}{$10.4 \pm 0.31$} & \multicolumn{1}{c|}{$13.5 \pm 0.36$} & \multicolumn{1}{c|}{$16.3 \pm 0.45$} & \multicolumn{1}{c|}{$18.2 \pm 0.73$} & \multicolumn{1}{c|}{$16.8 \pm 0.85$} & \multicolumn{1}{c|}{$11.0 \pm 0.90$} & $4.33 \pm 0.54$ \\
    & \smash{DR} & \multicolumn{1}{c|}{$3.39 \pm 0.08$} & \multicolumn{1}{c|}{$6.73 \pm 0.16$} & \multicolumn{1}{c|}{$10.0 \pm 0.24$} & \multicolumn{1}{c|}{$13.1 \pm 0.32$} & \multicolumn{1}{c|}{$15.9 \pm 0.42$} & \multicolumn{1}{c|}{$17.8 \pm 0.58$} & \multicolumn{1}{c|}{$16.4 \pm 0.78$} & \multicolumn{1}{c|}{$10.7 \pm 0.87$} & $4.29 \pm 0.54$ \\
    \midrule
    \multirow{8}{*}{\rotatebox[origin=c]{90}{Setting 2}} & \smash{Cox} & \multicolumn{1}{c|}{$1.28 \pm 0.02$} & \multicolumn{1}{c|}{$2.41 \pm 0.06$} & \multicolumn{1}{c|}{$2.81 \pm 0.09$} & \multicolumn{1}{c|}{$2.88 \pm 0.11$} & \multicolumn{1}{c|}{$2.78 \pm 0.12$} & \multicolumn{1}{c|}{$2.54 \pm 0.11$} & \multicolumn{1}{c|}{$2.18 \pm 0.09$} & \multicolumn{1}{c|}{$1.64 \pm 0.08$} & $0.96 \pm 0.07$ \\
    & \smash{CSD} & \multicolumn{1}{c|}{$1.23 \pm 0.01$} & \multicolumn{1}{c|}{$2.39 \pm 0.05$} & \multicolumn{1}{c|}{$2.80 \pm 0.09$} & \multicolumn{1}{c|}{$2.87 \pm 0.11$} & \multicolumn{1}{c|}{$2.77 \pm 0.12$} & \multicolumn{1}{c|}{$2.54 \pm 0.11$} & \multicolumn{1}{c|}{$2.18 \pm 0.10$} & \multicolumn{1}{c|}{$1.64 \pm 0.08$} & $1.02 \pm 0.06$ \\
    & \smash{CiPOT} & \multicolumn{1}{c|}{$1.35 \pm 0.03$} & \multicolumn{1}{c|}{$2.42 \pm 0.05$} & \multicolumn{1}{c|}{$2.76 \pm 0.09$} & \multicolumn{1}{c|}{$2.83 \pm 0.11$} & \multicolumn{1}{c|}{$2.74 \pm 0.12$} & \multicolumn{1}{c|}{$2.52 \pm 0.11$} & \multicolumn{1}{c|}{$2.18 \pm 0.10$} & \multicolumn{1}{c|}{$1.65 \pm 0.08$} & $0.96 \pm 0.07$ \\
    & \smash{RW} & \multicolumn{1}{c|}{$\textbf{0.84} \pm 0.01$} & \multicolumn{1}{c|}{$1.46 \pm 0.01$} & \multicolumn{1}{c|}{$1.84 \pm 0.03$} & \multicolumn{1}{c|}{$2.05 \pm 0.04$} & \multicolumn{1}{c|}{$2.14 \pm 0.06$} & \multicolumn{1}{c|}{$2.08 \pm 0.07$} & \multicolumn{1}{c|}{$1.86 \pm 0.07$} & \multicolumn{1}{c|}{$1.44 \pm 0.07$} & $0.84 \pm 0.06$ \\
    & \smash{RW$^+$} & \multicolumn{1}{c|}{$\textbf{0.84} \pm 0.01$} & \multicolumn{1}{c|}{$\textbf{1.44} \pm 0.01$} & \multicolumn{1}{c|}{$\textbf{1.82} \pm 0.03$} & \multicolumn{1}{c|}{$\textbf{2.03} \pm 0.04$} & \multicolumn{1}{c|}{$\textbf{2.09} \pm 0.06$} & \multicolumn{1}{c|}{$2.03 \pm 0.07$} & \multicolumn{1}{c|}{$1.82 \pm 0.07$} & \multicolumn{1}{c|}{$1.38 \pm 0.07$} & $0.78 \pm 0.05$ \\
    & \smash{HT} & \multicolumn{1}{c|}{$\textbf{0.84} \pm 0.01$} & \multicolumn{1}{c|}{$1.45 \pm 0.01$} & \multicolumn{1}{c|}{$\textbf{1.82} \pm 0.03$} & \multicolumn{1}{c|}{$\textbf{2.03} \pm 0.04$} & \multicolumn{1}{c|}{$2.10 \pm 0.06$} & \multicolumn{1}{c|}{$2.04 \pm 0.07$} & \multicolumn{1}{c|}{$1.83 \pm 0.07$} & \multicolumn{1}{c|}{$1.40 \pm 0.07$} & $0.82 \pm 0.06$ \\
    & \smash{HT$^+$} & \multicolumn{1}{c|}{$0.93 \pm 0.04$} & \multicolumn{1}{c|}{$1.51 \pm 0.03$} & \multicolumn{1}{c|}{$1.87 \pm 0.03$} & \multicolumn{1}{c|}{$2.09 \pm 0.06$} & \multicolumn{1}{c|}{$2.15 \pm 0.06$} & \multicolumn{1}{c|}{$2.07 \pm 0.07$} & \multicolumn{1}{c|}{$1.86 \pm 0.07$} & \multicolumn{1}{c|}{$1.40 \pm 0.07$} & $0.79 \pm 0.06$ \\
    & \smash{DR} & \multicolumn{1}{c|}{$\textbf{0.84} \pm 0.01$} & \multicolumn{1}{c|}{$1.45 \pm 0.01$} & \multicolumn{1}{c|}{$\textbf{1.82} \pm 0.03$} & \multicolumn{1}{c|}{$\textbf{2.03} \pm 0.04$} & \multicolumn{1}{c|}{$2.10 \pm 0.06$} & \multicolumn{1}{c|}{$\textbf{2.02} \pm 0.07$} & \multicolumn{1}{c|}{$\textbf{1.80} \pm 0.07$} & \multicolumn{1}{c|}{$\textbf{1.37} \pm 0.07$} & $\textbf{0.77} \pm 0.05$ \\
    \midrule
    \multirow{8}{*}{\rotatebox[origin=c]{90}{Setting 3}} & \smash{Cox} & \multicolumn{1}{c|}{$0.52 \pm 0.01$} & \multicolumn{1}{c|}{$0.97 \pm 0.02$} & \multicolumn{1}{c|}{$1.17 \pm 0.03$} & \multicolumn{1}{c|}{$1.20 \pm 0.04$} & \multicolumn{1}{c|}{$1.14 \pm 0.04$} & \multicolumn{1}{c|}{$1.01 \pm 0.04$} & \multicolumn{1}{c|}{$0.85 \pm 0.04$} & \multicolumn{1}{c|}{$0.60 \pm 0.03$} & $0.32 \pm 0.02$ \\
    & \smash{CSD} & \multicolumn{1}{c|}{$0.51 \pm 0.01$} & \multicolumn{1}{c|}{$0.96 \pm 0.01$} & \multicolumn{1}{c|}{$1.18 \pm 0.03$} & \multicolumn{1}{c|}{$1.20 \pm 0.04$} & \multicolumn{1}{c|}{$1.14 \pm 0.04$} & \multicolumn{1}{c|}{$1.02 \pm 0.04$} & \multicolumn{1}{c|}{$0.85 \pm 0.04$} & \multicolumn{1}{c|}{$0.59 \pm 0.03$} & $0.33 \pm 0.02$ \\
    & \smash{CiPOT} & \multicolumn{1}{c|}{$0.53 \pm 0.01$} & \multicolumn{1}{c|}{$0.96 \pm 0.02$} & \multicolumn{1}{c|}{$1.15 \pm 0.03$} & \multicolumn{1}{c|}{$1.18 \pm 0.04$} & \multicolumn{1}{c|}{$1.12 \pm 0.04$} & \multicolumn{1}{c|}{$1.00 \pm 0.04$} & \multicolumn{1}{c|}{$0.85 \pm 0.04$} & \multicolumn{1}{c|}{$0.60 \pm 0.03$} & $0.32 \pm 0.02$ \\
    & \smash{RW} & \multicolumn{1}{c|}{$\textbf{0.38} \pm 0.01$} & \multicolumn{1}{c|}{$0.68 \pm 0.01$} & \multicolumn{1}{c|}{$0.86 \pm 0.02$} & \multicolumn{1}{c|}{$0.95 \pm 0.02$} & \multicolumn{1}{c|}{$0.97 \pm 0.02$} & \multicolumn{1}{c|}{$0.93 \pm 0.03$} & \multicolumn{1}{c|}{$0.82 \pm 0.03$} & \multicolumn{1}{c|}{$0.59 \pm 0.03$} & $0.32 \pm 0.02$ \\
    & \smash{RW$^+$} & \multicolumn{1}{c|}{$\textbf{0.38} \pm 0.01$} & \multicolumn{1}{c|}{$\textbf{0.66} \pm 0.01$} & \multicolumn{1}{c|}{$\textbf{0.84} \pm 0.01$} & \multicolumn{1}{c|}{$0.94 \pm 0.02$} & \multicolumn{1}{c|}{$0.96 \pm 0.02$} & \multicolumn{1}{c|}{$\textbf{0.91} \pm 0.03$} & \multicolumn{1}{c|}{$\textbf{0.80} \pm 0.03$} & \multicolumn{1}{c|}{$\textbf{0.57} \pm 0.03$} & $\textbf{0.30} \pm 0.02$ \\
    & \smash{HT} & \multicolumn{1}{c|}{$\textbf{0.38} \pm 0.01$} & \multicolumn{1}{c|}{$\textbf{0.66} \pm 0.01$} & \multicolumn{1}{c|}{$0.85 \pm 0.01$} & \multicolumn{1}{c|}{$0.94 \pm 0.02$} & \multicolumn{1}{c|}{$0.96 \pm 0.03$} & \multicolumn{1}{c|}{$0.92 \pm 0.03$} & \multicolumn{1}{c|}{$0.81 \pm 0.03$} & \multicolumn{1}{c|}{$0.59 \pm 0.03$} & $0.32 \pm 0.02$ \\
    & \smash{HT$^+$} & \multicolumn{1}{c|}{$0.44 \pm 0.04$} & \multicolumn{1}{c|}{$0.71 \pm 0.03$} & \multicolumn{1}{c|}{$0.88 \pm 0.02$} & \multicolumn{1}{c|}{$0.98 \pm 0.03$} & \multicolumn{1}{c|}{$0.99 \pm 0.03$} & \multicolumn{1}{c|}{$0.93 \pm 0.03$} & \multicolumn{1}{c|}{$0.81 \pm 0.03$} & \multicolumn{1}{c|}{$0.59 \pm 0.04$} & $\textbf{0.30} \pm 0.02$ \\
    & \smash{DR} & \multicolumn{1}{c|}{$\textbf{0.38} \pm 0.01$} & \multicolumn{1}{c|}{$\textbf{0.66} \pm 0.01$} & \multicolumn{1}{c|}{$\textbf{0.84} \pm 0.01$} & \multicolumn{1}{c|}{$\textbf{0.93} \pm 0.02$} & \multicolumn{1}{c|}{$\textbf{0.95} \pm 0.02$} & \multicolumn{1}{c|}{$\textbf{0.91} \pm 0.03$} & \multicolumn{1}{c|}{$\textbf{0.80} \pm 0.03$} & \multicolumn{1}{c|}{$\textbf{0.57} \pm 0.03$} & $\textbf{0.30} \pm 0.02$ \\
    \midrule
    \multirow{8}{*}{\rotatebox[origin=c]{90}{Setting 4}} & \smash{Cox} & \multicolumn{1}{c|}{$1.30 \pm 0.09$} & \multicolumn{1}{c|}{$1.59 \pm 0.10$} & \multicolumn{1}{c|}{$1.73 \pm 0.10$} & \multicolumn{1}{c|}{$1.73 \pm 0.10$} & \multicolumn{1}{c|}{$1.66 \pm 0.10$} & \multicolumn{1}{c|}{$1.41 \pm 0.10$} & \multicolumn{1}{c|}{$1.17 \pm 0.09$} & \multicolumn{1}{c|}{$0.89 \pm 0.09$} & $0.52 \pm 0.07$ \\
    & \smash{CSD} & \multicolumn{1}{c|}{$1.13 \pm 0.02$} & \multicolumn{1}{c|}{$1.86 \pm 0.04$} & \multicolumn{1}{c|}{$2.18 \pm 0.05$} & \multicolumn{1}{c|}{$1.98 \pm 0.08$} & \multicolumn{1}{c|}{$1.77 \pm 0.09$} & \multicolumn{1}{c|}{$1.47 \pm 0.10$} & \multicolumn{1}{c|}{$1.21 \pm 0.09$} & \multicolumn{1}{c|}{$0.90 \pm 0.09$} & $0.52 \pm 0.06$ \\
    & \smash{CiPOT} & \multicolumn{1}{c|}{$1.33 \pm 0.09$} & \multicolumn{1}{c|}{$1.59 \pm 0.10$} & \multicolumn{1}{c|}{$1.71 \pm 0.10$} & \multicolumn{1}{c|}{$1.72 \pm 0.10$} & \multicolumn{1}{c|}{$1.66 \pm 0.10$} & \multicolumn{1}{c|}{$1.41 \pm 0.10$} & \multicolumn{1}{c|}{$1.18 \pm 0.09$} & \multicolumn{1}{c|}{$0.90 \pm 0.09$} & $0.53 \pm 0.07$ \\
    & \smash{RW} & \multicolumn{1}{c|}{$0.78 \pm 0.02$} & \multicolumn{1}{c|}{$1.24 \pm 0.05$} & \multicolumn{1}{c|}{$1.52 \pm 0.07$} & \multicolumn{1}{c|}{$1.65 \pm 0.08$} & \multicolumn{1}{c|}{$1.66 \pm 0.09$} & \multicolumn{1}{c|}{$1.46 \pm 0.10$} & \multicolumn{1}{c|}{$1.24 \pm 0.10$} & \multicolumn{1}{c|}{$0.96 \pm 0.10$} & $0.58 \pm 0.08$ \\
    & \smash{RW$^+$} & \multicolumn{1}{c|}{$\textbf{0.77} \pm 0.02$} & \multicolumn{1}{c|}{$\textbf{1.20} \pm 0.04$} & \multicolumn{1}{c|}{$1.47 \pm 0.06$} & \multicolumn{1}{c|}{$1.59 \pm 0.08$} & \multicolumn{1}{c|}{$1.59 \pm 0.09$} & \multicolumn{1}{c|}{$1.37 \pm 0.09$} & \multicolumn{1}{c|}{$1.15 \pm 0.09$} & \multicolumn{1}{c|}{$0.87 \pm 0.09$} & $0.52 \pm 0.07$ \\
    & \smash{HT} & \multicolumn{1}{c|}{$\textbf{0.77} \pm 0.02$} & \multicolumn{1}{c|}{$\textbf{1.20} \pm 0.04$} & \multicolumn{1}{c|}{$\textbf{1.46} \pm 0.06$} & \multicolumn{1}{c|}{$\textbf{1.58} \pm 0.08$} & \multicolumn{1}{c|}{$1.59 \pm 0.09$} & \multicolumn{1}{c|}{$1.38 \pm 0.10$} & \multicolumn{1}{c|}{$1.15 \pm 0.09$} & \multicolumn{1}{c|}{$0.87 \pm 0.09$} & $0.50 \pm 0.06$ \\
    & \smash{HT$^+$} & \multicolumn{1}{c|}{$0.91 \pm 0.07$} & \multicolumn{1}{c|}{$1.30 \pm 0.07$} & \multicolumn{1}{c|}{$1.55 \pm 0.07$} & \multicolumn{1}{c|}{$1.67 \pm 0.08$} & \multicolumn{1}{c|}{$1.63 \pm 0.09$} & \multicolumn{1}{c|}{$1.41 \pm 0.10$} & \multicolumn{1}{c|}{$1.16 \pm 0.09$} & \multicolumn{1}{c|}{$0.88 \pm 0.09$} & $0.52 \pm 0.07$ \\
    & \smash{DR} & \multicolumn{1}{c|}{$\textbf{0.77} \pm 0.02$} & \multicolumn{1}{c|}{$\textbf{1.20} \pm 0.04$} & \multicolumn{1}{c|}{$1.47 \pm 0.06$} & \multicolumn{1}{c|}{$\textbf{1.58} \pm 0.07$} & \multicolumn{1}{c|}{$\textbf{1.58} \pm 0.09$} & \multicolumn{1}{c|}{$\textbf{1.36} \pm 0.10$} & \multicolumn{1}{c|}{$\textbf{1.14} \pm 0.09$} & \multicolumn{1}{c|}{$\textbf{0.85} \pm 0.08$} & $\textbf{0.49} \pm 0.06$ \\
    \midrule
    \multirow{8}{*}{\rotatebox[origin=c]{90}{Setting 5}} & \smash{Cox} & \multicolumn{1}{c|}{$0.65 \pm 0.01$} & \multicolumn{1}{c|}{$1.22 \pm 0.01$} & \multicolumn{1}{c|}{$1.72 \pm 0.01$} & \multicolumn{1}{c|}{$2.14 \pm 0.02$} & \multicolumn{1}{c|}{$2.47 \pm 0.02$} & \multicolumn{1}{c|}{$2.67 \pm 0.03$} & \multicolumn{1}{c|}{$2.67 \pm 0.05$} & \multicolumn{1}{c|}{$2.34 \pm 0.09$} & $1.41 \pm 0.13$ \\
    & \smash{CSD} & \multicolumn{1}{c|}{$0.65 \pm 0.01$} & \multicolumn{1}{c|}{$1.22 \pm 0.01$} & \multicolumn{1}{c|}{$1.72 \pm 0.01$} & \multicolumn{1}{c|}{$2.14 \pm 0.02$} & \multicolumn{1}{c|}{$2.47 \pm 0.02$} & \multicolumn{1}{c|}{$2.67 \pm 0.03$} & \multicolumn{1}{c|}{$2.67 \pm 0.05$} & \multicolumn{1}{c|}{$2.34 \pm 0.09$} & $\textbf{1.40} \pm 0.13$ \\
    & \smash{CiPOT} & \multicolumn{1}{c|}{$0.65 \pm 0.01$} & \multicolumn{1}{c|}{$1.22 \pm 0.01$} & \multicolumn{1}{c|}{$1.72 \pm 0.01$} & \multicolumn{1}{c|}{$2.14 \pm 0.02$} & \multicolumn{1}{c|}{$2.47 \pm 0.02$} & \multicolumn{1}{c|}{$2.67 \pm 0.03$} & \multicolumn{1}{c|}{$2.67 \pm 0.05$} & \multicolumn{1}{c|}{$2.34 \pm 0.09$} & $\textbf{1.40} \pm 0.13$ \\
    & \smash{RW} & \multicolumn{1}{c|}{$0.66 \pm 0.02$} & \multicolumn{1}{c|}{$1.26 \pm 0.04$} & \multicolumn{1}{c|}{$1.77 \pm 0.03$} & \multicolumn{1}{c|}{$2.24 \pm 0.08$} & \multicolumn{1}{c|}{$2.59 \pm 0.07$} & \multicolumn{1}{c|}{$2.77 \pm 0.05$} & \multicolumn{1}{c|}{$2.76 \pm 0.06$} & \multicolumn{1}{c|}{$2.42 \pm 0.10$} & $1.50 \pm 0.13$ \\
    & \smash{RW$^+$} & \multicolumn{1}{c|}{$0.65 \pm 0.01$} & \multicolumn{1}{c|}{$1.22 \pm 0.01$} & \multicolumn{1}{c|}{$1.73 \pm 0.02$} & \multicolumn{1}{c|}{$2.16 \pm 0.02$} & \multicolumn{1}{c|}{$2.52 \pm 0.03$} & \multicolumn{1}{c|}{$2.72 \pm 0.04$} & \multicolumn{1}{c|}{$2.75 \pm 0.06$} & \multicolumn{1}{c|}{$2.39 \pm 0.10$} & $1.45 \pm 0.13$ \\
    & \smash{HT} & \multicolumn{1}{c|}{$0.65 \pm 0.01$} & \multicolumn{1}{c|}{$1.22 \pm 0.01$} & \multicolumn{1}{c|}{$1.73 \pm 0.01$} & \multicolumn{1}{c|}{$2.15 \pm 0.02$} & \multicolumn{1}{c|}{$2.49 \pm 0.02$} & \multicolumn{1}{c|}{$2.70 \pm 0.03$} & \multicolumn{1}{c|}{$2.73 \pm 0.06$} & \multicolumn{1}{c|}{$2.47 \pm 0.11$} & $1.48 \pm 0.13$ \\
    & \smash{HT$^+$} & \multicolumn{1}{c|}{$0.69 \pm 0.03$} & \multicolumn{1}{c|}{$1.29 \pm 0.04$} & \multicolumn{1}{c|}{$1.83 \pm 0.05$} & \multicolumn{1}{c|}{$2.26 \pm 0.06$} & \multicolumn{1}{c|}{$2.58 \pm 0.05$} & \multicolumn{1}{c|}{$2.79 \pm 0.05$} & \multicolumn{1}{c|}{$2.81 \pm 0.07$} & \multicolumn{1}{c|}{$2.40 \pm 0.10$} & $1.46 \pm 0.13$ \\
    & \smash{DR} & \multicolumn{1}{c|}{$0.65 \pm 0.01$} & \multicolumn{1}{c|}{$1.22 \pm 0.01$} & \multicolumn{1}{c|}{$1.73 \pm 0.01$} & \multicolumn{1}{c|}{$2.15 \pm 0.02$} & \multicolumn{1}{c|}{$2.53 \pm 0.06$} & \multicolumn{1}{c|}{$2.71 \pm 0.05$} & \multicolumn{1}{c|}{$2.72 \pm 0.06$} & \multicolumn{1}{c|}{$2.42 \pm 0.10$} & $1.44 \pm 0.13$ \\
    \midrule
    \multirow{8}{*}{\rotatebox[origin=c]{90}{Setting 6}} & \smash{Cox} & \multicolumn{1}{c|}{$0.72 \pm 0.01$} & \multicolumn{1}{c|}{$1.35 \pm 0.01$} & \multicolumn{1}{c|}{$1.90 \pm 0.02$} & \multicolumn{1}{c|}{$2.38 \pm 0.02$} & \multicolumn{1}{c|}{$2.78 \pm 0.03$} & \multicolumn{1}{c|}{$3.05 \pm 0.04$} & \multicolumn{1}{c|}{$3.15 \pm 0.05$} & \multicolumn{1}{c|}{$2.90 \pm 0.10$} & $1.85 \pm 0.17$ \\
    & \smash{CSD} & \multicolumn{1}{c|}{$0.72 \pm 0.01$} & \multicolumn{1}{c|}{$1.35 \pm 0.01$} & \multicolumn{1}{c|}{$1.90 \pm 0.02$} & \multicolumn{1}{c|}{$2.38 \pm 0.02$} & \multicolumn{1}{c|}{$2.78 \pm 0.03$} & \multicolumn{1}{c|}{$3.05 \pm 0.04$} & \multicolumn{1}{c|}{$3.15 \pm 0.05$} & \multicolumn{1}{c|}{$2.90 \pm 0.10$} & $\textbf{1.84} \pm 0.17$ \\
    & \smash{CiPOT} & \multicolumn{1}{c|}{$0.72 \pm 0.01$} & \multicolumn{1}{c|}{$1.35 \pm 0.01$} & \multicolumn{1}{c|}{$1.90 \pm 0.02$} & \multicolumn{1}{c|}{$2.38 \pm 0.02$} & \multicolumn{1}{c|}{$2.78 \pm 0.03$} & \multicolumn{1}{c|}{$3.05 \pm 0.04$} & \multicolumn{1}{c|}{$3.15 \pm 0.05$} & \multicolumn{1}{c|}{$2.90 \pm 0.10$} & $1.85 \pm 0.17$ \\
    & \smash{RW} & \multicolumn{1}{c|}{$0.73 \pm 0.01$} & \multicolumn{1}{c|}{$1.39 \pm 0.04$} & \multicolumn{1}{c|}{$1.96 \pm 0.04$} & \multicolumn{1}{c|}{$2.45 \pm 0.04$} & \multicolumn{1}{c|}{$2.85 \pm 0.04$} & \multicolumn{1}{c|}{$3.12 \pm 0.04$} & \multicolumn{1}{c|}{$3.24 \pm 0.07$} & \multicolumn{1}{c|}{$2.98 \pm 0.11$} & $1.95 \pm 0.18$ \\
    & \smash{RW$^+$} & \multicolumn{1}{c|}{$0.72 \pm 0.01$} & \multicolumn{1}{c|}{$1.35 \pm 0.01$} & \multicolumn{1}{c|}{$1.91 \pm 0.02$} & \multicolumn{1}{c|}{$2.42 \pm 0.04$} & \multicolumn{1}{c|}{$2.86 \pm 0.09$} & \multicolumn{1}{c|}{$3.19 \pm 0.09$} & \multicolumn{1}{c|}{$3.28 \pm 0.08$} & \multicolumn{1}{c|}{$2.96 \pm 0.11$} & $1.88 \pm 0.17$ \\
    & \smash{HT} & \multicolumn{1}{c|}{$0.72 \pm 0.01$} & \multicolumn{1}{c|}{$1.35 \pm 0.01$} & \multicolumn{1}{c|}{$1.91 \pm 0.02$} & \multicolumn{1}{c|}{$2.39 \pm 0.02$} & \multicolumn{1}{c|}{$2.79 \pm 0.03$} & \multicolumn{1}{c|}{$3.07 \pm 0.04$} & \multicolumn{1}{c|}{$3.19 \pm 0.05$} & \multicolumn{1}{c|}{$2.98 \pm 0.11$} & $1.94 \pm 0.18$ \\
    & \smash{HT$^+$} & \multicolumn{1}{c|}{$0.82 \pm 0.10$} & \multicolumn{1}{c|}{$1.45 \pm 0.09$} & \multicolumn{1}{c|}{$2.10 \pm 0.15$} & \multicolumn{1}{c|}{$2.61 \pm 0.15$} & \multicolumn{1}{c|}{$3.02 \pm 0.13$} & \multicolumn{1}{c|}{$3.27 \pm 0.10$} & \multicolumn{1}{c|}{$3.30 \pm 0.08$} & \multicolumn{1}{c|}{$2.97 \pm 0.11$} & $1.89 \pm 0.17$ \\
    & \smash{DR} & \multicolumn{1}{c|}{$0.72 \pm 0.01$} & \multicolumn{1}{c|}{$1.35 \pm 0.01$} & \multicolumn{1}{c|}{$1.91 \pm 0.02$} & \multicolumn{1}{c|}{$2.40 \pm 0.03$} & \multicolumn{1}{c|}{$2.81 \pm 0.05$} & \multicolumn{1}{c|}{$3.08 \pm 0.05$} & \multicolumn{1}{c|}{$3.25 \pm 0.09$} & \multicolumn{1}{c|}{$3.02 \pm 0.12$} & $1.88 \pm 0.17$ \\
    \bottomrule
  \end{tabular}}
\end{table}

\begin{figure}[h]
    \centering
    \includegraphics[width=0.8\linewidth]{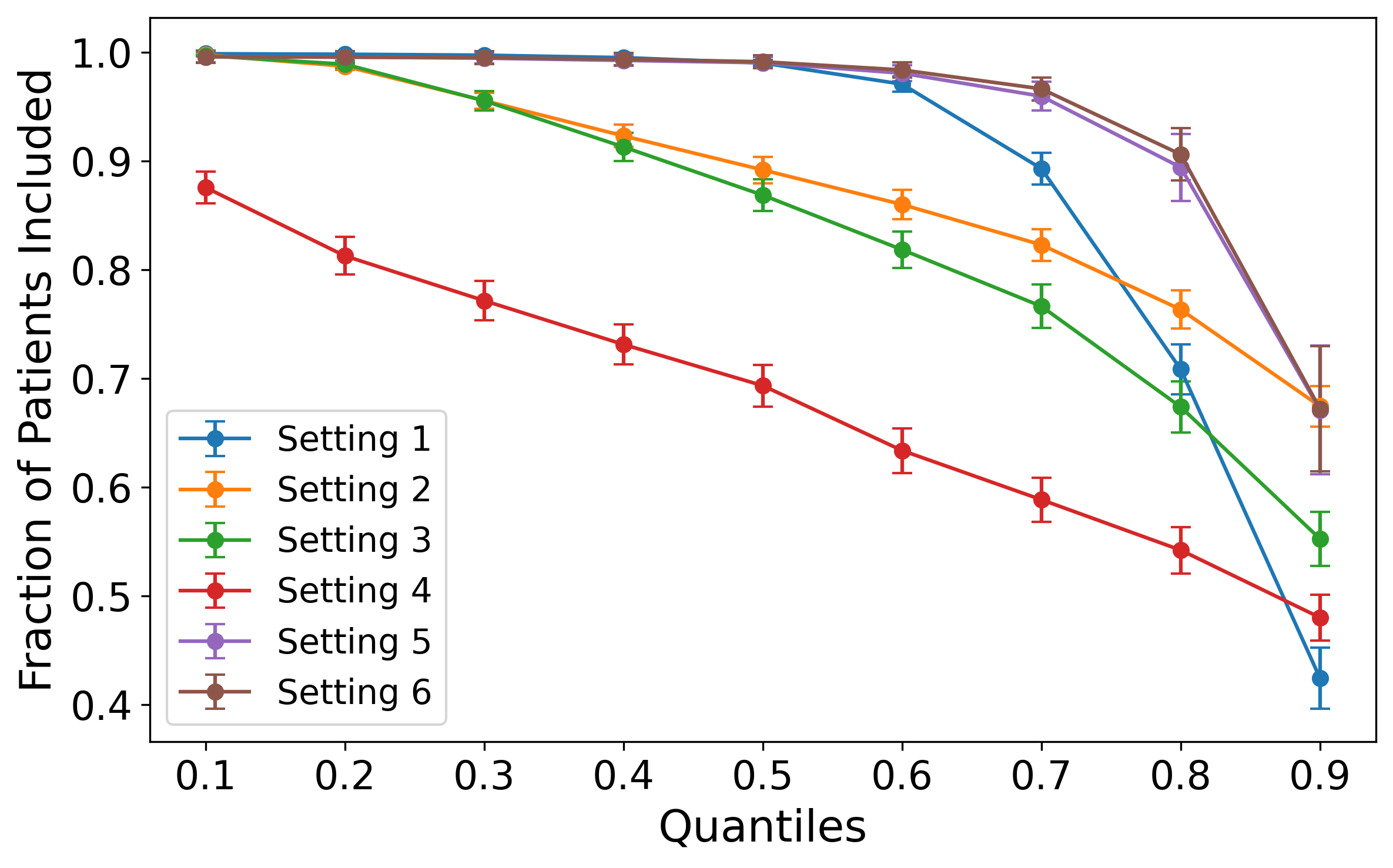}
    \caption{Fraction of total samples included in the computation of the quantile loss across nine different levels. Values shown are the mean over 100 seeds $\phantom{}\pm2$ standard error bars.}
    \label{app:fig:simulation_quantile_included}
\end{figure}

\subsection{Full real-data results}

For completeness, we now present complete results on the real data for both naïve and IPCW evaluation of the metrics as described in Sec.~\ref{sec:metrics}. We also include the RW$^+$-ISR and HT-ISR methods in our results. Table~\ref{app:tab:real_data_results_complete} lists the IPCW and Naïve estimates of AUPIT and IBS. Figure~\ref{app:fig:quantile_losses_pathology_and_rna} visualizes the quantile losses for the pathology reports and RNA sequence models. As described in Sec.~\ref{app:sec:metrics}, when evaluating quantile loss, if the predicted survival curve does not reach the target quantile by the maximum valid prediction time (maximum observed time for the ISR methods) for any one method, we exclude that patient from evaluation for all methods. 

As discussed in Sec.~\ref{sec:real_data}, DR-ISR exhibits the strongest performance for the IPCW estimates, even improving upon the baselines CSD and CiPOT which explicitly target PIT calibration on the AUPIT metric. RW$^+$-ISR also exhibits strong performance, which is expected by improving upon RW-ISR in terms of data efficiency, as discussed in Sec.~\ref{app:sec:new_estimators}. On lower quantiles performance of all methods is quite similar but DR-ISR and RW$^+$-ISR once again improve at higher quantiles.

For the naïve evaluation of the metrics, we see that RW-ISR and CSD exhibit strongest performance across all metrics. Recall, however, that by only evaluating over the uncensored instances, we introduce a distribution shift. In particular, a patient is uncensored if the event time is lower than the censoring time. Intuitively then, the uncensored instances are those patients for whom the event happened sooner than may be expected from their true survival curve. Therefore, methods that underestimate survival probabilities are likely to perform better on naïve evaluation of the metrics. We suspect that both RW-ISR and CSD underestimate survival probabilities given the big gap in their performance on naïve versus IPCW metric evaluations.

{\small \begin{table}[h]
  \caption{IPCW and Naïve estimates of AUPIT, IBS values over the test set for the pathology and RNA-seq model. Values are mean over 100 seeds $\phantom{}\pm2$ standard error.}
  \vspace{2pt}
  \label{app:tab:real_data_results_complete}
  \centering
  \begin{tabular}{lccccc}
    \toprule
    & Method    &  IPCW AUPIT      & IPCW IBS & Naïve AUPIT & Naïve IBS \\
    \midrule
    \multirow{7}{*}{\rotatebox[origin=c]{90}{Pathology}} 
& Deep Cox & $0.0665 \pm 0.0033$ & $0.184 \pm 0.0068$ & $0.213 \pm 0.0013$ & $0.0547 \pm 0.0014$ \\
& CSD & $0.0720 \pm 0.0056$ & $0.185 \pm 0.0069$ & $\textbf{0.163} \pm 0.0020$ & $0.0514 \pm 0.0013$ \\
& CiPOT & $0.0627 \pm 0.0034$ & $0.184 \pm 0.0069$ & $0.205 \pm 0.0013$ & $0.0537 \pm 0.0014$ \\
& RW-ISR & $0.0761 \pm 0.0054$ & $0.189 \pm 0.0082$ & $0.166 \pm 0.0047$ & $\textbf{0.0300} \pm 0.0009$ \\
& RW$^+$-ISR & $0.0615 \pm 0.0036$ & $\textbf{0.168} \pm 0.0064$ & $0.226 \pm 0.0019$ & $0.0501 \pm 0.0014$ \\
& HT-ISR & $0.0657 \pm 0.0042$ & $0.185 \pm 0.0078$ & $0.191 \pm 0.0025$ & $0.0404 \pm 0.0015$ \\
& DR-ISR & $\textbf{0.0538} \pm 0.0035$ & $0.182 \pm 0.0070$ & $0.210 \pm 0.0021$ & $0.0547 \pm 0.0021$ \\
    \midrule
    \multirow{7}{*}{\rotatebox[origin=c]{90}{RNA-seq}} 
& Deep Cox & $0.0371 \pm 0.0020$ & $0.154 \pm 0.0053$ & $0.199 \pm 0.0011$ & $0.0433 \pm 0.0011$ \\
& CSD & $0.0478 \pm 0.0038$ & $0.155 \pm 0.0056$ & $0.158 \pm 0.0020$ & $0.0393 \pm 0.0010$ \\
& CiPOT & $0.0358 \pm 0.0020$ & $0.154 \pm 0.0054$ & $0.198 \pm 0.0017$ & $0.0423 \pm 0.0011$ \\
& RW-ISR & $0.0698 \pm 0.0055$ & $0.156 \pm 0.0071$ & $\textbf{0.152} \pm 0.0050$ & $\textbf{0.0278} \pm 0.0009$ \\
& RW$^+$-ISR & $0.0361 \pm 0.0024$ & $\textbf{0.147} \pm 0.0062$ & $0.198 \pm 0.0021$ & $0.0382 \pm 0.0013$ \\
& HT-ISR & $0.0371 \pm 0.0025$ & $0.158 \pm 0.0055$ & $0.198 \pm 0.0022$ & $0.0427 \pm 0.0012$ \\
& DR-ISR & $\textbf{0.0345} \pm 0.0024$ & $0.152 \pm 0.0055$ & $0.198 \pm 0.0020$ & $0.0417 \pm 0.0011$ \\
    \bottomrule
  \end{tabular}
\end{table}}

\begin{figure}[h!]
    \centering
    \includegraphics[width=0.99\linewidth]{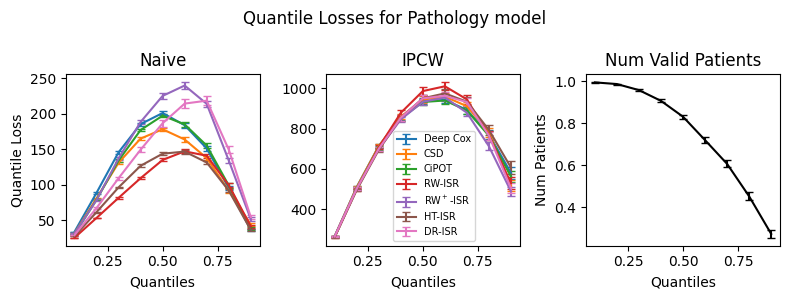}
    \includegraphics[width=0.99\linewidth]{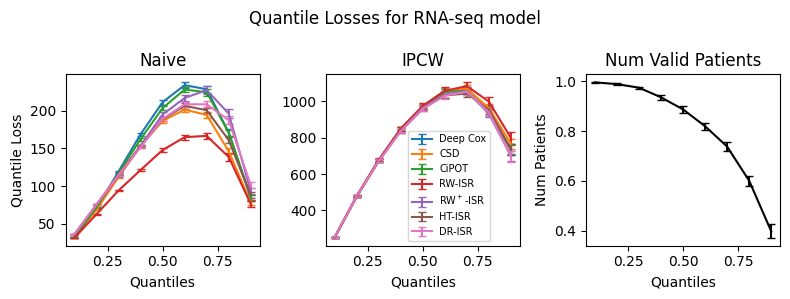}
    \caption{Quantile losses for the pathology reports (top) and RNA (bottom) model. Right panel visualizes proportion of patients for which the survival curve reached the desired quantile across all methods and hence were included in evaluation. See Sec.~\ref{app:sec:metrics} for details. Values shown are mean over 100 seeds $\phantom{}\pm2$ standard error bars.}
    \label{app:fig:quantile_losses_pathology_and_rna}
\end{figure}

\clearpage

\section{Experimental Assets}
\label{app:sec:assets}

The TCGA patient outcomes were obtained from~\citep{Liu2024TCGAOutcomes} which is publicly available under the CC BY-NC-ND license. The pathology reports were obtained from~\citep{Kefelli2024TCGAreports} available under the CC-BY license. The BioClinical BERT model from ~\citep{alsentzer2019BioBERT} is available under the MIT license. The RNA-seq samples were obtained from~\citep{Vivian2017Toil} available under the Apache 2.0 license. Lastly, the BulkFormer model from~\citep{Kang2025BulkFormer} is available under the MIT license. 

The code for the baseline methods was sourced from the official GitHub repositories, available under the MIT license. For training the linear Cox models, we use the official \texttt{lifelines} package, available under the MIT license.




\end{document}